\documentclass[twoside]{article}

\usepackage[round]{natbib}

\usepackage{subcaption} 
 
\usepackage{multirow}
\usepackage[utf8]{inputenc}
\usepackage{amsmath,amssymb,array}
\usepackage{bbm}
\usepackage{bm}
\usepackage{xcolor}
\usepackage{graphicx}
\usepackage{xargs}
\usepackage[font=small,skip=0pt]{caption}
\usepackage{wrapfig}
\usepackage{caption}

\usepackage{multicol}
\usepackage{booktabs}
\usepackage{todonotes}
\usepackage{amsthm}
\usepackage{titletoc}

\renewcommand\citet{\cite}
\usepackage[hyphens]{url}
\usepackage[hidelinks,colorlinks=true,citecolor=blue!50!black,linkcolor=black,urlcolor=green!50!black]{hyperref}
\usepackage{colortbl}
\usepackage{csquotes} 
\usepackage{mathtools}
\usepackage{float} 
\usepackage{adjustbox}
\usepackage{dsfont}

\usepackage[dvipsnames]{xcolor} 
\definecolor{myblue}{RGB}{0,180,216} 
\definecolor{mypink}{RGB}{239,39,166} 
\definecolor{myorange}{RGB}{255,111,0} 
\definecolor{myyellow}{RGB}{255,190,11} 

\definecolor{codegreen}{rgb}{0,0.6,0}

\newtheorem{theorem}{Theorem}

\newtheorem{definition}{Definition}
\newtheorem{corollary}{Corollary}

\usepackage[table]{xcolor}
\usepackage{multirow}
\usepackage{array}

\usepackage{algorithm}
\usepackage[noend]{algpseudocode}

\newcommand{\cmark}{\textcolor{green}{\checkmark}}
\newcommand{\xmark}{\textcolor{red}{\ding{55}}} 
\usepackage{pifont}

\definecolor{interaction_palette}{HTML}{3F88C5}
\definecolor{correlation_palette}{HTML}{57AA99}
\newcommand{\interaction}[1]{\textcolor{interaction_palette}{\bfseries #1}}
\newcommand{\correlation}[1]{\textcolor{correlation_palette}{\bfseries #1}}

\usepackage[accepted]{aistats2026}

\begin{document}

\runningauthor{Herbinger, Laberge, Muschalik, Pequignot, Wright, Fumagalli}
\runningtitle{GRANITE: A Generalized Regional Framework}

\twocolumn[

\aistatstitle{GRANITE: A Generalized Regional Framework for Identifying Agreement in Feature-Based Explanations}

\aistatsauthor{ Julia Herbinger\\ Leibniz Institute for Prevention \\ Research and Epidemiology - BIPS \And Gabriel
Laberge\\ Polytechnique Montréal   \And Maximilian Muschalik\\ LMU Munich, MCML \AND Yann Pequignot\\ Université Laval à Québec \And Marvin N. Wright\\ Leibniz Institute for Prevention \\Research and Epidemiology - BIPS\\ University of Bremen \And  Fabian Fumagalli\\ LMU Munich, MCML \\ Bielefeld University, CITEC }

\aistatsaddress{} 

]

\begin{abstract}
Feature-based explanation methods aim to quantify how features influence the model's behavior, either locally or globally, but different methods often disagree, producing conflicting explanations. This disagreement arises primarily from two sources: how feature interactions are handled and how feature dependencies are incorporated. We propose GRANITE, a generalized regional explanation framework that partitions the feature space into regions where interaction and distribution influences are minimized. This approach aligns different explanation methods, yielding more consistent and interpretable explanations. GRANITE unifies existing regional approaches, extends them to feature groups, and introduces a recursive partitioning algorithm to estimate such regions. We demonstrate its effectiveness on real-world datasets, providing a practical tool for consistent and interpretable feature explanations.
\end{abstract}

\section{INTRODUCTION}\label{sec:intro}

Black-box models, such as neural networks or ensembles, often achieve superior predictive performance but remain difficult to interpret, which is a major concern in high-stakes applications like healthcare or finance. Feature-based explanation methods aim to reveal how features influence predictions, either \textit{locally}, for example, using SHAP values \citep{lundberg_2017_unified} or individual conditional expectation (ICE) curves \citep{goldstein_peeking_2015}, or \textit{globally}, for example, using permutation or conditional feature importance (PFI and CFI) \citep{fisher2019all, strobl2008conditional}.

Yet, these methods frequently \emph{disagree}, producing contradictory explanations for the same model \citep{krishna2024disagreement, mitruț2024clarity, pirie2023agree, schwarzschild2023reckoning}. For example, 
Figure~\ref{fig:disagreement_illustration} (lower-left) shows different explanation methods assigning \textit{conflicting importance to the same features} at the global level. \cite{fumagalli2025unifying} identify two main sources of such disagreement:
(1) how methods handle \emph{interactions}---e.g., \emph{PredDiff} fully accounts for the interaction between $X_1$ and $X_2$, whereas \emph{SHAP} methods include it only partially---and (2) how they incorporate \emph{feature distributions} when features are correlated---e.g., the dependence of $X_4$ on $X_3$ is reflected in \emph{conditional SHAP} and \emph{PredDiff} but not in \emph{marginal SHAP}. Differences along either dimension cause explanations to diverge, complicating interpretation. Importantly, in the absence of interactions and dependencies, explanations are \emph{simple, additive, and consistent}, suggesting that reducing these effects improves interpretability. 
Figure~\ref{fig:disagreement_illustration} shows that restricting analysis to regions minimizing interactions (e.g., $X_2 = 1$), distribution influence (e.g., $X_3 = 1$), or both (e.g., $X_2 = 1$ \& $X_3 = 1$), progressively aligns explanations, producing increasingly consistent and interpretable explanations.

This insight motivates our approach. We propose the \textbf{Generalized Regional frAmework for ideNtIfying agreemenT in feature-based Explanations (GRANITE)}, which mitigates disagreement by partitioning the feature space into regions where interaction or distribution influence is minimized, yielding consistent and interpretable explanations.

\textbf{Contributions.} Our main contributions are:

\textbf{1.} Formalizing the disagreement problem for feature-based explanations, arising solely from differences in handling interactions and feature distributions.

\textbf{2.} The GRANITE framework, which partitions the feature space into regions minimizing disagreement to produce consistent explanations, implemented via recursive partitioning and demonstrated on real data.

\textbf{3.} Unifying existing regional approaches and addressing previously unexplored gaps.

\begin{figure}
    \centering
    \includegraphics[width=\linewidth]{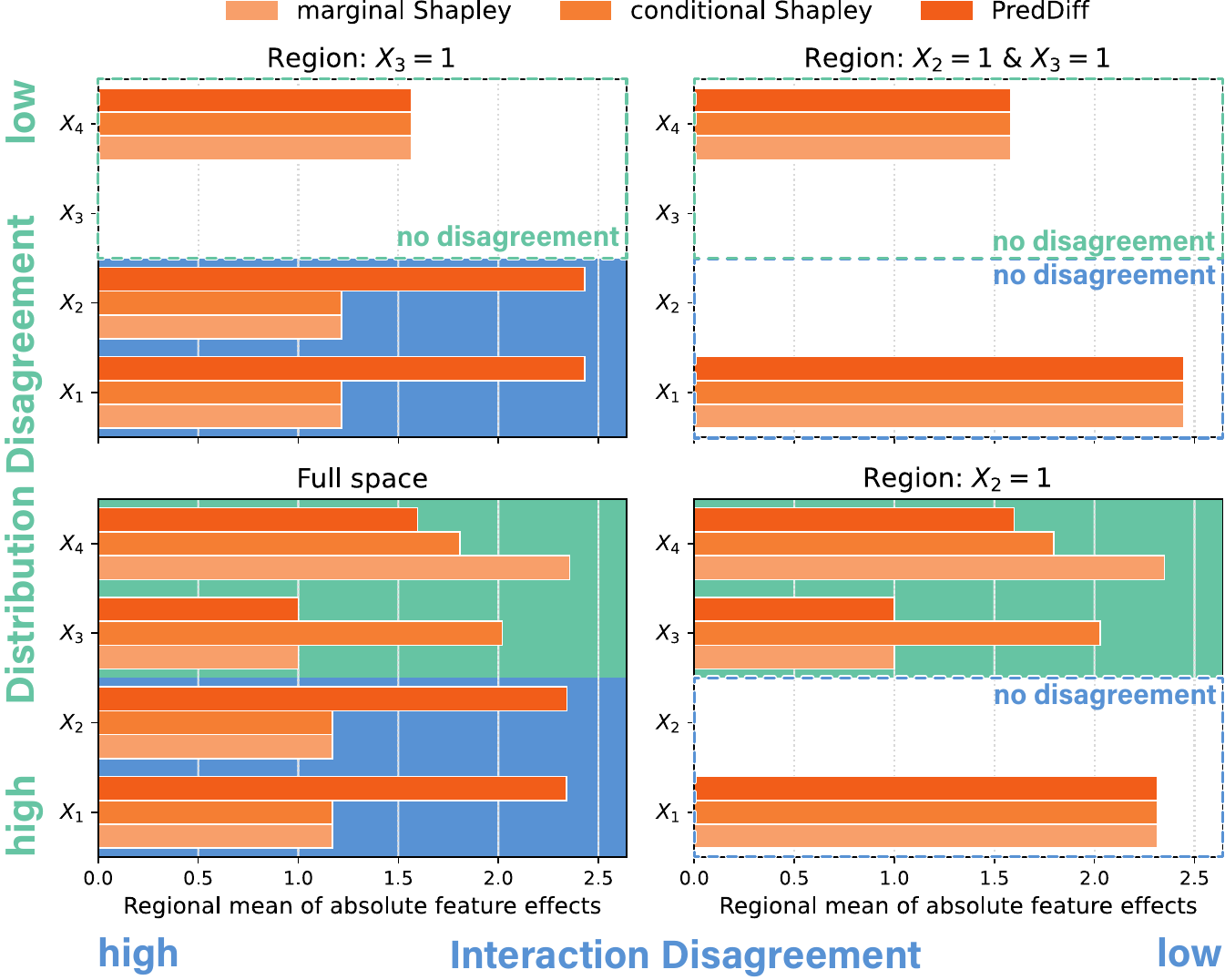}
    \\[0.3em]
    \caption{Disagreement of average absolute marginal and conditional SHAP \citep{lundberg_2017_unified} and PredDiff \citep{Robnik2008} values for synthetic data, where $X_1$ interacts with $X_2$ and $X_4$ depends on $X_3$, cf. Section~\ref{sec:disagreement}. 
    Explanations disagree globally (bottom left) due to feature \interaction{interactions} and \correlation{dependencies}. 
    Regional explanations remove disagreements caused by interactions (bottom right) and dependencies (top left), with overall agreement when combined (top right).
    }
    \label{fig:disagreement_illustration}
\end{figure}

\section{BACKGROUND}\label{sec:background}

\paragraph{General Notation}
We consider a prediction task over a feature space $\mathcal{X} \subseteq \mathbb{R}^d$ and target space $\mathcal{Y}$, 
with random variables $X = (X_1, \ldots, X_d)$ and $Y$. We write $P$ for the reference distribution, which may denote the distribution of $X$ or the joint distribution of $(X,Y)$ depending on context 
(see Table~\ref{tab:distribution_overview}).  
A (black box) model $F: \mathcal X \to \mathcal{Y}$ maps observed feature vectors $x \in \mathbb{R}^d$ to predictions.  
We denote the feature index set by $D := \{1,\ldots,d\}$ with power set $2^D$. We write $i = \{i\}$, use $-S := D \setminus S$ for the complement of $S \subseteq D$, and $x_S$ are the entries of $x$ for features in $S$.

\subsection{Functional Decomposition}\label{sec:fanova}
Functional decomposition has a long history in statistics and has received increasing attention for interpreting machine learning models. Particularly, the standard functional ANOVA (fANOVA) decomposition \citep{stone1994fanova, hooker_discovering_2004, Owen_2013} provides for any \emph{independent} joint distribution $P$ and square-integrable prediction function $F$ a decomposition: 
\begin{equation*}
    F(x) = f_\emptyset + \sum_{i=1}^d f_i(x_i) + \sum_{i\neq j} f_{ij}(x_i,x_j) \dots = \sum_{S \subseteq D} f_S(x),
\end{equation*}
where each component $f_S$ does not depend on features in $-S$.
Here, $f_\emptyset$ is constant, $f_i(x_i)$ is the $i$-th main effect, and $f_S(x)$ with $|S|>1$ represents the \textit{pure} 
interaction effect of features in $S$.
The components $f_S$ are defined by removing the dependency on features outside $S$ via $F_S(x) := \mathbb{E}_{P_{X_{-S}}}[F(x_S,X_{-S})]$ with joint marginal distribution $P_{X_{-S}}$, and recursively subtract lower-order effects as
\begin{align}\label{eq_fanova_effect}
    f_S(x) = F_S(x) - \sum_{T \subset S} f_T(x),
\end{align}
Following \citet{fumagalli2025unifying}, we adopt a unifying framework for feature-based explanations that extends fANOVA and supports various masking operations $F_S$.

\subsection{Unifying Framework}\label{sec:background_unifying}

\citet{fumagalli2025unifying} introduced a unifying framework for feature-based explanations based on the fANOVA decomposition and the Möbius transform \citep{rota1964foundations, grabisch2000equivalent} from cooperative 
game theory. Based on their framework, explanations can be defined by four dimensions: 

(1) Distribution influence: Are features \emph{masked} via baseline value, marginal, or conditional distribution?\\ 
(2) Explanation game: Is the \emph{behavior} of the explanation local or global?\\
(3) Interaction influence: Are interactions not, partially, or fully accounted for in the \emph{influence} measure?\\
(4) Influence type: Are we interested in individual feature or feature group influence (joint / interaction)?

\subsubsection{Definition of Feature Influence}
Adapting the notation of \cite{enouen2025instashap} we view feature-based explanations as the result of three consecutive operations: \emph{masking}, 
\emph{behavior}, and \emph{interaction}  operators. We link these operators to the first three dimensions of \cite{fumagalli2025unifying} for individual feature 
influences; extensions to feature groups are deferred to Appendix~\ref{app:group_extension} and discussed in Section~\ref{sec:feature_groups}.

\textit{(1) Masking.}
The \emph{masking} operator maps a predictor $F$ to a family of predictors $F_S$ with $S \subseteq D$, where each $F_S$ is restricted to features in $S$.
To this end, the features outside of $S$ are removed (masked) using a reference distribution $P$, which is either the conditional or marginal distribution.
In addition, we consider a masking operator using a single baseline input $b \in \mathcal X$.
The baseline, marginal, and conditional masking operations are increasingly influenced by dependencies in the feature distribution \citep{fumagalli2025unifying}.

\begin{definition}[\correlation{Masking}]\label{def:mask}
The \textbf{masking} operator $\mathcal M$ maps $F$ to a family of masked models
\[
(\mathcal M F): \mathcal X \times 2^D \to \mathbb{R}, \quad 
F_S(x) := (\mathcal M F)(x,S),
\]
where $\mathcal M$ is chosen according to Table~\ref{tab:operators}.
\end{definition}

\textit{(2) Behavior.}
The \emph{behavior} operator maps a family of masked models $\mathcal M F$ to a \emph{value function} that quantifies the influence or importance for each feature subset.

\begin{definition}[Behavior]\label{def:behavior}
The \textbf{behavior} operator $\mathcal B$ maps the masked models $F_{\cdot} \equiv \mathcal M F$ to a set function
\[
(\mathcal B (\mathcal M F)): 2^D \to \mathbb{R}, \quad 
\nu(S) := (\mathcal B(\mathcal M F))(S),
\]
where $S \subseteq D$, and $\mathcal B$ is chosen according to Table~\ref{tab:operators}. 
\end{definition}

The influence (importance) of a subset $S \subseteq D$ is measured \emph{locally} by the masked model prediction ($\mathcal B_{\text{loc},x_0})$, or \emph{globally} by the variance ($\mathcal B_{\text{sens}}$) or performance ($\mathcal B_{\text{risk}}$) of the masked model.

\textit{(3) Interaction.} On top of the \textit{behavior} operator, we apply the \textit{interaction} operator which quantifies explanations based on the influence of higher-order effects. \cite{fumagalli2025unifying} distinguishes between \emph{pure}, \emph{partial}, and \emph{full} effects to measure the influence of individual features.

\begin{definition}[\interaction{Interaction}]\label{def:interaction}
The \textbf{interaction} operator $\mathcal I$ maps $\nu \equiv \mathcal B (\mathcal M F)$ to an attribution function
\[
(\mathcal I (\mathcal B (\mathcal M F))): D \to \mathbb{R}, \quad
\phi(i) := (\mathcal I (\mathcal B (\mathcal M F)))(i),
\]
where $i \in D$ and $\mathcal I$ is chosen according to Table~\ref{tab:operators}.
\end{definition}

\begin{table}[htb]
\centering
\caption{Overview of feature influence operators.}
\resizebox{\columnwidth}{!}{%
\begin{tabular}{l|l|l}
\textbf{Concept} & \textbf{Operator} & \textbf{Definition}\\\toprule
\multicolumn{3}{c}{\textbf{\correlation{Masking Operators}}}\\\midrule
baseline & $\mathcal M_b$ & $\mathbb{E}_{\delta_{b_{-S}}}[F(x_S,X_{-S})] = F(x_S,b_{-S})$ \\
marginal & $\mathcal M_m$ & $\mathbb{E}_{P_{X_{-S}}}[F(x_S, X_{-S})]$ \\
conditional & $\mathcal M_c$ & $\mathbb{E}_{P_{X_{-S} \mid X_S =x_S}}[F(x_S, X_{-S})]$ \\\midrule
\multicolumn{3}{c}{\textbf{Behavior Operators}}\\\midrule
local & $\mathcal B_{\text{loc},x_0}$ & $F_S(x_0)$ \\
sensitivity & $\mathcal B_{\text{sens}}$ & $\mathbb{V}_{P_X}[F_S(X)]$ \\
risk & $\mathcal B_{\text{risk}}$ & $-\mathbb{E}_{P_{X,Y}}[\ell(F_S(X),Y)]$ \\\midrule
\multicolumn{3}{c}{\textbf{\interaction{Interaction Operators}}}\\\midrule
pure & $\mathcal I_{\text{pure}}$ & $\nu(i)-\nu(\emptyset)$ \\
partial & $\mathcal I_{\text{partial}}$ & $\sum\limits_{S \subseteq -i} \frac{1}{d \binom{d-1}{|S|}}[\nu(S \cup i)-\nu(S)]$ \\
full & $\mathcal I_{\text{full}}$ & $\nu(D)-\nu(-i)$
\end{tabular}%
}
\label{tab:operators}
\end{table}

\textit{Pure effects} measure the influence of a feature in the \emph{absence} of all other features and thereby ignore its role in higher-order effects.
\textit{Partial effects} are defined via the Shapley value \citep{shapley_1951} of $\nu \equiv \mathcal B(\mathcal M F)$, which yields a fair distribution of interaction effects across involved features. 
\textit{Full effects} measure the influence in the \emph{presence} of all remaining features and hence incorporate higher-order effects entirely.

Interaction operators for \emph{individual effects} generalize to \emph{interactions} and \emph{joint effects} for arbitrary subsets $\phi:2^D\to\mathbb{R}$, cf. Appendix~\ref{app:group_extension}.
Examples of interaction and joint effects 
include the Shapley-Taylor interaction index \citep{sundararajan2020shapley} and 
the grouped-PFI \citep{au2022grouped}.

\section{THE GRANITE METHOD}\label{sec:granite}
In this section, we define the disagreement problem of feature-based explanations and introduce the GRANITE method, which aims to identify regions in the feature space where explanations agree.

\subsection{The Disagreement Problem}\label{sec:disagreement}
According to \cite{fumagalli2025unifying}, feature-based explanations are defined by four dimensions: (1) feature distribution 
influence (\emph{masking}), (2) explanation game (\emph{behavior}), (3) influence of higher-order effects (\emph{interaction}), and 
(4) type of influence measure (individual or groups of features). 
When comparing explanations, (2) and (4) are typically fixed to one category, since local and global, or single-feature and group-feature explanations describe fundamentally different aspects. 
This motivates our definition of the disagreement problem for feature-based explanations:
\begin{definition}[Disagreement Problem]\label{def:disagree}
For a fixed influence measure (individual, joint, or interaction) and behavior 
$\mathcal B$, two feature-based explanations $\phi_1$ and $\phi_2$ \textbf{disagree} if and only if they differ due to the choice of \correlation{masking} $\mathcal M$ or \interaction{interaction} 
$\mathcal I$ (see Table \ref{tab:operators}).
\end{definition}

\paragraph{Illustration.}
We consider the toy example from Figure~\ref{fig:disagreement_illustration}. 
Let $X_1 \sim \mathbb{N}(0,1)$, $X_2, X_3 \in \{-1,1\}$ with $P(X_2 = \pm 1) = P(X_3 = \pm 1) = 0.5$, and $X_4$ be conditionally distributed as 
$(X_4\mid X_3 = x_3) \sim \mathcal{N}(x_3,1)$. 
Let the model be $f(X) = 3X_1X_2 + X_3 + 2X_4$, and assume a local behavior $\mathcal B = \mathcal{B}_{loc,x_0}$. 
This setup illustrates two types of disagreement (Table~\ref{tab:example}). First, for $X_1$, the interaction \interaction{$3X_1X_2$} induces opposing effects depending on $X_2$: Under marginal masking $\mathcal{M}_{m}$, the pure effect of $X_1$ averages to \interaction{zero}, while the full effect \interaction{fully} captures the interaction, creating a clear discrepancy. 
Second, for feature $X_4$, which enters the model linearly, disagreement arises between marginal and conditional masking: Fixing $\mathcal I = \mathcal{I}_{pure}$ yields the linear term \correlation{$2x_4$} for $\mathcal M_m$, while correlation between $X_3$ and $X_4$ introduces an additional nonlinear component \correlation{$\mu_{3\mid 4}$} under $\mathcal M_c$. Thus, even without interactions, feature dependencies can induce disagreement.

\begin{table*}[t]
\caption{Comparison of pure and full effects of features $X_1$ and $X_4$ under marginal ($\mathcal{M}_m$) and conditional ($\mathcal M_c$) distributions for the entire feature space $\Omega:= \mathcal X$ and for regions defined by either $X_2$ or $X_3$ where
$\Omega_{(x_i=j)}:=\{x\in \mathcal X \mid x_i = j\}$. Also, $\mu_{3\mid4}:=\mathbb{E}[X_3\mid X_4=x_4]$ and '-' indicates the feature effect in the given region is the same as in $\Omega:= \mathcal X$.}\label{tab:example}
\begin{tabular}{@{}ll|cc||cc|cc||cc|cc@{}}
\toprule
\multirow{2}{*}{\textbf{$X_i$}} & \multirow{2}{*}{\textbf{Type}} & 
\multicolumn{2}{c||}{\large $\Omega := \mathcal X$} & 
\multicolumn{2}{c|}{\large $\Omega_{(x_2=1)}$} & 
\multicolumn{2}{c||}{\large $\Omega_{(x_2=-1)}$} & 
\multicolumn{2}{c|}{\large $\Omega_{(x_3=1)}$} & 
\multicolumn{2}{c}{\large $\Omega_{(x_3=-1)}$} \\
 &  & $\mathcal{M}_m$ & $\mathcal M_c$ & $\mathcal{M}_m$ & $\mathcal M_c$ & $\mathcal{M}_m$ & $\mathcal M_c$ & $\mathcal{M}_m$ & $\mathcal M_c$ & $\mathcal{M}_m$ & $\mathcal M_c$ \\
\midrule
\multirow{2}{*}{$X_1$} & $\mathcal I_{\text{pure}}$ & \interaction{$\bm{0}$} & $0$ & $3x_1$ & $3x_1$ & $-3x_1$ & $-3x_1$ & - & - & - & - \\
 & $\mathcal I_{\text{full}}$ & \interaction{$\bm{3x_1 x_2}$} & $3x_1 x_2$ & $3x_1$ & $3x_1$ & $-3x_1$ & $-3x_1$ & - & - & -& - \\\midrule
 \multirow{2}{*}{$X_4$} & $\mathcal I_{\text{pure}}$ & \correlation{$\bm{2x_4}$} & \correlation{$\bm{2x_4 + \mu_{3 \mid 4}}$} & - & - & - & - & $2x_4 - 2$ & $2x_4 - 2$ & $2x_4 + 2$ & $2x_4 + 2$ \\
 & $\mathcal I_{\text{full}}$ & $2x_4$ & $2x_4 + \mu_{3 \mid 4}$ & - & - & - & - & $2x_4 - 2$ & $2x_4 - 2$ & $2x_4 + 2$ & $2x_4 + 2$ \\
\bottomrule
\end{tabular}
\end{table*}

\subsection{Definitions of Regional Explanations}
To reduce disagreement, GRANITE proposes explanations with respect to \emph{regions} $\Omega \subseteq \mathcal{X}$ of the feature space.
To this end, the feature space $\mathcal X$ is partitioned,
\[
\mathcal{X} = \bigcup_{k=1}^{K} \Omega_k, \quad \Omega_i \cap \Omega_j = \emptyset \text{ for } i \neq j,
\]
into \emph{regions} $\Omega_k$.
In practice, $\Omega_k$ is determined by recursive partitioning (e.g., decision trees), or any algorithm assigning samples to disjoint regions based on similarity, proximity, or axis-aligned constraints.
\begin{table}[htb]
\centering
\caption{Distribution $P$ used in masking and behavior}
\begin{tabular}{l|l|l}
\textbf{Context} & \textbf{Variant} & \textbf{Distribution} \\\toprule
\multirow{2}{*}{Masking} 
  & marginal & $P_{X_{-S}}$  \\
  & conditional & $P_{X_{-S}\mid X_S = x_S}$ \\\midrule
\multirow{2}{*}{Behavior} 
  & local/sensitivity & $P_X$ \\
  & risk & $P_{X,Y}$ \\
\end{tabular}
\label{tab:distribution_overview}
\end{table}

We define \emph{regional explanations} $\phi_{\mid \Omega}$ as the result of the operators from Table~\ref{tab:operators} restricted to the region $\Omega$.
For this purpose, we restrict all \emph{reference distributions} summarized in Table~\ref{tab:distribution_overview} to the support of $\Omega$, denoted by $P^\Omega$.
A special case is the $\delta$-distribution, where the baseline $b$ must either be fixed (possibly outside of $\Omega$) or depend on $\Omega$.
The restricted operators $\mathcal M^\Omega$ and $\mathcal B^\Omega$ are then defined by Definitions~\ref{def:mask}--\ref{def:behavior}, respectively, and remain valid with $P$ replaced by $P^\Omega$.
The interaction operator $\mathcal I$ remains unchanged.
The individual regional influence measures are then defined by
\begin{equation}
    \phi_{\mid \Omega}(i) := \left(\mathcal{I}\left(\mathcal{B}^\Omega(\mathcal{M}^\Omega F)\right)\right)(i), \quad i \in D.
\end{equation}

\subsection{Finding Agreement among Explanations}

GRANITE seeks regions where two explanations
\begin{align*}
    \phi_{1\mid\Omega} \equiv \mathcal I_1(\mathcal B^\Omega (\mathcal M^\Omega_1 F)) \text{ and } \phi_{2\mid \Omega} \equiv \mathcal I_2(\mathcal B^\Omega (\mathcal M^\Omega_2 F))
\end{align*}
agree by minimizing their disagreement.

 \begin{definition}[Regional Disagreement]\label{def:regional_disagreement_risk}

Let $\ell: \mathbb{R} \to \mathbb{R}_{\geq 0}$.
We define \textbf{regional disagreement} as
\begin{align*}
\mathcal R_{\ell \mid \Omega}(\phi_{1 \mid \Omega},\phi_{2\mid\Omega}) := \mathbb{E}_{P_{X}^\Omega}\left[ \sum_{i \in D} \ell\left(\phi_{1\mid\Omega}(i)-\phi_{2\mid\Omega}(i)\right)\right],
\end{align*}
where $\mathbb{E}_{P^\Omega_{X}}$ only affects local explanations with $\mathcal B =\mathcal B^\Omega_{x, \text{loc}}$ by averaging the distances for all $x \in \Omega$.
\end{definition}

Examples for $\ell$ are the absolute or squared value. 
The regional disagreement extends directly to explanations with groups of features, i.e., for joint effects and interaction effects, as discussed in Section~\ref{sec:feature_groups}.

GRANITE aims to identify regions where two explanations agree, and every datapoint should be assigned exclusively to one region.
This naturally leads to partitions $\{\Omega_1, \ldots , \Omega_K\}$ with disjoint regions $\Omega_1,\dots,\Omega_K \subseteq \mathcal X$
that minimize the total regional disagreement:
\begin{equation}
   \min_{\{\Omega_k\}_{k=1}^K} \sum_{k=1}^K \mathcal{R}_{\ell\mid\Omega_k} (\phi_{1\mid\Omega_k}, \phi_{2\mid\Omega_k})
   \label{eq:objective}
\end{equation}

Given our explainability objective, an implicit constraint on partitions is that they should be interpretable and not include too many regions.
In Section \ref{sec:estimation}, we propose an algorithm to solve this task.

Depending on whether the \emph{masking} or \emph{interaction} operator differs, GRANITE minimizes the corresponding source of disagreement (interaction effects or feature distribution influence). The following paragraphs analyze the case of individual feature influence, with an extension to feature groups in Section~\ref{sec:feature_groups}.

\textbf{\interaction{Minimizing Interactions.}}
For fixed \emph{masking} and \emph{behavior} operators, different \emph{interaction} operators (e.g., pure, partial, full) may lead to different feature influence measures $\phi_1$ and $\phi_2$, with differences arising solely from higher-order effects.

\begin{theorem}\label{theorem:min_interactions}
Let $\mathcal M_1^\Omega = \mathcal M_2^\Omega$. Then for $i \in D$,
\begin{align*}
    \phi_{1\mid \Omega}(i)-\phi_{2\mid \Omega}(i) = \sum_{S \subseteq D: i \in S, \vert S \vert \geq 2} \tilde{w}^{S} \Delta(S),
\end{align*}
where $\tilde{w}^S \geq 0$ are specific interaction weights, and 
\begin{align*}
\Delta(S) := \sum_{L \subseteq S} (-1)^{\vert S \vert - \vert L \vert} (\mathcal B^\Omega(\mathcal M_1^\Omega F))(L),
\end{align*}
\emph{pure} interactions of the value function $\nu \equiv \mathcal B^\Omega(\mathcal M_1^\Omega F)$.
\end{theorem}

This decomposition always exists, as it is derived from the Möbius transform of the cooperative game $\nu(S) := (\mathcal{B}(\mathcal{M} F))(S)$, with weights being the differences between the representations of each interaction operator $\mathcal{I}$ (Appendix~\ref{app:proof_thm1}).

\textit{Illustration.}
In the example from Table~\ref{tab:example}, fixing $\mathcal{M} = \mathcal{M}_m$, the pure and full effects differ only in their interaction operator $\mathcal I$.
For feature $X_1$, the disagreement \interaction{${3x_1x_2}$} reflects interaction effects. 
By Theorem~\ref{theorem:min_interactions}, the disagreement is driven solely by higher-order terms, so GRANITE’s partitioning along $X_2$ minimizes these interactions within regions, which leads to agreement between pure and full effects.

In general, the disagreement between explanations changes when using different behavior operators.
However, under feature independence and squared error $\ell$, the disagreement of full and pure local explanations is exactly the difference of the sensitivity explanations.

\begin{theorem}
\label{thm:general-local-to-global-variance}
Let $\ell(x) := x^2$, and consider $\phi_{\text{pure}}^{\text{loc}}$,$\phi_{\text{full}}^{\text{loc}}$,$\phi_{\text{pure}}^{\text{sens}}$, and $\phi_{\text{full}}^{\text{sens}}$, where the subscript indicates the interaction operator $\mathcal I_{\text{pure}},\mathcal I_{\text{full}}$, and the superscript the behavior operator $\mathcal B_{x_{0},\text{loc}},\mathcal B_{\text{sens}}$.
Under feature independence, i.e. $\mathcal M = \mathcal M_m =\mathcal M_c$, we have
\[
\mathcal{R}_{\ell \mid \Omega}(\phi^{\text{loc}}_{\text{full}\mid \Omega},\phi^{\text{loc}}_{\text{pure}\mid \Omega}) 
=
\sum_{i\in D}(\phi^\text{sens}_{\text{full}\mid \Omega}(i)-\phi^{\text{sens}}_{\text{pure}\mid \Omega}(i)).
\]

\end{theorem}

Importantly, with feature dependencies or other interaction operators, squared local disagreement and sensitivity disagreement do not necessarily coincide; though both GRANITE objectives are variance-based, they can yield different explanations in practice.

\paragraph{\correlation{Minimizing Distribution Influence.}}
We now consider interaction operators $\mathcal I=\mathcal{I}_1 = \mathcal I_2$ and compare two regional explanations $\phi_{1\mid\Omega}$, $\phi_{2\mid\Omega}$ that differ only in their masking operators with $\mathcal{M}^\Omega_1 = \mathcal M_c^\Omega$ and $\mathcal M^\Omega_2 =\mathcal M_m^\Omega$. 
In this setting, the difference between regional explanations reflects the differences in conditional and marginal distribution caused by feature dependencies, which GRANITE aims to minimize:

\begin{theorem}\label{thm:regional-distributional-differences}
Let $w^\Omega_{x_S} := \frac{dP^\Omega_{X_{-S}|X_S=x_S}}{d P^\Omega_{X_{-S}}}$ be well-defined with density $dP$, $\mathcal M^\Omega_1=\mathcal M^\Omega_c$, $\mathcal M^\Omega_2 =\mathcal M^\Omega_m$ with $\mathcal B^\Omega_1 = \mathcal B^\Omega_2 = \mathcal B_{x,\text{loc}}$. 
Then, $\phi_{1\mid\Omega}-\phi_{2\mid\Omega} \equiv \mathcal I (\mathcal B^\Omega_{loc,x_0}\delta^\Omega)$, where
\begin{align*}
    \delta^\Omega(x,S) := \mathbb{E}_{P^\Omega_{X_{-S}}}[(w^\Omega_{x_S}(X_{-S})-1) F(x_S,X_{-S})].
\end{align*}
In words, the disagreement is equal to the local explanation of $\delta^\Omega$, driven by the difference between conditional and marginal distribution.
For $\mathcal B^\Omega_{\text{sens}}$ and $\mathcal B^\Omega_{\text{risk}}$, similar variants of $\delta^\Omega$ are presented in Appendix~\ref{appx_proof_distributions}.
\end{theorem}

\textit{Illustration.}
In the example from Table~\ref{tab:example}, we fix $\mathcal{I}=\mathcal{I}_{\text{pure}}$ and compare the marginal and conditional effects for $X_4$, which differ only in $\mathcal{M}$. Across the feature space, the dependence between $X_3$ and $X_4$ introduces an extra term \correlation{${\mu_{3 \mid 4}}$} in the conditional effect. By Theorem~\ref{thm:regional-distributional-differences}, this disagreement stems solely from the masking distribution. GRANITE’s partitioning along $X_3$ minimizes this difference within regions, aligning marginal and conditional effects.

\subsection{Extension to Feature Groups}\label{sec:feature_groups}

Definition~\ref{def:regional_disagreement_risk} introduced disagreement for individual influence measures $\phi: D\to \mathbb{R}$.
This disagreement generalizes to group explanations $\phi:\mathcal S \to \mathbb{R}$, where $\mathcal S \subseteq 2^D$, by summing over $\sum_{S \in \mathcal S}$ instead of $\sum_{i \in D}$.
Following \cite{fumagalli2025unifying}, we distinguish joint influence of a group from its interaction influence. 
Joint influence only minimizes interactions or distributional effects between groups, which is especially meaningful when features are naturally interpreted together---e.g., genetic \citep{lozano2009grouped} or sensor \citep{chakraborty2008selecting} data. 
Interaction influence instead targets only higher-order interactions or distributional effects. By restricting minimization to interactions above a chosen order, lower-order effects remain intact and interpretable, reducing the number of regions while preserving decomposability up to that order.

\begin{corollary}\label{theorem:feature_groups}
Theorems~\ref{theorem:min_interactions} and \ref{thm:regional-distributional-differences} extend directly to joint and interaction influence measures.
\end{corollary}

\textit{Illustration.}
Let $X_i \sim \mathbb{N}(0,1)$, $i = 1,\ldots,4$, and $\textstyle F(X) = (X_1 + 0.7X_1X_2)\cdot \mathds{1}{X_3 \geq 0} + (-X_1 + 0.7X_1X_4)\cdot \mathds{1}{X_3 < 0}$. We apply GRANITE to the interaction influence of feature pairs, with local behavior $\mathcal{B}_{\text{loc}, x_0}$ and marginal masking $\mathcal{M}_m$, restricting minimization to interactions above order two. On $2000$ samples, GRANITE splits at $X_3 = 0$, yielding two regions where higher-order interactions vanish. The regional decompositions, $F_{\mid\Omega_l}(x) = -x_1 + 0.7 x_1 x_4$ and $F_{\mid\Omega_r}(x) = x_1 + 0.7 x_1 x_2$, remain interpretable and require fewer splits (Figure~\ref{fig:interaction_example}). This is especially useful in the presence of continuous interactions, where existing regional frameworks encounter limitations \citep{herbinger2022repid, herbinger2024gadget, laberge2024tackling}. An analogous approach applies to feature dependencies, and a joint influence example is discussed in Appendix~\ref{app:experiments}.
\begin{figure}[tbh]
    \centering
    \includegraphics[width=\linewidth]{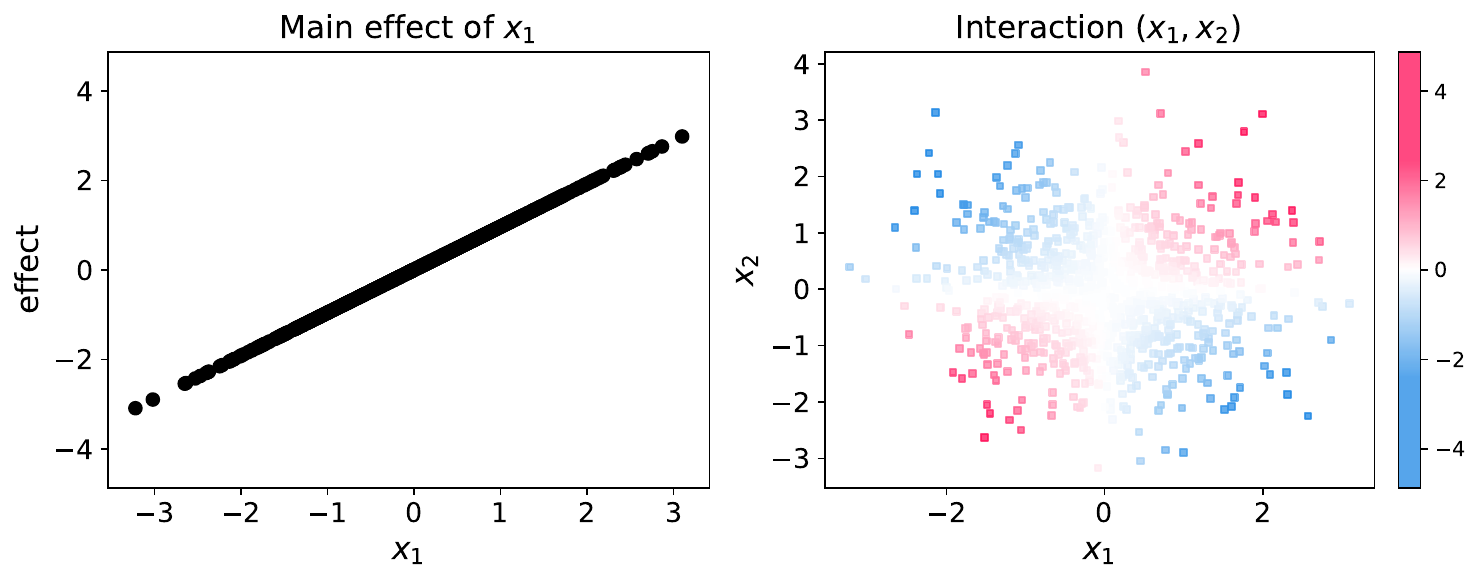}
    \caption{Main effect of $x_1$ (left) and its interaction with $x_2$ (right) in region $\Omega_r$, where $F_{\mid\Omega_r}(x) = x_1 + 0.7\,x_1x_2$.}.
    \label{fig:interaction_example}
\end{figure}

\subsection{Estimations}\label{sec:estimation}
In practice, minimizing the disagreement via Equation \eqref{eq:objective} requires: 1) estimating
explanations $\phi_{\mid \Omega}$
from data
and 2) efficiently exploring the space of partitions.

\textbf{Estimating Explanations.}
For computational efficiency, we consider regional explanations $\phi_{\mid \Omega} := \mathcal I(\mathcal B^\Omega (\mathcal M^\Omega F))$ with \textbf{pure} ($\mathcal I_{\text{pure}})$ and \textbf{full} ($\mathcal I_{\text{full}}$) interaction operators.
This reduces evaluations of $\nu^\Omega = (\mathcal B^\Omega( \mathcal M^\Omega_m F))$ to the subsets $\emptyset, i, -i$, and $D$.
While GRANITE can be used with $\mathcal I=\mathcal I_{\text{partial}}$, minimizing differences between pure and full effects, also reduces the disagreement with partial effects.
 
To estimate $\phi_{\mid \Omega}$, we first estimate $P$ with the empirical distribution over $N$ data points $\{x^{(n)}\}_{n=1}^N$. 
For $S\subseteq D$, we compute all possible combinations
\begin{equation}
   R^S_{n\tilde n} = F(x^{(n)}_S, x^{(\tilde n)}_{-S}), \quad n,\tilde n=1,\dots,N,
\end{equation}
and estimate the marginal masking operator with
\begin{equation}
    (\mathcal{M}_m^\Omega F)(x^{(n)}, S) = \frac{1}{\vert \{\tilde n: x^{(\tilde n)} \in \Omega\} \vert}\sum_{x^{(\tilde n)} \in \Omega} R^S_{n \tilde n},
    \label{eq:marginal_approx}
\end{equation}
For conditional masking, we refer to the Appendix~\ref{app:algorithm}.
Due to the symmetry $\bm{R}^{-i} = (\bm{R}^{i})^T$, GRANITE only needs to compute $\{\bm{R}^{i}\}_{i=1}^d$ to evaluate $\nu^\Omega(i)$ and $\nu^\Omega(-i)$, since this symmetry is preserved when applying the behavior operators $\mathcal{B}^\Omega$ over the instances in $\Omega$.
Moreover, $\nu^\Omega(\emptyset)$ and $\nu^\Omega(D)$ are directly estimated from the predictions.
Further details, including space complexity, are provided in the Appendix~\ref{app:algorithm}.

\textbf{Computing Partitions.}
Solving Equation \eqref{eq:objective} is intractable since the space of partitions of $\mathcal{X}$ is infinite.
Existing regional explanation frameworks address this via clustering or 
recursive partitioning \citep{britton2019vine, herbinger2024gadget, laberge2024tackling, molnar2023model}. 
Clustering often yields non-interpretable regions \citep{herbinger2022repid}, whereas recursive partitioning produces well-defined ones. GRANITE supports any method that identifies distinct regions under our risk constraint.

We adopt a recursive partitioning approach, constraining regions to be the leaves of a binary decision tree with bounded
depth, where each leaf contains at least $N_{\text{min}}$ samples. To minimize disagreement risk, the feature space is recursively split into two subregions, $\Omega_-$ and $\Omega_+$, until any of these constraints is reached.
A bottom-up pruning step based on a regularization term $\alpha K$ with $\alpha \ge 0$ and $K$ being the number of leaves, undoes splits that do not sufficiently reduce risk.

To split along any of the $d$ features, the algorithm considers at most $B$ split candidates.
For each candidate, regional explanations on $\Omega_-$ and $\Omega_+$ are computed using the pre-computed matrices $\bm{R}^S$, and the
corresponding disagreement is evaluated. The optimal split is selected and the procedure repeats.

\emph{Computationally}, at most $\mathcal{O}(2^{\text{max-depth}})$ internal nodes are evaluated, each with $\mathcal{O}(dB)$ split candidates.
Since computing the disagreement risk is $\mathcal{O}(dN^2)$, the overall partitioning complexity is
$\mathcal{O}(2^{\text{max-depth}} d^2 N^2 B)$. The full procedure is summarized in Appendix \ref{app:algorithm}.

\section{REGIONAL UNIFICATION}
Having introduced GRANITE, we now review and unify regional explanation methods.

\begin{table*}[!t]
\caption{Regional explanation frameworks and their unification within GRANITE, showing which disagreement types and behaviors each approach addresses. \cmark: fulfilled, \xmark: not fulfilled, $*$: all possible choices, --: not defined.}\label{tab:overview_frameworks}
\centering
\begin{tabular}{l|c|c|c|c|c|c|c|c}
\toprule
&\multicolumn{3}{|c|}{\textbf{Disagreement}} & \multicolumn{3}{c|}{\textbf{Behavior}} &  & \\
\hline
\textbf{Framework} & \interaction{$\mathcal I_{\text{full}}\rightarrow\mathcal I_{\text{pure}}$} & \interaction{$\mathcal I_{\text{partial}}\rightarrow\mathcal I_{\text{pure}}$} & \correlation{$\mathcal M_c\rightarrow \mathcal M_m$} & $\mathcal B_{loc}$ & $\mathcal B_{risk}$ & $\mathcal B_{sens}$ & $\mathcal M$ & $\mathcal I$ \\
\hline
VINE & \cmark & \xmark & -- & \cmark & \xmark & \xmark & $m$  & -- \\
CohEx & \xmark & \cmark & -- & \cmark & \xmark & \xmark &  $m$ & --  \\
REPID & \cmark & \xmark & -- & \cmark & \xmark & \xmark  & $m$ & -- \\
GADGET & \cmark & \cmark & -- & \cmark & \xmark & \xmark & $m, c$  & -- \\
FD-Trees & \cmark & \cmark & -- & \cmark & \xmark & \xmark  & $m$ & --  \\
Transf. Trees & -- & -- & \cmark & \cmark & \cmark & \xmark & -- & full, pure  \\\midrule
GRANITE & \cmark & \cmark & \cmark & \cmark & \cmark & \cmark & $*$  & $*$ \\
\bottomrule
\end{tabular}
\end{table*}

\textbf{Related Work.}
Prior research has shown that feature-based explanations disagree due to interactions or differences between marginal and conditional perturbations \citep{krishna2024disagreement, herbinger2022repid, chen_true_2020}.
Existing regional frameworks focus on the reduction of disagreement caused by interactions. 
Heuristic approaches, such as VINE \citep{britton2019vine} and CohEx \citep{meng2024cohex}, cluster ICE curves or SHAP values.
REPID \citep{herbinger2022repid} recursively partitions ICE curves with functional ANOVA guarantees, GADGET \citep{herbinger2024gadget} generalizes this to other feature effect methods, and FD-Trees \citep{laberge2024tackling} align PDP, SHAP, and PFI via functional decomposition.
With GRANITE, we provide a principled approach via disagreement to minimize interactions within regions.

In contrast, methods targeting distributional effects are limited: transformation trees \citep{molnar2023model} reduce feature dependencies to align PDPs with MPlots and PFI with CFI.
With GRANITE and our notion of regional disagreement, we provide a mathematically rigorous framework to identify regions that minimize feature dependencies.

\textbf{Regional Unification in GRANITE.}
Above listed regional explanation frameworks are special cases of GRANITE. All minimize disagreement for individual measures but differ in targeting interactions or distributional influence (see Table \ref{tab:overview_frameworks}). Below, we show how each fits into GRANITE, with details in Appendix~\ref{app:unification}.

\emph{\interaction{Minimizing Interactions:}}
VINE and CohEx cluster local explanations under marginal masking: VINE minimizes full-to-pure and CohEx partial-to-pure. REPID, GADGET, and FD-Trees instead use recursive partitioning: REPID targets full-to-pure under marginal masking, GADGET and FD-Trees additionally reduce partial-to-pure, and GADGET also handles full-to-pure under conditional masking.

\emph{\correlation{Minimizing Distribution Influence:}}
transformation trees minimize conditional-to-marginal influence for pure local effects and full global risk behavior.

In summary, existing methods restrict to individual measures and one type of disagreement, each covering only certain behaviors, masking schemes, or interaction operators. GRANITE unifies these cases, closes the gaps for individual measures, and extends beyond them to joint and interaction influence measures.

\section{EXPERIMENTS}\label{sec:experiments}
GRANITE is highly flexible: with three influence types (individual, joint, and interaction) and three behaviors, it provides nine settings to minimize interaction or distribution disagreement, each suited to different interpretability goals. Here, we illustrate several of these on the \textbf{Bikesharing} dataset, addressing key interpretability questions. Additional real-world examples, including interaction-focused analyses and joint influence measures, are provided in Appendix~\ref{app:experiments}; the full implementation and experimental code are available at \url{https://github.com/gablabc/GRANITE}.
The Bikesharing dataset contains hourly bike rentals in the Capital Bikeshare system (2011–2012). We train a gradient boosting model on $80\%$ of the data using ten temporal and weather-related features (see Appendix~\ref{app:experiments} for details). The remaining $20\%$ serves as a test set.
GRANITE is applied with different objectives on the first $N=1000$ test samples to demonstrate how it adapts to distinct 
interpretability questions. Unless stated otherwise, the default hyperparameters $\alpha=0.05,N_{\text{min}}=20,B=40$ are used.

\interaction{Simpler Model Explanations and More Interpretable Decompositions.}
A key challenge in model interpretation is that feature interactions complicate explanations. ICE plots reveal interactions but not their form or strength, 
while PDPs provide simple additive insights but ignore interactions, leading to incomplete understanding. The most interpretable explanations arise when the prediction function can be expressed 
primarily through low-order terms.

We demonstrate this on the Bikesharing data using \textbf{marginal local effects}. GRANITE minimizes interaction disagreement between \textbf{ICE (full effects)} and \textbf{PDP (pure effects)}, producing three regions with strongly reduced disagreement. In \textbf{Figure \ref{fig:bike_a}}, partitioning by \emph{workingday} (and \emph{hour}) reveals distinct patterns for \emph{hour}: bike use peaks during rush hours on working days but in the afternoon on non-working days, explaining much of the disagreement. Remaining disagreement reflects a three-way interaction between \emph{hour}, \emph{temperature}, and \emph{workingday}, which reduces regionally to a continuous hour–temperature effect.
To avoid excessive partitioning, we next limit GRANITE to \textbf{minimizing effects beyond second order}. It splits on \emph{workingday} which isolates the hour–temperature interaction (\textbf{Figure \ref{fig:bike_b}}).
Temperature effects are lessened in the morning and night, amplified in the afternoon on non-working days, and amplified during commute times on working days.

\begin{figure*}[!t]
    \centering
    \includegraphics[width=\linewidth]{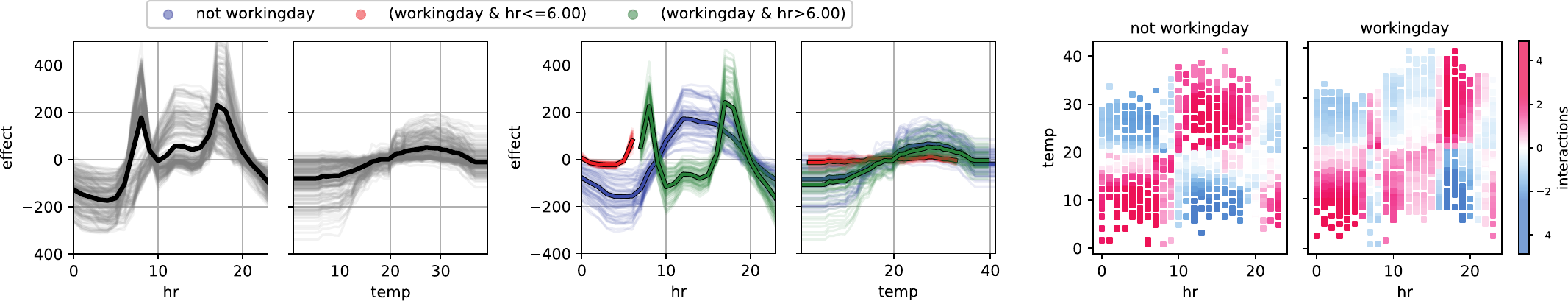}
    
    \begin{subfigure}[t]{0.65\linewidth}
        \centering
        \vspace{-1.25em}
        \caption{Global and regional individual feature effects}
        \label{fig:bike_a}
    \end{subfigure}
    \hfill
    \begin{subfigure}[t]{0.34\linewidth}
        \centering
        \vspace{-1.25em}
        \caption{Interaction effects}
        \label{fig:bike_b}
    \end{subfigure}
    \\[0.25em]
    \caption{\interaction{Minimizing interaction disagreement} with $\mathcal{B}_{loc,x_0}$ and $ \mathcal{M}_m$ for: \textbf{(a) Individual feature influence} with full (ICE curves as thin lines) versus pure (PDP as thick line) shown globally (left) and regionally (right) for \emph{hour} and \emph{temperature}. \textbf{(b) Interaction feature influence} of \emph{hour} and \emph{temperature} for regions depending on \emph{workday}.}
    \label{fig:bike}
\end{figure*}
\begin{figure*}[!t]
    \centering
    \includegraphics[width=.93\textwidth]{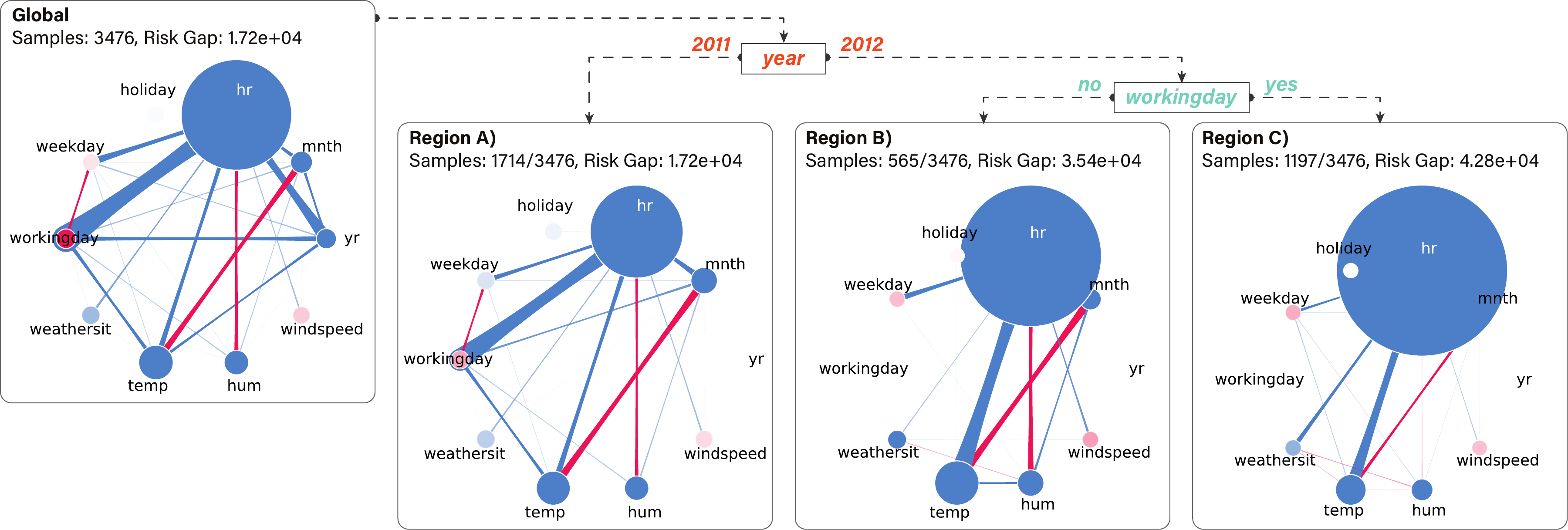}
    \vspace{0.5em}
    \caption{\interaction{Minimizing interaction disagreement} with $\mathcal{B}_{\text{risk}}$ and $\mathcal{M}_m$: GRANITE minimizes full ($\mathcal I_{\text{full}})$ versus pure ($\mathcal I_{\text{pure}}$) global risk, while the plots show first- and second-order partial risk (SAGE-like interactions) globally and regionally.}
    \label{fig:sage}
\end{figure*}

\interaction{Understanding the Role of Interactions in Global Feature Importance.} 
Global feature importance can be sensitivity-based, measuring explained variance, or performance-based, assessing loss when a feature is removed. For both, rankings depend on how interactions are handled. Full measures such as PFI include interactions, obscuring the contribution of main effects and how importance varies across subgroups.

In the Bikesharing data, \emph{workingday} contributes only via interactions, while \emph{hour}, \emph{temperature}, and \emph{year} differ between \textbf{full and pure importance}, producing inconsistent global rankings based on \textbf{Sobol variances (marginal sensitivity; Figure~\ref{fig:global_sens})} \citep{Sobol_2001}. GRANITE disentangles these effects by partitioning on \emph{workingday} and early vs. daytime hours, revealing \textit{which} features matter and \textit{where}: early morning shows no importance, \emph{temperature} dominates \emph{year} on non-working days after morning, and \emph{year} dominates on working days. 
Applied to \textbf{marginal full vs. pure risk (Figure \ref{fig:sage})}, GRANITE shows performance differences e.g., small risk gaps for $yr=2011$ vs. larger gaps for $yr=2012$ on working days. 
First- and second-order Shapley interactions \citep{bordt2022shapley} extend SAGE \citep{covert2020sage} and reveal that regions mitigate hour–workingday/year interactions, while hour–temperature synergy remains strong, requiring both features for accurate predictions.

\correlation{Explanations that Remain True to the Data.}
A key interpretability challenge is ensuring explanations remain faithful to the underlying data distribution. 
Conditional explanations account for feature dependencies but are often complex and computationally demanding, 
whereas marginal explanations are simpler but may misrepresent model behavior.

We illustrate this by comparing \textbf{pure local effects} using \textbf{MPlots} \citep{friedman_greedy_2001, apley_visualizing_2020} \textbf{(conditional)} and \textbf{PDPs (marginal)} for \emph{temperature} and \emph{humidity} (\textbf{Figure \ref{fig:mplot_pdp}}). Large discrepancies between marginal and conditional effects occur at low and high feature 
values, highlighting the impact of ignored dependencies in PDPs. GRANITE resolves these differences by partitioning the feature 
space based on \emph{hour} and \emph{temperature}, after which MPlot and PDP results align closely, producing regionally consistent marginal interpretations that remain faithful to both the model and the data distribution. 
A similar analyses is performed for CFI vs. PFI (full conditional versus full marginal risk) in Appendix~\ref{app:experiments}.

\textbf{Quantitative Analysis.}
Building on the qualitative insights from the visualizations, we now quantify disagreement reduction on the Bikesharing dataset across different models (GBT, MLP, RF) 
and increasing region complexity, controlled by the maximum tree depth and  setting $\alpha=0.01$.
We evaluate interaction disagreement using ICE vs.\ PDP (local, Figure~\ref{fig:bike_a}) and PFI vs.\ pure risk (global, Figure~\ref{fig:sage}), as well as 
distributional disagreement using CFI vs.\ PFI and MPlot vs.\ PDP (masking, Figure~\ref{fig:mplot_pdp}). Across all settings, GRANITE consistently reduces disagreement,
with results reported in Table~\ref{tab:regions_depth}. Disagreement decreases with depth and is most pronounced at Depth~3; for example, in the CFI vs.\ PFI setting,
features become nearly uncorrelated within the learned regions, causing PFI and CFI to coincide in both value and interpretation.

\begin{table*}[!t]
\centering
\caption{Disagreement reduction across depth levels on the Bikesharing dataset for interaction and distribution disagreements under different models and losses. Values report remaining disagreement (in \%) relative to Depth~0 (entire feature space $\mathcal{X}$) on the test set, with regions learned on the training set. Results are averaged over 5 random seeds ($\pm$ std).}
\label{tab:regions_depth}
\begin{tabular}{lllllcccc}
\toprule
\textbf{Target} & \textbf{Behavior} & \textbf{Model} & \textbf{Loss} & \textbf{Depth 0} & \textbf{Depth 1} & \textbf{Depth 2} & \textbf{Depth 3} \\
\midrule
\interaction{$\mathcal{I}_{\text{full}}$ vs $\mathcal{I}_{\text{pure}}$} & $\mathcal{B}_{loc}$ & GBT & ICE vs PDP & 100 & $30 \pm 1.64$ & $12 \pm 1.22$ & $4 \pm 0.71$ \\
\interaction{$\mathcal{I}_{\text{full}}$ vs $\mathcal{I}_{\text{pure}}$} & $\mathcal{B}_{loc}$ & MLP & ICE vs PDP & 100 & $26 \pm 2.88$ & $10 \pm 1.00$ & $3 \pm 0.55$ \\
\interaction{$\mathcal{I}_{\text{full}}$ vs $\mathcal{I}_{\text{pure}}$} & $\mathcal{B}_{loc}$ & RF  & ICE vs PDP & 100 & $30 \pm 1.67$ & $11 \pm 1.48$ & $4 \pm 0.71$ \\
\interaction{$\mathcal{I}_{\text{full}}$ vs $\mathcal{I}_{\text{pure}}$} & $\mathcal{B}_{risk}$ & GBT & PFI vs Pure Risk & 100 & $42 \pm 5.76$ & $16 \pm 2.68$ & $6 \pm 0.84$ \\
\interaction{$\mathcal{I}_{\text{full}}$ vs $\mathcal{I}_{\text{pure}}$} & $\mathcal{B}_{risk}$ & MLP & PFI vs Pure Risk & 100 & $36 \pm 1.82$ & $12 \pm 1.67$ & $3 \pm 0.71$ \\
\interaction{$\mathcal{I}_{\text{full}}$ vs $\mathcal{I}_{\text{pure}}$} & $\mathcal{B}_{risk}$ & RF  & PFI vs Pure Risk & 100 & $41 \pm 7.86$ & $15 \pm 1.14$ & $4 \pm 0.55$ \\
\correlation{$\mathcal{M}_{c}$ vs $\mathcal{M}_{m}$} & $\mathcal{B}_{loc}$ & GBT & MPlot vs PDP & 100 & $36 \pm 3.70$ & $16 \pm 4.72$ & $11 \pm 2.55$ \\
\correlation{$\mathcal{M}_{c}$ vs $\mathcal{M}_{m}$} & $\mathcal{B}_{loc}$ & MLP & MPlot vs PDP & 100 & $37 \pm 5.36$ & $15 \pm 3.08$ & $11 \pm 4.62$ \\
\correlation{$\mathcal{M}_{c}$ vs $\mathcal{M}_{m}$} & $\mathcal{B}_{loc}$ & RF  & MPlot vs PDP & 100 & $37 \pm 4.55$ & $18 \pm 3.70$ & $16 \pm 3.96$ \\
\correlation{$\mathcal{M}_{c}$ vs $\mathcal{M}_{m}$} & $\mathcal{B}_{risk}$ & GBT & CFI vs PFI & 100 & $26 \pm 7.53$ & $6 \pm 1.41$ & $3 \pm 1.34$ \\
\correlation{$\mathcal{M}_{c}$ vs $\mathcal{M}_{m}$} & $\mathcal{B}_{risk}$ & MLP & CFI vs PFI & 100 & $21 \pm 1.82$ & $9 \pm 2.41$ & $2 \pm 0.55$ \\
\correlation{$\mathcal{M}_{c}$ vs $\mathcal{M}_{m}$} & $\mathcal{B}_{risk}$ & RF  & CFI vs PFI & 100 & $22 \pm 1.64$ & $10 \pm 3.49$ & $5 \pm 3.78$ \\
\bottomrule
\end{tabular}
\end{table*}

\begin{figure}[tbh]
    \centering
    \vspace{-0.6em}
    \includegraphics[width=\linewidth]{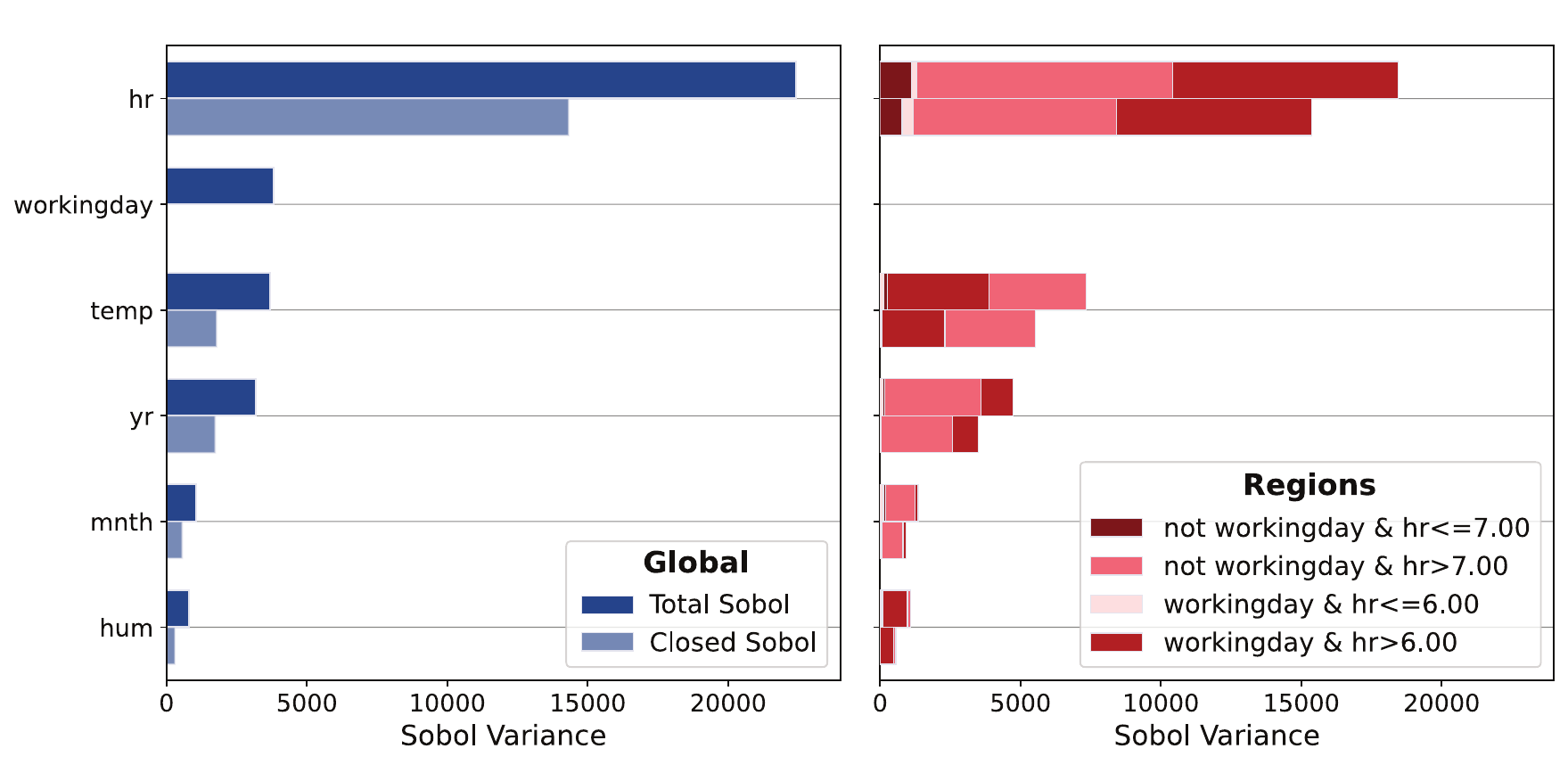}
    \\[0.4em]
    \caption{\interaction{Minimizing interaction disagreement} with $\mathcal{B}_{sens}$ and $\mathcal{M}_m$: Total Sobol' ($\mathcal I_{\text{full}}$, upper bar) and closed Sobol' ($\mathcal I_{\text{pure}}$, lower bar) shown globally (left) and regionally (right) for top-$6$ features.
    Stacked bars sorted by size for the regions create a layered view of regional rankings.}
    \label{fig:global_sens}
\end{figure}
\begin{figure}[h]
    \centering
    \includegraphics[width=\linewidth]{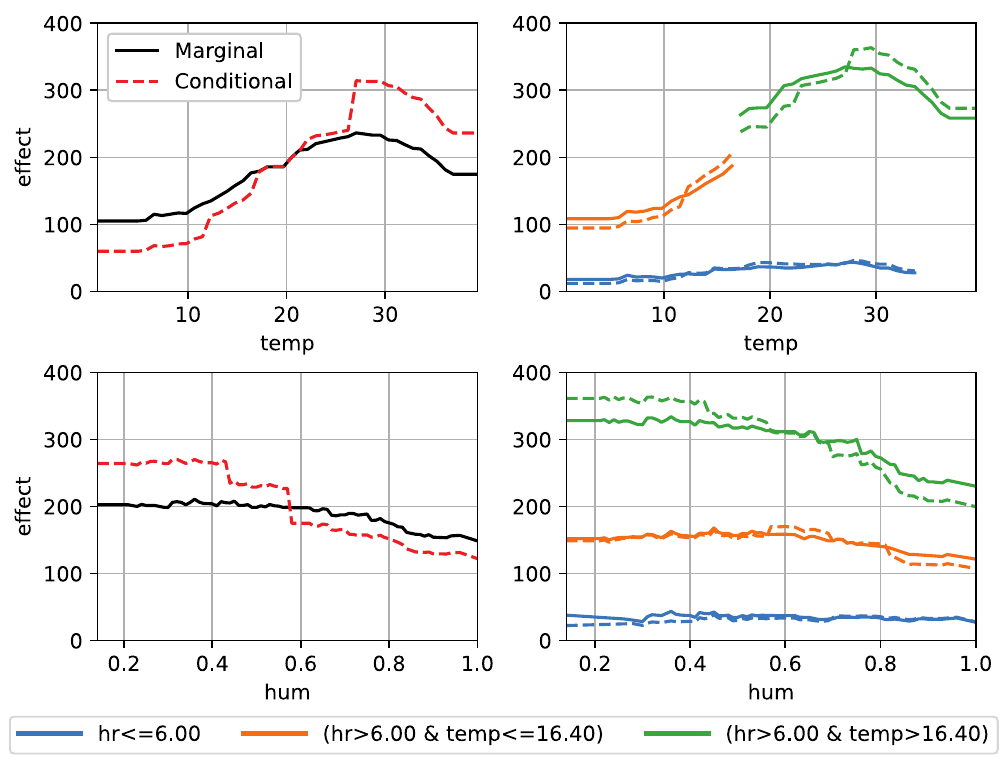}
    \\[0.4em]
    \caption{\correlation{Minimizing distribution disagreement} with $\mathcal{B}_{loc,x_0}$ and $\mathcal{I}_{pure}$ between MPlot ($\mathcal M_c$) and PDP ($\mathcal M_m$). 
    Global plots (left) disagree, while regional plots (right) agree for \emph{temperature} (top) and \emph{humidity} (bottom).}
    \label{fig:mplot_pdp}
\end{figure}

\section{DISCUSSION AND LIMITATIONS}
We presented GRANITE, a principled framework to align feature-based explanations by minimizing disagreement of different choices of masking ($\mathcal M$) and interaction ($\mathcal I$) operators. Our theoretical results demonstrated that this disagreement is caused by interactions for $\mathcal I$ and feature dependencies for $\mathcal M$.
GRANITE aligns explanations by efficiently partitioning the feature space into regions, and applies to local and global explanations ($\mathcal B$), as well as feature groups.

\paragraph{Limitations and Future Work.}
First, GRANITE works best when effects change abruptly, yielding a small number of clearly interpretable regions. 
When effects vary smoothly, many regions may be required, reducing interpretability.
GRANITE partially mitigates this through explicit control of region complexity via the partitioning depth and by increasing expressiveness through feature groups (e.g., including pairwise interactions), which can reduce the number of required regions.
Selecting such feature groups in a principled manner remains an important direction for future work.
Second, we minimized interaction or distributional influences separately, but both can be addressed simultaneously, enabling low-order additive decompositions under independence.
However, achieving full agreement may require many splits, highlighting a trade-off between interpretability and alignment.
Third, memory usage grows quadratically with the number of data points used to estimate explanations. While sub-sampling $1000$ points produced reasonable regions,
future work could
choose subsample sizes based on uncertainty in the regional explanations \citep{molnar2023relating}.
Finally, while applied on tabular data, GRANITE can be extended to 
concepts for vision models or time-dependent prediction problems.

\clearpage
\section*{ACKNOWLEDGEMENTS}
We gratefully thank the anonymous reviewers for their valuable feedback for improving this work.
Julia Herbinger and Marvin N. Wright gratefully acknowledge funding by the German Research Foundation (DFG), Emmy Noether Grant 437611051.
Maximilian Muschalik and Fabian Fumagalli gratefully acknowledge funding by the Deutsche Forschungsgemeinschaft (DFG, German Research Foundation): TRR 318/1 2021 – 438445824. Yann Pequignot expresses appreciation for the support from the DEEL Project CRDPJ 537462-18 funded by the Natural Sciences and Engineering Research Council of Canada (NSERC) and the Consortium for Research and Innovation in Aerospace in Québec (CRIAQ), together with its industrial partners Thales Canada inc, Bell Textron Canada Limited, CAE inc and Bombardier inc.\footnote{\url{https://deel.quebec}}
The authors of this work take full responsibility for its content.

\renewcommand{\refname}{REFERENCES}
\bibliography{bib}

\newpage
\appendix
\onecolumn 

\startcontents[sections]
\printcontents[sections]{l}{1}{\setcounter{tocdepth}{3}}

\clearpage

\section{PROOFS}\label{app:proofs}

In this section, we provide detailed proofs of the results stated in the main paper.

\subsection{Proof of Theorem 1}\label{app:proof_thm1}

\begin{proof}
The value function of both regional explanations are defined by $\nu \equiv \mathcal B^\Omega(\mathcal M^\Omega F)$ with $ \mathcal M^\Omega = \mathcal M_1^\Omega = \mathcal M_2^\Omega$.
The pure interactions of the value function are given by
\begin{align*}
\Delta(S) := \sum_{L \subseteq S} (-1)^{\vert S \vert - \vert L \vert} \nu(L),
\end{align*}
and for each interaction operator we obtain for $\phi_{\mid \Omega} \equiv \mathcal I \nu$ by \citet{fumagalli2025unifying}[Table 2]
\begin{align*}
    \phi_{\mid \Omega}(i) = \Delta(i) + \sum_{S\subseteq D, i \in S, \vert S \vert \geq 2} w^S \Delta(S)
\end{align*}
with 
\begin{align*}
    w^S := \begin{cases}
        0, &\text{for } \mathcal I = \mathcal I_{\text{pure}},
        \\
        \frac{1}{\vert S \vert}, &\text{for } \mathcal I = \mathcal I_{\text{partial}},
        \\
        1 &\text{for } \mathcal I = \mathcal I_{\text{full}}.
    \end{cases}
\end{align*}

Therefore, we obtain for two regional explanations
\begin{align*}
    \phi_{1\mid\Omega}(i) - \phi_{2\mid\Omega}(i) = \sum_{S\subseteq D, i \in S, \vert S \vert \geq 2} \tilde w^S \Delta(S),
\end{align*}
where $\tilde w^S := w_1^S - w_2^S$, where $w_1^S,w_2^S$ are the weights associated with $\phi_{1\mid\Omega},\phi_{2\mid\Omega}$, respectively.
\end{proof}

\subsection{Proof of Theorem 2}

\begin{proof}
    We let $\ell(x) := x^2$, and compute regional explanations with respect to the local value function $\nu_x \equiv \mathcal B_{x,\text{loc}}(\mathcal M F)$ with $\mathcal M = \mathcal M_m = \mathcal M_c$.
    The difference between the two local regional explanations is given by the previous results in the proof of Theorem~\ref{theorem:min_interactions} as
    \begin{align*}
        \phi^{\text{loc}}_{\text{full}\mid\Omega}(i) - \phi^{\text{loc}}_{\text{pure}\mid\Omega}(i) = \sum_{S \subseteq D, i \in S,\vert S \vert \geq 2} \Delta_x(S),
    \end{align*}
    since $\tilde w^S \equiv 1$ in this case, where we have added the dependency on $x$ in $\Delta$ as a subscript.
    Due to feature independence, the pure interaction effects $\Delta_x(S)$ directly correspond to the fANOVA effects $f_S(x)$ in this case, cf \citep{fumagalli2025unifying}[Corollary 1], which are are orthogonal and mean-centered yielding the variance decomposition property \citep{Owen_2013}.
    We can thus compute the regional disagreement as
    \begin{align*}
        \mathcal R_{\ell \mid \Omega}(\phi^{\text{loc}}_{\text{full}\mid\Omega},\phi^{\text{loc}}_{\text{pure}\mid\Omega}) &= \mathbb{E}_{P_X}\bigg[\sum_{i \in D} \ell(\phi^{\text{loc}}_{\text{full}\mid\Omega}(i) - \phi^{\text{loc}}_{\text{pure}\mid\Omega}(i))\bigg] 
        \\
        &= \sum_{i \in D} \mathbb{E}_{P_X}\bigg[\ell(\sum_{S \subseteq D, i \in S,\vert S \vert \geq 2} \Delta(S))\bigg]
        \\
        &= \sum_{i \in D} \mathbb{V}_{P_X}\bigg[\sum_{S \subseteq D, i \in S,\vert S \vert \geq 2} \Delta(S)\bigg] 
        \\
        &= \sum_{i \in D} \sum_{S \subseteq D, i \in S,\vert S \vert \geq 2} \mathbb{V}_{P_X}[\Delta(S)].
    \end{align*}
    On the other hand, we obtain for the sensitivity behavior $\mathcal B_{\text{sens}}$ with value function $\nu \equiv \mathbb{V}_{P_X}[\nu_X(\cdot)]$
    \begin{align*}
        \phi^{\text{sens}}_{\text{full}\mid\Omega}(i) - \phi^{\text{sens}}_{\text{pure}\mid\Omega}(i) = \mathbb{V}_{P_X}\bigg[\sum_{S\subseteq D, i \in S} \Delta_X(S)\bigg] - \mathbb{V}_{P_X}[\Delta_X(i)] 
        = \sum_{S\subseteq D, i \in S, \vert S \vert \geq 2}\mathbb{V}_{P_X}[\Delta_X(S)],
    \end{align*}
    where we have used $w^S \equiv 1$ for full effects, and orthogonality with mean-centeredness to apply the variance to the sum.
    We thus obtain
    \begin{align*}
        \mathcal R_{\ell \mid \Omega}(\phi^{\text{loc}}_{\text{full}\mid\Omega},\phi^{\text{loc}}_{\text{pure}\mid\Omega}) = \sum_{i \in D} \sum_{S \subseteq D, i \in S,\vert S \vert \geq 2} \mathbb{V}_{P_X}[\Delta(S)] = \sum_{i \in D}  \left(\phi^{\text{sens}}_{\text{full}\mid\Omega}(i) - \phi^{\text{sens}}_{\text{pure}\mid\Omega}(i)\right),
    \end{align*}
    which concludes the proof.
\end{proof}

\subsection{Proof of Theorem 3}\label{appx_proof_distributions}

\begin{proof}
Let $w^\Omega_{x_S} := \frac{dP^\Omega_{X_{-S}|X_S=x_S}}{d P^\Omega_{X_{-S}}}$ be well-defined with density $dP$, and $\mathcal M^\Omega_1=\mathcal M^\Omega_c$ and $\mathcal M^\Omega_2 =\mathcal M^\Omega_m$. 
Moreover, consider the local behavior operator, i.e., $\mathcal B^\Omega_1 = \mathcal B^\Omega_2 = \mathcal B^\Omega_{x,\text{loc}}$.
Due to additivity of $\mathcal I$ and $\mathcal B^{\Omega}_{x,\text{loc}}$, we directly obtain
\begin{align*}
   \phi_{1\mid\Omega}(i)-\phi_{2\mid\Omega}(i) = (\mathcal I (\mathcal B^\Omega_{loc,x_0}(\mathcal M_c^\Omega F - \mathcal M_m^\Omega F)))(i) \text{ for } i \in D.
\end{align*}
Using $w^\Omega_{x_S}$, we can express the conditional masking operator in terms of the marginal masking operator as
\begin{align*}
    (\mathcal M_c F)(x,S) &= \mathbb{E}_{P^\Omega_{X_{-S}\mid X_S = x_S}}[F(x_S,X_{-S})] 
    = \int_\Omega F(x_S,z) dP^\Omega_{X_{-S}\mid X_S = x_S}(z) 
    \\
    &= \int_\Omega F(x_S,z) w^\Omega_{x_S}(z) dP^\Omega_{X_{-S}}(z)
    = \mathbb{E}_{P^\Omega_{X_{-S}}}[w^\Omega_{x_S}(X_{-S})F(x_S,X_{-S})].
\end{align*}
Hence,
\begin{align*}
    (\mathcal M_c F-\mathcal M_m F)(x,S)= \mathbb{E}_{P^\Omega_{X_{-S}}}[(w^\Omega_{x_S}(X_{-S})-1) F(x_S,X_{-S})] = \delta^\Omega(x,S).
\end{align*}
For local behavior $\mathcal B_{x,\text{loc}}$, we thus obtain
\begin{align*}
   \phi_{1\mid\Omega}(i)-\phi_{2\mid\Omega}(i) = \mathcal I (\mathcal B^\Omega_{loc,x_0}\delta^\Omega)(i) \text{ for } i \in D,
\end{align*}
which concludes the first part of the proof.

\paragraph{Global sensitivity behavior.}
We now consider $\mathcal B^\Omega_1 = \mathcal B^\Omega_2 = \mathcal B^\Omega_{\text{sens}}$ and $\delta^\Omega_\text{loc} := \delta^\Omega$.
We cannot directly obtain the desired result, since, in contrast to $\mathcal B^\Omega_{\text{loc}}$, $\mathcal B^\Omega_{\text{sens}}$ is not additive.
We thus now express the variance of the conditionally masked model through the variance of the marginally masked model:
\begin{align*}
    (\mathcal B^\Omega_{\text{sens}}(\mathcal M^\Omega_c F))(S) &= \mathbb{V}_{P^\Omega_{X_S}}[\mathbb{E}_{P^\Omega_{X_{-S}\mid X_S = X_S}}[F(X_S,X_{-S})]] = \mathbb{V}_{P_{X_S}}[\mathbb{E}_{P^\Omega_{X_{-S}}}[w^\Omega_{X_S}(X_{-S})F(X_S,X_{-S})]]
    \\
    &= \mathbb{V}_{P^\Omega_{X_S}}[\mathbb{E}_{P^\Omega_{X_{-S}}}[F(X_S,X_{-S})] + \delta_{\text{loc}}^\Omega(X,S)]
    \\
    &= \mathbb{V}_{P^\Omega_{X_S}}[\mathbb{E}_{P^\Omega_{X_{-S}}}[F(X_S,X_{-S})]] + \mathbb{V}_{P^\Omega_{X_S}}[\delta_{\text{loc}}^\Omega(X,S)] + 2\mathrm{Cov}(\mathbb{E}_{P^\Omega_{X_{-S}}}[F(X_S,X_{-S})],\delta_{X,\text{loc}}^\Omega(X,S))
    \\
    &= (\mathcal B^\Omega_{\text{sens}}(\mathcal M^\Omega_m F))(S) + \mathbb{V}_{P^\Omega_{X_S}}[\delta_{\text{loc}}^\Omega(X,S)] + 2\mathrm{Cov}(\mathbb{E}_{P^\Omega_{X_{-S}}}[F(X_S,X_{-S})],\delta_{X,\text{loc}}^\Omega(X,S)).
\end{align*}
We then compute
\begin{align*}
    (\mathcal B^\Omega_{\text{sens}}(\mathcal M^\Omega_c F))(S)-(\mathcal B^\Omega_{\text{sens}}(\mathcal M^\Omega_m F))(S)
    &= \mathbb{V}_{P^\Omega_{X_S}}[\delta_{\text{loc}}^\Omega(X,S)] + 2\mathrm{Cov}(\mathbb{E}_{P^\Omega_{X_{-S}}}[F(X_S,X_{-S})],\delta_{X,\text{loc}}^\Omega(X,S))
    \\
    &=: \delta_{\text{sens}}^\Omega(S).
\end{align*}
Hence, the regional disagreement is characterized by the regional explanation of $\delta_{\text{sens}}^\Omega$, which is influenced by the difference between marginal and conditional distribution.

\paragraph{Global risk behavior.}
We now consider $\mathcal B^\Omega_1 = \mathcal B^\Omega_2 = \mathcal B^\Omega_{\text{risk}}$ and the squared loss $\ell(x) := x^2$.
We again express the risk behavior of the conditionally masked model through the risk behavior of the marginally masked model:
\begin{align*}
    (\mathcal B^\Omega_{\text{risk}}(\mathcal M^\Omega_c F))(S) &= -\mathbb{E}_{P^\Omega_{X,Y}}[\ell(\mathbb{E}_{P^\Omega_{X_{-S}\mid X_S = X_S}}[F(X_S,X_{-S})],Y)] 
    \\
    &=  -\mathbb{E}_{P^\Omega_{X,Y}}[\ell(\mathbb{E}_{P^\Omega_{X_{-S}}}[w^\Omega_{X_S}(X_{-S})F(X_S,X_{-S})],Y)]  
    \\
    &=  -\mathbb{E}_{P^\Omega_{X,Y}}[\ell(\mathbb{E}_{P^\Omega_{X_{-S}}}[F(X_S,X_{-S})] + \delta_{\text{loc}}^\Omega(X,S),Y)]  
    \\
    &=  -\mathbb{E}_{P^\Omega_{X,Y}}[(\mathbb{E}_{P^\Omega_{X_{-S}}}[F(X_S,X_{-S})] + \delta_{\text{loc}}^\Omega(X,S)-Y)^2]  
    \\
    &=  -\mathbb{E}_{P^\Omega_{X,Y}}[(\mathbb{E}_{P^\Omega_{X_{-S}}}[F(X_S,X_{-S})]-Y)^2 + \delta_{\text{loc}}^\Omega(X,S)^2 
    \\
    &+ 2 (\mathbb{E}_{P^\Omega_{X_{-S}}}[F(X_S,X_{-S})]-Y) \delta_{\text{loc}}^\Omega(X,S)]  
    \\
    &= (\mathcal B^\Omega_{\text{risk}}(\mathcal M^\Omega_m F))(S) -\mathbb{E}_{P^\Omega_{X,Y}} [\delta_{\text{loc}}^\Omega(X,S)^2 
    + 2 (\mathbb{E}_{P^\Omega_{X_{-S}}}[F(X_S,X_{-S})]-Y) \delta_{\text{loc}}^\Omega(X,S)].
\end{align*}

Hence, we can compute the difference as
\begin{align*}
 (\mathcal B^\Omega_{\text{risk}}(\mathcal M^\Omega_c F))(S) -  (\mathcal B^\Omega_{\text{risk}}(\mathcal M^\Omega_m F))(S) &= -\mathbb{E}_{P^\Omega_{X,Y}} [\delta_{\text{loc}}^\Omega(X,S)^2 + 2 (\mathbb{E}_{P^\Omega_{X_{-S}}}[F(X_S,X_{-S})]-Y) \delta_{\text{loc}}^\Omega(X,S)] 
 \\
 &=: \delta^\Omega_{\text{risk}}(S).
\end{align*}
In conclusion, the regional disagreement is characterized by the regional explanation of the value function $\delta^\Omega_{\text{risk}}$, which is influenced by the difference between marginal and conditional distribution.
\end{proof}

\subsection{Proof of Corollary 1}
\begin{proof}
Theorem~\ref{theorem:min_interactions} and Theorem~\ref{thm:regional-distributional-differences} assumed a feature influence measure for individual features, i.e., $\phi: D \to \mathbb{R}$.
We now extend these results to general influence measures defined for $\phi: \mathcal S \to \mathbb{R}$ with $\mathcal S \subseteq 2^D$.

\subsubsection{Extension of Theorem 1}
We distinguish between interaction effects and joint effects.
\paragraph{Extension to interactions.}
For \emph{interaction effects}, the representation of $\phi_{\mid\Omega}$ for $S \subseteq D$ is directly extended according to \citet{fumagalli2025unifying}[Table 2] as
\begin{align*}
    \phi_{\mid\Omega}(S) = \Delta(S) + \sum_{T\subseteq D, S \subset T} w^T \Delta(T)
\end{align*}
with 
\begin{align*}
    w^S := \begin{cases}
        0, &\text{for } \mathcal I = \mathcal I^{\text{interaction}}_{\text{pure}},
        \\
        \tau^{\vert S\vert}_{\vert T \setminus S \vert}, &\text{for } \mathcal I = \mathcal I^{\text{interaction}}_{\text{partial}},
        \\
        1 &\text{for } \mathcal I = \mathcal I^{\text{interaction}}_{\text{full}},
    \end{cases}
\end{align*}
where $\tau$ is a index-specific weight.
Consequently, the difference between two regional explanations is again characterized by
\begin{align*}
     \phi_{1\mid\Omega}(S) - \phi_{2\mid\Omega}(S) = \sum_{T \subseteq D, S \subset T} \tilde w^T \Delta(T)
\end{align*}
where $\tilde w^T := w_1^T - w_2^T$, where $w_1^T,w_2^T$ are the weights associated with $\phi_{1\mid\Omega},\phi_{2\mid\Omega}$, respectively.

\paragraph{Extension to joint effects.}
Lastly, for \emph{joint effects}, we obtain the representation according to \citep{fumagalli2025unifying}[Table 2] as
\begin{align*}
    \phi_{\mid\Omega}(S) = \sum_{L \subseteq S, L \neq \emptyset}\Delta(L) + \sum_{T\subseteq D, S \cap T \notin \{\emptyset,S,T\}} w^T \Delta(T)
\end{align*}
with
\begin{align*}
    w^S := \begin{cases}
        0, &\text{for } \mathcal I = \mathcal I^{\text{joint}}_{\text{pure}},
        \\
        \tau^{\vert S\vert}_{\vert T \setminus S \vert}, &\text{for } \mathcal I = \mathcal I^{\text{joint}}_{\text{partial}},
        \\
        1 &\text{for } \mathcal I = \mathcal I^{\text{joint}}_{\text{full}},
    \end{cases}
\end{align*}
Therefore, for two regional explanations with joint effects 
\begin{align*}
     \phi_{1\mid\Omega}(S) - \phi_{2\mid\Omega}(S) = \sum_{T \subseteq D, T \cap S \notin \{\emptyset,S,T\}} \tilde w^T \Delta(T)
\end{align*}
where $\tilde w^T := w_1^T - w_2^T$, where $w_1^T,w_2^T$ are the weights associated with $\phi_{1\mid\Omega},\phi_{2\mid\Omega}$, respectively.
Note that in case of joint effects, the difference of regional explanations is also influenced by lower-order interactions  $(\vert T \vert < \vert S \vert)$ that are not fully contained in $S$, i.e., lower-order interactions of features between groups.

\subsubsection{Extension of Theorem 3}
In the proof of Theorem~\ref{thm:regional-distributional-differences}, we have used the additivity of the interaction operator $\mathcal I$.
For joint effects and interaction effects, corresponding to generalized values \citep{DBLP:journals/dam/MarichalKF07} and cardinal-probabilistic interactions \citep{Fujimoto.2006}, respectively, the additivity axiom is satisfied similar to the Shapley value.
Therefore, all previous arguments apply in this case, and the regional disagreement is characterized by the regional explanation of the corresponding $\delta^\Omega$.

\end{proof}

\newpage
\section{DEFINITIONS FOR FEATURE GROUP EXPLANATIONS}\label{app:group_extension}

In this section, we define the interaction operators for feature group explanations, encompassing both \emph{joint} and \emph{interaction} influence measures.

We begin by extending the notion of individual marginal influence for a feature $i \in D$ to the \emph{joint feature influence} of a feature set $S \subseteq D$, defined as
\[
\nu(T \cup S) - \nu(T),
\]
which represents the marginal contribution of the feature group $S$ given the presence of a feature set $T \subseteq -S$.
Building on this, we introduce the \emph{discrete derivative}, which generalizes the notion of pure (interaction) influence to any coalition $T$.
This extends the definition from Theorem~\ref{theorem:min_interactions}, where the pure effect was measured relative to the empty set (via the Möbius transform), to the case of an arbitrary coalition $T$.

Formally, the discrete derivative is defined as
\begin{align*}
    \Delta(S, T) := \sum_{L\subseteq S} (-1)^{|S|-|L|} \nu (T \cup L),
\end{align*}
where $T \subseteq -S$, $\nu(S) := (\mathcal{B}(\mathcal{M} F))(S)$ and $S \subseteq D$.
The discrete derivative $\Delta(S, T)$ quantifies the \emph{pure influence} of a feature group $S$ in the presence of a feature set $T$.
For $T = \emptyset$, it reduces to the Möbius transform \citep{grabisch2000equivalent}, which captures the pure interaction influence and is used in Theorem~\ref{theorem:min_interactions}, where $\Delta(S) = \Delta(S,\emptyset)$.

We now define the interaction operators based on the joint feature influence and the discrete derivative. Definitions based on the Möbius representation are provided in \cite{fumagalli2025unifying}.

\begin{table}[htb]
\centering
\caption{Definitions of joint and interaction influence operators $\mathcal{I}$ for feature groups, expressed in terms of the joint feature influence and the discrete derivative. The weights $w_{S,T}$ denote index-specific coefficients used to aggregate partial operators across coalitions $T \subseteq  -S$.}
\begin{tabular}{l|l|l}
\textbf{Type} &\textbf{Operator} & \textbf{Definition}\\\toprule
joint & $\mathcal I_{\text{pure}}^{\text{joint}}$ & $\nu(S) - \nu(\emptyset)$  \\
joint & $\mathcal I_{\text{partial}}^{\text{joint}}$ & $\sum_{T \subseteq -S} w_{S,T}(\nu (T \cup S) - \nu(T))$  \\
joint & $\mathcal I_{\text{full}}^{\text{joint}}$ & $\nu(D) - \nu(-S)$\\\midrule
interaction & $\mathcal I_{\text{pure}}^{\text{interaction}}$ & $\Delta(S, \emptyset)$  \\
interaction & $\mathcal I_{\text{partial}}^{\text{interaction}}$ &  $\sum_{T \subseteq D \subseteq -S} w_{S,T} \Delta(S,T)$ \\
interaction & $\mathcal I_{\text{full}}^{\text{interaction}}$ & $\Delta(S, -S)$
\end{tabular}
\label{tab:interaction_operators}
\end{table}

\newpage
\section{ALGORITHM DETAILS}\label{app:algorithm}

The \textsc{Granite} algorithm requires two empirical estimates: (1) an estimate of the explanations and (2) an estimate of the optimal partition of regions.
Both estimates are discussed separately, followed by a scalability experiment.

\subsection{Estimating Explanations}\label{app:algorithm:explain}
To compute $\phi_{\mid \Omega}$, we estimate the theoretical distribution $P$ with the empirical distribution $\widehat{P}$ over $N$ data 
points $\{x^{(n)}\}_{n=1}^N$. For $S\subseteq D$, we compute all possible $N\times N$ combinations
\begin{equation}
    R^S_{n\tilde n} = F(x^{(n)}_S, x^{(\tilde n)}_{-S}), \quad n,\tilde n=1,\dots,N,
    \label{eq:R_matrix_app}
\end{equation}
which can be done in a model-agnostic fashion, although more efficient implementations are available for decision trees and Explainable Boosting Machines 
\citep[Chapter 5.4]{laberge2024trustworthy}. Marginal masking using $\widehat{P}^\Omega$ is then computed as
\begin{equation}
    (\mathcal{M}_m^\Omega F)(x^{(n)}, S) = \frac{1}{\vert \{\tilde n: x^{(\tilde n)} \in \Omega\} \vert}\sum_{x^{(\tilde n)} \in \Omega} R^S_{n \tilde n}.
    \label{eq:marginal_approx_app}
\end{equation}

For conditional masking, we assume that bins 
$b_1, b_2, \ldots, b_B$ have been pre-computed for $x_S$. If $S=\{i\}$, these bins may correspond to univariate quantiles of $x_i$.
For $S=D\setminus \{i\}$, bins may instead correspond to the leaves of a Decision Tree trained to predict $x_i$ from $x_{-i}$ \citep{strobl2008conditional}.
Let $t^{(i)}$ be the index of the bin that captures $x^{(i)}$, then the conditional masking is computed as
\begin{equation}
   (\mathcal M_c^\Omega F)(x^{(n)}, S) =\frac{1}{C_i} \sum_{
           \substack{\tilde n : x^{(\tilde n)} \in \Omega\,\wedge\, x^{(\tilde n)}_S \in b_{t^{(i)}}}} R^S_{n \tilde n},
   \label{eq:conditional_approx}
\end{equation}
where $C_i = |\{\tilde n : x^{(\tilde n)} \in \Omega \,\wedge \,x^{(\tilde n)}_S \in b_t^{(i)}\}|$. 

Behavior operators $\mathcal{B}^\Omega$ aggregate these masked evaluations for instances in $\Omega$ via $\nu^\Omega \equiv \mathcal 
B^\Omega( \mathcal M^\Omega_m F)$. By focusing on disagreement between \textbf{full} ($\mathcal I_{\text{full}})$ and \textbf{pure} 
($\mathcal I_{\text{pure}}$) interaction operators, it is only necessary to evaluate $\nu^\Omega(D), \nu^\Omega(-i), \nu^\Omega(i), \nu^\Omega(\emptyset)$ 
for each feature $i=1,2,\ldots, d$. Since $\nu^\Omega(-i)$ can be derived from the transposed matrix $(\bm{R}^{i})^T$, GRANITE requires only the matrices 
$\{\bm{R}^{i}\}_{i=1}^d$ to compute explanations.

The space complexity of explanation computations varies depending on the stage of the algorithm. First, fitting GRANITE (compute explanations to find the regions) 
requires $\mathcal{O}(d N^2)$ space complexity, which motivates subsampling (1000-2000 samples in our experiments). Second, applying a fitted GRANITE on new data 
is cheaper: If $M$ regions share data evenly, explanation computation requires only $\mathcal{O}(N^2d/M)$ space complexity. This allows us to use the full 
test sets when reporting the final regional explanations.

\subsection{Estimating Partitions}\label{app:algorithm:partition}

GRANITE aims to identify a partition $\{\Omega_1, \ldots , \Omega_K\}$ of disjoint regions 
$\Omega_1,\dots,\Omega_K \subseteq \mathcal X$ such that two explanations exhibit increased agreement when restricted to 
these regions. Simplifying the regional disagreement risk to
$\mathcal{R}(\Omega)\equiv \mathcal{R}_{\ell\mid\Omega_k} (\phi_{1\mid\Omega_k}, \phi_{2\mid\Omega_k})$, we must solve
\begin{equation}
   \min_{\{\Omega_k\}_{k=1}^K} \sum_{k=1}^K \mathcal{R}(\Omega_k).
   \label{eq:objective_app}
\end{equation}

Since the set of all possible partitions of $\mathcal{X}$ is infinite, it is necessary to simplify the problem. 
We do so by constraining the regions to be the leaves of a binary decision tree with bounded depth. 
Moreover, the leaves are constrained to contain at 
least $N_{\text{min}}$ samples.
\begin{equation}
    \min_{
        \{\Omega^{\text{tree}}_k\}_{ k=1}^K:
        K \leq 2^{\text{max-depth}},
        \widehat{P}(\Omega^{\text{tree}}_k) \geq N_{\text{min}}/N
    } \quad
    \sum_{k=1}^K \mathcal{R}(\Omega^{\text{tree}}_k) + \alpha K
   \label{eq:objective_approx_app}
\end{equation}
This objective employs a hyperparameter $\alpha \geq 0$ to regularize the decision tree w.r.t. its number of leaves $K$.
Taking inspiration from existing greedy strategies to learn decision trees (e.g. CART), we start from a single region $\Omega=\mathcal X$ and split 
it into two sub-regions $\Omega_-$ and $\Omega_+$ based on a univariate split
$\Omega^{j, \gamma}_- = \Omega \cap \{x\in \mathcal X : x_j\leq \gamma\}$, $\Omega^{j, \gamma}_+ = \Omega \cap \{x\in \mathcal X : x_j> \gamma\}$.
Ideally, the univariate split would minimize disagreements
\begin{equation}
   j^\star, \gamma^\star = 
        \text{argmin}_{j\in D: \gamma\in \mathbb{R}} \mathcal{R}(\Omega^{j, \gamma}_-)
        + \mathcal{R}(\Omega^{j, \gamma}_+),
   \label{eq:objective_approx_local_app}
\end{equation}
(ignoring regularization since any split introduces the same number of leaves). However, we cannot span across
all possible values of $\gamma\in \mathbb{R}$, so we instead discretize the set $\{x^{(n)}_j: x^{(n)} \in \Omega\}$
into $B$ bins, leading to a set $S_{j,B,\Omega}$ of $B-1$ bin edges, and solve
\begin{equation}
   j^\star, \gamma^\star =
        \text{argmin}_{j\in D: \gamma\in S_{j,B, \Omega}} \mathcal{R}(\Omega^{j, \gamma}_-)
        + \mathcal{R}(\Omega^{j, \gamma}_+).
   \label{eq:objective_approx_local_app_bin}
\end{equation}
Once the optimal split according to Equation \ref{eq:objective_approx_local_app_bin} is performed, the procedure
is repeated recursively on the two sub-regions $\Omega^{j^\star,\gamma^\star}_-$ and $\Omega^{j^\star,\gamma^\star}_+$.
The recursion terminates when any of the following conditions is met: the number of data points in a region falls 
below $2 N_{\text{min}}$ (as no further split could satisfy the constraints), the maximum tree depth is reached,
or the risk falls below $\alpha$ (indicating that no further reduction in risk can compensate for the
regularization). Once the recursion stops, the algorithm backtracks through each parent node to assess
whether its split sufficiently decreased the objective relative to the number of leaves introduced.
If not splitting the internal node results in a better objective, the internal node is converted into a
leaf. This backtracking step is analogous to the pruning stage performed after greedy tree growth in CART.
The full procedure is presented in Algorithm \ref{alg:GRANITE}.


The method splits at most $\mathcal{O}(2^{\text{max-depth}})$ nodes, each requiring the evaluation
of the risk over $\mathcal{O}(dB)$ split candidates. Since evaluating the disagreement risk (based on the 
$\{\bm{R}^i\}_{i=1}^d$ tensors) has a complexity of $\mathcal{O}(d N^2)$, the overall complexity of the 
partitioning algorithm is $\mathcal{O}(2^{\text{max-depth}} d^2 N^2 B)$.

\begin{algorithm}
    \caption{GRANITE}
    \begin{algorithmic}[1]
    \Require Pre-computing the empirical distribution $\widehat{P}$ over a finite dataset $\{x^{(n)}\}_{n=1}^N$.
    \Require Pre-computing the model evaluation matrices $\{\bm{R}^i\}_{i=1}^d$.
    \Require Being able to compute the risk $\mathcal{R}(\Omega) \equiv \mathcal{R}_{\ell\mid\Omega} (\phi_{1\mid\Omega}, \phi_{2\mid\Omega})$ on any region.
    \Require Regularization hyperparameter $\alpha \geq 0$ (default is $\alpha=0.05$).
    \Require Regularization hyperparameter $N_{\text{min}}$ (default is $N_{\text{min}}=20$).
    \Require Number of bins $B$ (default is $B=40$).
    \State \textcolor{codegreen}{\# This function is called recursively to return a tree structure.}
    \Procedure{GRANITE}{$\Omega, \text{curr\_depth}, \text{curr\_risk}$}
        \State Initialize LeafNode.
        \State \textcolor{codegreen}{\# If the objective cannot be increased further under the constraints, this node becomes a leaf.}
        \If{ $\text{curr\_depth = max-depth \textbf{or} curr\_risk} < \alpha \textbf{ or } \widehat{P}(\Omega) < 2 N_{\text{min}}/N$}
            \State \Return $\text{LeafNode}, \text{curr\_risk}, 1$
        \EndIf
        \State Initialize InternalNode
        \State Define $S_{j,B,\Omega}$ as the $B-1$ bin edges resulting from binning $\{x^{(n)}_j: x^{(n)}\in\Omega\}$.
        \State Define $\Omega^{j, \gamma}_- = \Omega \cap \{x\in \mathcal X : x_j\leq \gamma\}$
        \State Define $\Omega^{j, \gamma}_+ = \Omega \cap \{x\in \mathcal X : x_j > \gamma\}$
        \State $j^\star, \gamma^\star = \text{argmin}_{j\in D: \gamma\in S_{j,B,\Omega}}  \mathcal{R}(\Omega^{j,\gamma}_-) + \mathcal{R}(\Omega^{j, \gamma}_+)$
        \State $\text{InternalNode.child\_left, new\_risk\_left, n\_leaves\_left} \qquad= $ \Call{GRANITE}{$\Omega^{j^\star, \gamma^\star}_-, \text{curr\_depth} + 1, \mathcal{R}(\Omega^{j^\star,\gamma^\star}_-)$}
        \State $\text{InternalNode.child\_right, new\_risk\_right, n\_leaves\_right} = $ \Call{GRANITE}{$\Omega^{j^\star, \gamma^\star}_+, \text{curr\_depth} + 1, \mathcal{R}(\Omega^{j^\star,\gamma^\star}_+)$}

        \State \textcolor{codegreen}{\# While backtracking, regularization may force us to let the current node be a leaf instead.}
        \If{ $\text{curr\_risk} + \alpha < \text{new\_risk\_left} + \text{new\_risk\_right} + \alpha (\text{n\_leaves\_left} + \text{n\_leaves\_right})$}
            \State \Return LeafNode, curr\_risk, 1
        \Else
            \State \Return InternalNode, new\_risk\_left + new\_risk\_right, n\_leaves\_left + n\_leaves\_right
        \EndIf
    \EndProcedure
    \State Tree, disagreement\_risk, n\_leaves = \Call{GRANITE}{$\mathcal X, 0 ,0$}
    \end{algorithmic}
    \label{alg:GRANITE}
\end{algorithm}

\begin{figure}[t]
    \centering
    \begin{subfigure}{0.32\textwidth}
        \centering
        \includegraphics[width=\linewidth]{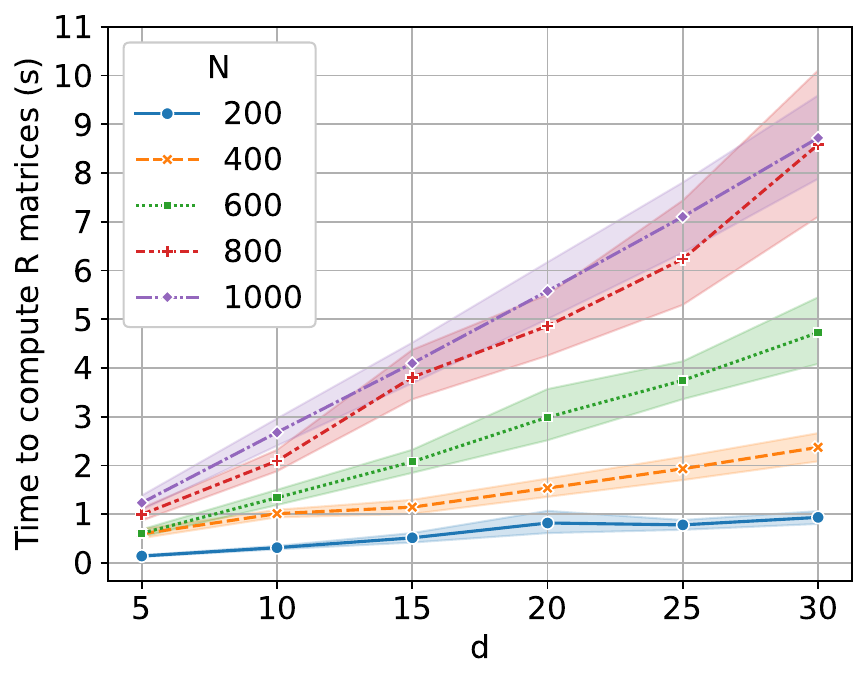}
        \caption{}
    \end{subfigure}
    \begin{subfigure}{0.32\textwidth}
        \centering
        \includegraphics[width=\linewidth]{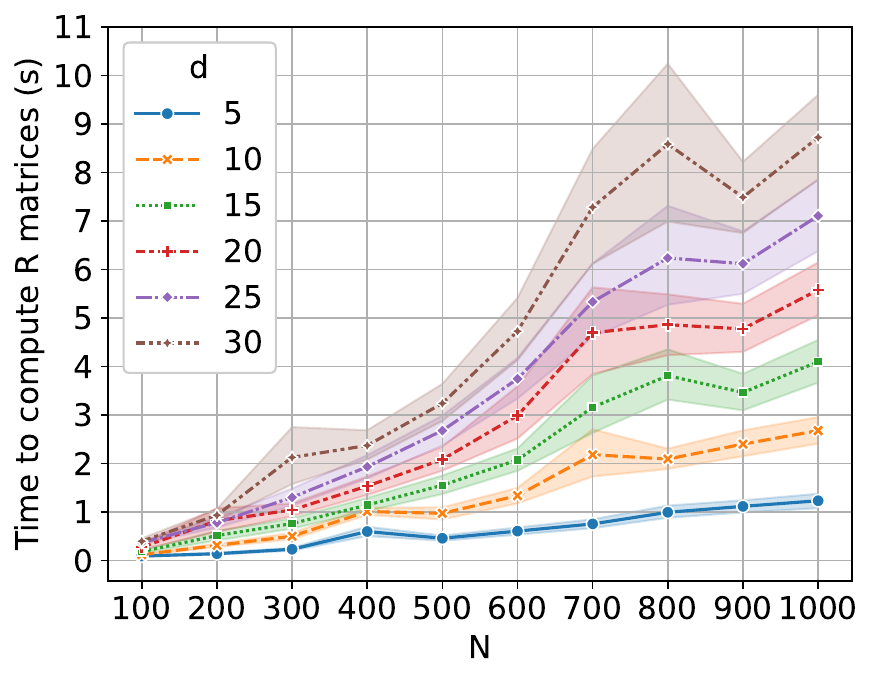}
        \caption{}
    \end{subfigure}
    \vspace{0.2cm}
    \caption{Runtimes for computing the $\{\bm{R}^i\}_{i=1}^d$ matrices. The theoretical complexity is $\mathcal{O}(d N^2)$.}
    \label{fig:decomposition_time}
\end{figure}

\begin{figure}[t]
    \centering
    \begin{subfigure}{0.32\textwidth}
        \centering
        \includegraphics[width=\linewidth]{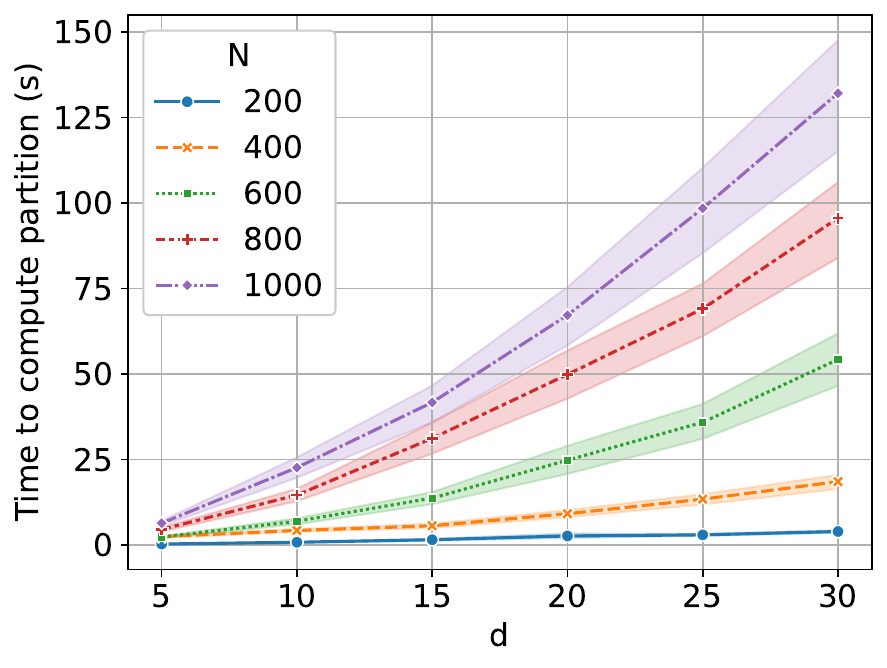}
        \caption{}
    \end{subfigure}
    \begin{subfigure}{0.32\textwidth}
        \centering
        \includegraphics[width=\linewidth]{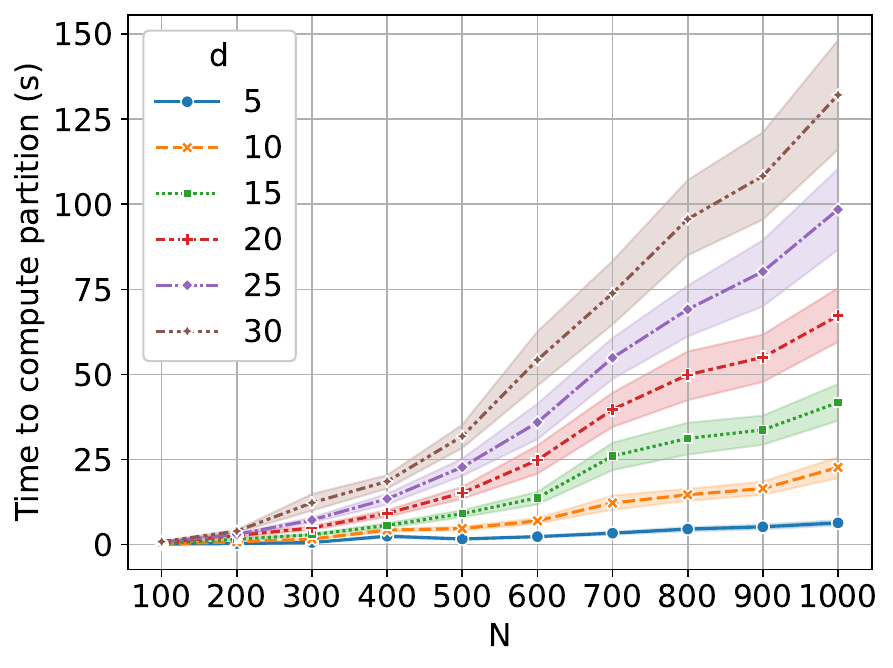}
        \caption{}
    \end{subfigure}
    \begin{subfigure}{0.32\textwidth}
        \centering
        \includegraphics[width=\linewidth]{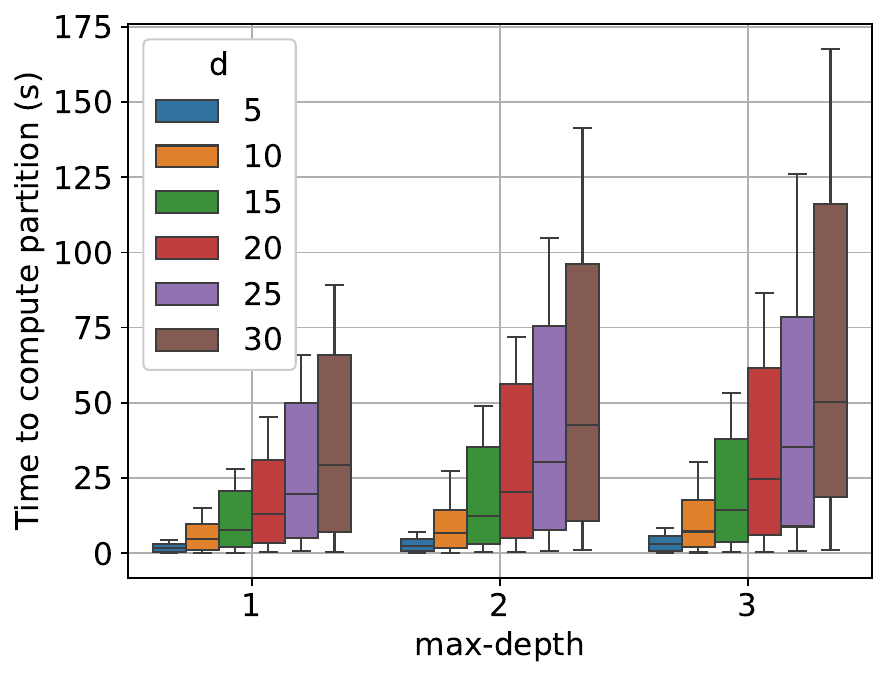}
        \caption{}
    \end{subfigure}
    \vspace{0.2cm}
    \caption{Runtimes for computing the partition (i.e. Algorithm \ref{alg:GRANITE}). The theoretical complexity is $\mathcal{O}(2^{\text{max-depth}} d^2 N^2 B)$.}
    \label{fig:partition_time}
\end{figure}

\newpage
\subsection{Scalability Experiments}\label{app:algorithm:experiment}

We investigate the scalability of explanations and partition computations on toy examples where the parameters 
$N, d$ and \texttt{max-depth} are controlled. For simplicity, we sample $N$ data points from a $d$-dimensional
isotropic Gaussian and generate a random model that is piece-wise linear over the leaves of a randomly generated
decision tree of depth \texttt{max-depth}. Figure \ref{fig:decomposition_time} presents the required time to compute the 
$\bm{R}^i$ matrices for various values of $N$ and $d$. We indeed observe a linear runtime for $d$ and a quadratic one for 
$N$. This is not surprising seeing as the model has to be called $dN^2$ times. Also, note that it never takes more than 
$10$ seconds to compute the matrices, even when studying $1000$ samples with $30$ features. These runtimes are 
acceptable but they could be further improved by using a model-specific implementation.

Figure \ref{fig:partition_time} shows the runtime of the partitioning algorithm given the parameters $d$, $N$ and
\texttt{max-depth}. The hyperparameters $\alpha,N_{\text{min}}$ and $B$ are fixed to their default value. 
We observe quadratic complexity w.r.t. $d$ and $N$, while the expected exponential complexity w.r.t \texttt{max-depth}
is not apparent. Indeed, we would expect the runtimes with \texttt{max-depth=3} to be 4 times higher than the 
runtimes of \texttt{max-depth=1}. We think this occurs because the regularization $\alpha=0.05$, although small, 
restricts the algorithm from growing the tree unless reductions in risk are significant.

\newpage
\section{FRAMEWORK UNIFICATION}\label{app:unification}

In this section, we provide details on how existing regional frameworks are unified within GRANITE by showing that each of them minimizes a special case of the regional disagreement of Definition \ref{def:disagree}:
\begin{align}\label{eq:disagreement}
    \mathcal R_{\ell \mid \Omega}(\phi_{1 \mid \Omega},\phi_{2\mid\Omega}) := \mathbb{E}_{P_{X}^\Omega}\left[ \sum_{i \in D} \ell\left(\phi_{1\mid\Omega}(i)-\phi_{2\mid\Omega}(i)\right)\right]
\end{align}

\paragraph{VINE \citep{britton2019vine}:}
VINE clusters derivative ICE (d-ICE) curves \citep{goldstein_peeking_2015} using Euclidean distance. This means regions are found by a clustering algorithm that minimizes the Euclidean distance between d-ICE curves and their cluster centroid being the derivative PDP (d-PDP).
A d-ICE curve for the $i$-th feature and for a fixed observation $x^\star$ can be decomposed by functional ANOVA (see Section \ref{sec:fanova}) into
\begin{align*}
    \frac{\partial f(x_i,x_{-i}^\star)}{\partial x_i} = \sum_{S\subseteq D}\frac{\partial f_S(x_i,x_{-i}^\star)}{\partial x_i} = \sum_{S\subseteq D: i \in S}\frac{\partial f_S(x_i,x_{-i}^\star)}{\partial x_i},
\end{align*}
where the last step follows from the fact that all effects $f_S$ independent of $x_i$ become zero. 
Hence, d-ICE curves fully account for interactions with the feature of interest ($x_i$) and when considered for all observations together it recovers the marginal distribution of features in $-i$. Therefore, d-ICE curves are a \textbf{full} explanation method to measure \textbf{individual feature influence}, under \textbf{local behavior} $\mathcal{B}_{loc, x_0}$ and \textbf{marginal masking} $\mathcal{M}_m$.

The respective d-PDP is then defined by
\begin{align*}
    \mathbb{E}_{X_{-j}}\left[\frac{\delta f(x_i,X_{-i})}{\delta x_i}\right] = \sum_{S\subseteq D: i \in S}\mathbb{E}_{X_{-j}}\left[\frac{\delta f_S(x_i, X_{-j})}{\delta x_i}\right],
\end{align*}
where the influence of features in $-i$ is integrated out using their joint marginal distribution $P_{X_{-i}}$. Therefore, d-PDP is a \textbf{pure} explanation method to measure \textbf{individual feature influence}, under \textbf{local behavior} $\mathcal{B}_{loc, x_0}$ and \textbf{marginal masking} $\mathcal{M}_m$ \citep{fumagalli2025unifying}.

Hence, VINE minimizes the regional disagreement between d-ICE ($\phi_{1\mid\Omega}$) and d-PDP ($\phi_{2\mid\Omega}$), which are defined by
\begin{align*}
    \phi_{1\mid\Omega}(i) = (\mathcal{I}_{\text{full}}(\mathcal{B}_{loc,x_0}^\Omega(\mathcal{M}^\Omega_m F)))(i) \quad \text{and} \quad \phi_{2\mid\Omega}(i) = (\mathcal{I}_{\text{pure}}(\mathcal{B}_{loc,x_0}^\Omega(\mathcal{M}^\Omega_m F)))(i),
\end{align*}
using $\ell(\phi_{1\mid\Omega}(i) - \phi_{2\mid\Omega}(i)) = \sqrt{(\phi_{1\mid\Omega}(i) - \phi_{2\mid\Omega}(i))^2}$ and the joint marginal distribution $P_{X_{-i}}^\Omega$.

It follows that VINE \interaction{minimizes interactions} for \textbf{individual feature influence} measures between \textbf{full} and \textbf{pure} effects under \textbf{local behavior} and \textbf{marginal masking} using a clustering algorithm. VINE only works for single features, i.e., we can only minimize interactions between one feature $i$ and all other features $-i$. Hence, the disagreement risk is defined by
$$
\mathcal R_{\ell \mid \Omega}(\phi_{1 \mid \Omega},\phi_{2\mid\Omega}) := \mathbb{E}_{P_{X}^\Omega}\left[ \ell\left(\phi_{1\mid\Omega}(i)-\phi_{2\mid\Omega}(i)\right)\right].
$$
Applying VINE to different features is possible but will most likely lead to different clusters, complicating interpretation.

\paragraph{Cohex \citep{meng2024cohex}}:
Cohex clusters local explanations by minimizing the mean squared error between local explanations and their cluster (regional) mean. While they argue that Cohex is applicable to many local explanation methods, they focus on marginal (interventional) SHAP values within the paper. Marginal SHAP is a \textbf{partial} explanation method \citep{fumagalli2025unifying, herbinger2024gadget} to measure \textbf{individual feature influence}, under \textbf{local behavior} $\mathcal{B}_{loc, x_0}$ and \textbf{marginal masking} $\mathcal{M}_m$.
The cluster (regional) mean of SHAP values per feature $i$ integrates out the influence of features in $-i$ using their joint marginal distribution $P_{X_{-i}}$. Thus, the cluster mean is a \textbf{pure} explanation method to measure \textbf{individual feature influence}, under \textbf{local behavior} $\mathcal{B}_{loc, x_0}$ and \textbf{marginal masking} $\mathcal{M}_m$ \citep{herbinger2024gadget}.

Hence, Cohex minimizes the regional disagreement between marginal SHAP ($\phi_{1\mid\Omega}$) and its regional mean ($\phi_{2\mid\Omega}$), which are defined by
\begin{align*}
    \phi_{1\mid\Omega}(i) = (\mathcal{I}_{\text{partial}}(\mathcal{B}_{loc,x_0}^\Omega(\mathcal{M}^\Omega_m F)))(i) \quad \text{and} \quad \phi_{2\mid\Omega}(i) = (\mathcal{I}_{\text{pure}}(\mathcal{B}_{loc,x_0}^\Omega(\mathcal{M}^\Omega_m F)))(i),
\end{align*}
using $\ell(\phi_{1\mid\Omega}(i) - \phi_{2\mid\Omega}(i)) = (\phi_{1\mid\Omega}(i) - \phi_{2\mid\Omega}(i))^2$ and the joint marginal distribution $P_{X_{-i}}^\Omega$.

Cohex \interaction{minimizes interactions} for \textbf{individual influence measures} between \textbf{partial} and \textbf{pure} effects under \textbf{local behavior} and \textbf{marginal masking} using clustering.
In contrast to VINE, Cohex directly minimizes the sum (L2 norm) of all features and thus minimizes interactions between all features as defined in Eq.~\eqref{eq:disagreement}.

\paragraph{REPID \citep{herbinger2022repid}:}
REPID applies recursive partitioning to minimize the squared error between mean-centered ICE curves and mean-centered PDPs within each region for one feature of interest. \cite{herbinger2022repid} and \cite{fumagalli2025unifying} show that mean-centered ICE curves \textbf{fully} account for interactions for the respective feature. 
\cite{fumagalli2025unifying} classifies single mean-centered ICE curves as \textbf{individual feature influence} measures, under \textbf{local behavior} $\mathcal{B}_{loc, x_0}$ and \textbf{baseline masking} $\mathcal{M}_b$. However, when all ICE curves within a region are considered, the values used for features in $-i$ reflect their joint regional marginal distribution $P_{X_{-i}}^\Omega$, thus considered together they reflect \textbf{marginal masking} $\mathcal{M}_m$. 
The mean-centered PDP is the average of the mean-centered ICE curves and integrates out the influence of features in $-i$ using their joint marginal distribution $P_{X_{-j}}$ \citep{herbinger2022repid}. Thus, the regional mean-centered PDP is a \textbf{pure} explanation method to measure \textbf{individual feature influence} under \textbf{local behavior} $\mathcal{B}_{loc, x_0}$ and \textbf{marginal masking} $\mathcal{M}_m$ \citep{fumagalli2025unifying}.

Hence, REPID minimizes the regional disagreement between mean-centered ICE curves ($\phi_{1\mid\Omega}$) and mean-centered regional PDPs ($\phi_{2\mid\Omega}$), which are defined by
\begin{align*}
    \phi_{1\mid\Omega}(i) = (\mathcal{I}_{\text{full}}(\mathcal{B}_{loc,x_0}^\Omega(\mathcal{M}^\Omega_m F)))(i) \quad \text{and} \quad \phi_{2\mid\Omega}(i) = (\mathcal{I}_{\text{pure}}(\mathcal{B}_{loc,x_0}^\Omega(\mathcal{M}^\Omega_m F)))(i),
\end{align*}
using $\ell(\phi_{1\mid\Omega}(i) - \phi_{2\mid\Omega}(i)) = (\phi_{1\mid\Omega}(i) - \phi_{2\mid\Omega}(i))^2$ and the joint marginal distribution $P_{X_{-i}}^\Omega$.

Similarly to VINE, REPID \interaction{minimizes interactions} for \textbf{individual influence measures} between \textbf{full} and \textbf{pure} effects under \textbf{local behavior} and \textbf{marginal masking} for only one feature of interest. Thus, the disagreement risk is also given by
$$
\mathcal R_{\ell \mid \Omega}(\phi_{1 \mid \Omega},\phi_{2\mid\Omega}) := \mathbb{E}_{P_{X}^\Omega}\left[ \ell\left(\phi_{1\mid\Omega}(i)-\phi_{2\mid\Omega}(i)\right)\right].
$$

Compared to VINE, REPID uses a recursive partitioning algorithm instead of clustering, which has the advantage that it leads to interpretable regions in the feature space.

\paragraph{GADGET \citep{herbinger2024gadget}:}
GADGET minimizes the disagreement between any local feature effect method that can be decomposed into components $h$ that contain only effects that depend on $i$, i.e.,
$$
\phi_{1\mid\Omega}(i) = h_i(x_i) + \sum_{S\subseteq D: i\in S, |S|\geq 2} h_S(x),
$$
and its regional mean, i.e.,
$$
\phi_{2\mid \Omega}(i) = h_i(x_i) + \sum_{S\subseteq D: i\in S, |S|\geq 2} \mathbb{E}_{X_{-i}}[h_S(X)].
$$

Since their approach is based on local effects, the explanations are defined by
\begin{align*}
    \phi_{1\mid\Omega}(i) = (\mathcal{I}_1(\mathcal{B}_{loc,x_0}^\Omega(\mathcal{M}_1^\Omega F)))(i) \quad \text{and} \quad \phi_{2\mid\Omega}(i) = (\mathcal{I}_{\text{pure}}(\mathcal{B}_{loc,x_0}^\Omega(\mathcal{M}_2^\Omega F)))(i),
\end{align*}
where $\mathcal{I}_1 \in \{\mathcal{I}_{\text{partial}}, \mathcal{I}_{\text{full}}\}$ and $\mathcal{M}_1 = \mathcal{M}_2$.
The disagreement loss is defined by the L2 loss, i.e., $\ell(\phi_{1\mid\Omega}(i) - \phi_{2\mid\Omega}(i)) = (\phi_{1\mid\Omega}(i) - \phi_{2\mid\Omega}(i))^2$.

In \cite{herbinger2024gadget} they particularly derive it for the following three effect methods:
\begin{itemize}
    \item \textbf{Mean-centered ICE curves and PDPs:} With $h_S(x) = f_S(x)$ and thus $\mathcal{I}_1 = \mathcal{I}_{\text{full}}$ and $\mathcal{M} = \mathcal{M}_m$.
    \item \textbf{Marginal SHAP values and Average SHAP Depedence:} With $h_S(x) = \frac{1}{|S|}f_S(x)$ and thus $\mathcal{I}_1 = \mathcal{I}_{\text{partial}}$ and $\mathcal{M} = \mathcal{M}_m$.
    \item \textbf{Local Derivatives and Accumulated Local Effects (ALE) Plots \citep{apley_visualizing_2020}:} With $h_S(x) = \frac{\delta}{\delta x_S}f_S(x)$ and thus $\mathcal{I}_1 = \mathcal{I}_{\text{full}}$ and $\mathcal{M} = \mathcal{M}_c$.
\end{itemize}

GADGET \interaction{minimizes interactions} for \textbf{individual influence measures} between \textbf{full} and \textbf{pure} as well as between \textbf{partial} and \textbf{pure} effects under \textbf{local behavior} and \textbf{marginal} and \textbf{conditional} masking \citep{herbinger2024gadget, fumagalli2025unifying}. While GADGET also uses recursive partitioning and the L2 loss comparably to REDPID, GADGET does consider all (or a subset of) features as defined in Eq.~\eqref{eq:disagreement}.

\paragraph{FD-Trees \citep{laberge2024tackling}:}
Similarly to GADGET, FD-Trees minimize disagreement between local feature effect methods that can be decomposed into the components $h$ as defined above. However, compared to GADGET, the explanations in FD-Trees can take the following operator choices:
\begin{align*}
    \phi_{1\mid\Omega}(i) = (\mathcal{I}_1(\mathcal{B}_{loc,x_0}^\Omega(\mathcal{M}_m^\Omega F)))(i) \quad \text{and} \quad \phi_{2\mid\Omega}(i) = (\mathcal{I}_{2}(\mathcal{B}_{loc,x_0}^\Omega(\mathcal{M}_m^\Omega F)))(i),
\end{align*}
where $\mathcal{I}_1, \mathcal{I}_2 \in \{\mathcal{I}_{\text{pure}},\mathcal{I}_{\text{partial}}, \mathcal{I}_{\text{full}}\}$ with $\mathcal{I}_1 \neq \mathcal{I}_2$.
The disagreement loss is defined by the L2 loss, i.e., $\ell(\phi_{1\mid\Omega}(i) - \phi_{2\mid\Omega}(i)) = (\phi_{1\mid\Omega}(i) - \phi_{2\mid\Omega}(i))^2$.

In \cite{laberge2024tackling} they particularly consider the following three disagreements:
\begin{itemize}
    \item \textbf{PFI on Predictions (Mean-centered ICE curves) and PDPs:} With $h_S(x) = f_S(x)$ and thus $\mathcal{I}_1 = \mathcal{I}_{\text{full}}$ and $\mathcal{M}_1 = \mathcal{M}_m$ for $\phi_1$ and $\phi_2$ being the expectation w.r.t. features in $-i$ as defined above.
    \item \textbf{Marginal SHAP values and PDPs (Average SHAP Depedence):} With $h_S(x) = \frac{1}{|S|}f_S(x)$ and thus $\mathcal{I}_1 = \mathcal{I}_{\text{partial}}$ and $\mathcal{M}_1 = \mathcal{M}_m$ for $\phi_1$ and $\phi_2$ being the expectation w.r.t. features in $-i$ as defined above.
    \item \textbf{PFI on Predictions (Mean-centered ICE curves) and Marginal SHAP values:} With $h_S(x) =f_S(x)$ and thus $\mathcal{I}_1 = \mathcal{I}_{\text{full}}$ and $\mathcal{M}_1 = \mathcal{M}_m$ for $\phi_1$ and $h_S =\frac{1}{|S|}f_S(x)$ and thus $\mathcal{I}_1 = \mathcal{I}_{\text{partial}}$ and $\mathcal{M}_1 = \mathcal{M}_m$.
\end{itemize}

FD-Trees \interaction{minimize interactions} for \textbf{individual influence measures} between \textbf{full} and \textbf{pure}, \textbf{full} and \textbf{partial} as well as between \textbf{partial} and \textbf{pure} effects under \textbf{local behavior} and \textbf{marginal} masking \citep{laberge2024tackling, fumagalli2025unifying}. Similarly to GADGET, FD-Trees also apply recursive partitioning and do consider all features as defined in Eq.~\eqref{eq:disagreement}.

\paragraph{Transformation Trees \citep{molnar2023model}:}
Transformation Trees minimize the disagreement between the conditional distribution of the $j$-th feature $P_{X_j\mid X_{-j} = x_{-j}}$ and its marginal distribution $P_{X_j}$ using recursive partitioning. In the paper, the authors argue that this approach minimizes the disagreement between CFI (conditional) and PFI (marginal), as well as between MPlot (conditional) and PDP (marginal).

More generally, this approach minimizes disagreement between full conditional and full marginal individual influence measures since these measures are defined by integrating out only the $j$-th feature based on its conditional and marginal distribution, which directly corresponds to $P_{X_j\mid X_{-j} = x_{-j}}$ and $P_{X_j}$. 
\begin{align*}
    \phi_{1\mid\Omega}(i) = (\mathcal{I}_{\text{full}}(\mathcal{B}^\Omega(\mathcal{M}_c^\Omega F)))(i) \quad \text{and} \quad \phi_{2\mid\Omega}(i) = (\mathcal{I}_{\text{full}}(\mathcal{B}^\Omega(\mathcal{M}_m^\Omega F)))(i),
\end{align*}

Thus, it follows that the disagreement between CFI and PFI is minimized using Transformation Trees since they are both full individual influence measures \citep{fumagalli2025unifying}.

However, the disagreement between pure conditional and pure marginal individual influence measures, such as MPlot and PDP, is not explicitly targeted, as this would require minimizing the discrepancy between \(P_{X_{-j}\mid X_{j} = x_{j}}\) and \(P_{X_{-j}}\), which is not the objective of Transformation Trees. Nevertheless, if Transformation Trees achieve complete independence between \(X_j\) and \(X_{-j}\) within the final regions, it follows not only that \(P_{X_j\mid X_{-j} = x_{-j}} = P_{X_j}\), but also that \(P_{X_{-j}\mid X_{j} = x_{j}} = P_{X_{-j}}\). In this case, the conditional and marginal pure effects coincide.
Transformation Trees generally aim to \correlation{minimize distributional influence} for \textbf{individual influence measures} between \textbf{conditional} and \textbf{marginal} masking in the case of \textbf{full} measures. However, the paper considers \textbf{full} measures in the context of \textbf{risk behavior} and \textbf{pure} measures in the context of \textbf{local behavior} \citep{molnar2023model, fumagalli2025unifying}. In the latter case, the disagreement between conditional and marginal measures is not explicitly minimized but vanishes if the tree’s optimization leads to terminal regions where $X_j$ is independent of $X_{-j}$. Similar to REPID, Transformation Trees rely on recursive partitioning and consider only one feature of interest.

\newpage
\section{ADDITIONAL EXPERIMENTS}\label{app:experiments}

In addition to Section~\ref{sec:experiments}, we present further applications of GRANITE on various real-world datasets to address different interpretability questions and provide additional insights. We first extend the analysis of the \textit{Bikesharing} dataset, followed by an investigation of the \textit{Kin8nm} dataset to explore higher-order interactions, and conclude with the \textit{Diabetes} dataset to examine joint feature influence, with full implementation and code for the experiments available at \url{https://github.com/gablabc/GRANITE}.

\subsection{Bikesharing Extended}\label{app:experiments:bikesharing}

Bikesharing is a regression dataset whose task is to predict the number of bike rentals given 17K examples and 10 features: \emph{yr}, \emph{mnth}, \emph{hr}, 
\emph{holiday}, \emph{weekday}, \emph{workingday}, \emph{weathersit}, \emph{temp}, \emph{hum}, and \emph{windspeed}. The dataset was split into train and test sets
with an 80-20 ratio, resulting in 14K training instances and 3K test instances. A \texttt{HistGradientBoostingRegressor}\footnote{\url{https://scikit-learn.org/stable/modules/generated/sklearn.ensemble.HistGradientBoostingRegressor.html}} was tuned using 5-fold cross-validation on the training set, with hyperparameters selected via random search. The final model attained a Root-Mean-Squared-Error of 43.08 on the test set, while the target has a standard deviation of 181.45. Given that the error is more than four times lower than the model variability, we can safely assume that its post-hoc explanations are able to highlight meaningful trends in the data.

\interaction{Minimizing Feature Interactions.} We first investigated the model behavior by reporting the ICE curves (full-marginal) and PDP (pure-marginal) for the 10
features of Bikesharing (see Figure \ref{fig:bike_full_pure_global_all}). It is apparent that there are strong disagreements between the ICE and PDP curves for features
\emph{hr}, \emph{workingday}, \emph{temp}, and \emph{yr}, which makes it difficult to understand how these features individually impact the model response.
By partitioning the input into three disjoint regions, however, GRANITE is able to reduce the disagreements between the pure and full explanations 
(see Figure \ref{fig:bike_full_pure_regional_all}). We note that all ICE/PDP curves are flat when it is early in the morning on a working day. This indicates that the 
model does not exhibit any interesting variability during these times. This regional trend was initially hidden in the non-regional explanations, since 
they mask features using the whole dataset as a reference.

\begin{figure}[htb]
    \centering
    \includegraphics[width=0.85\linewidth]{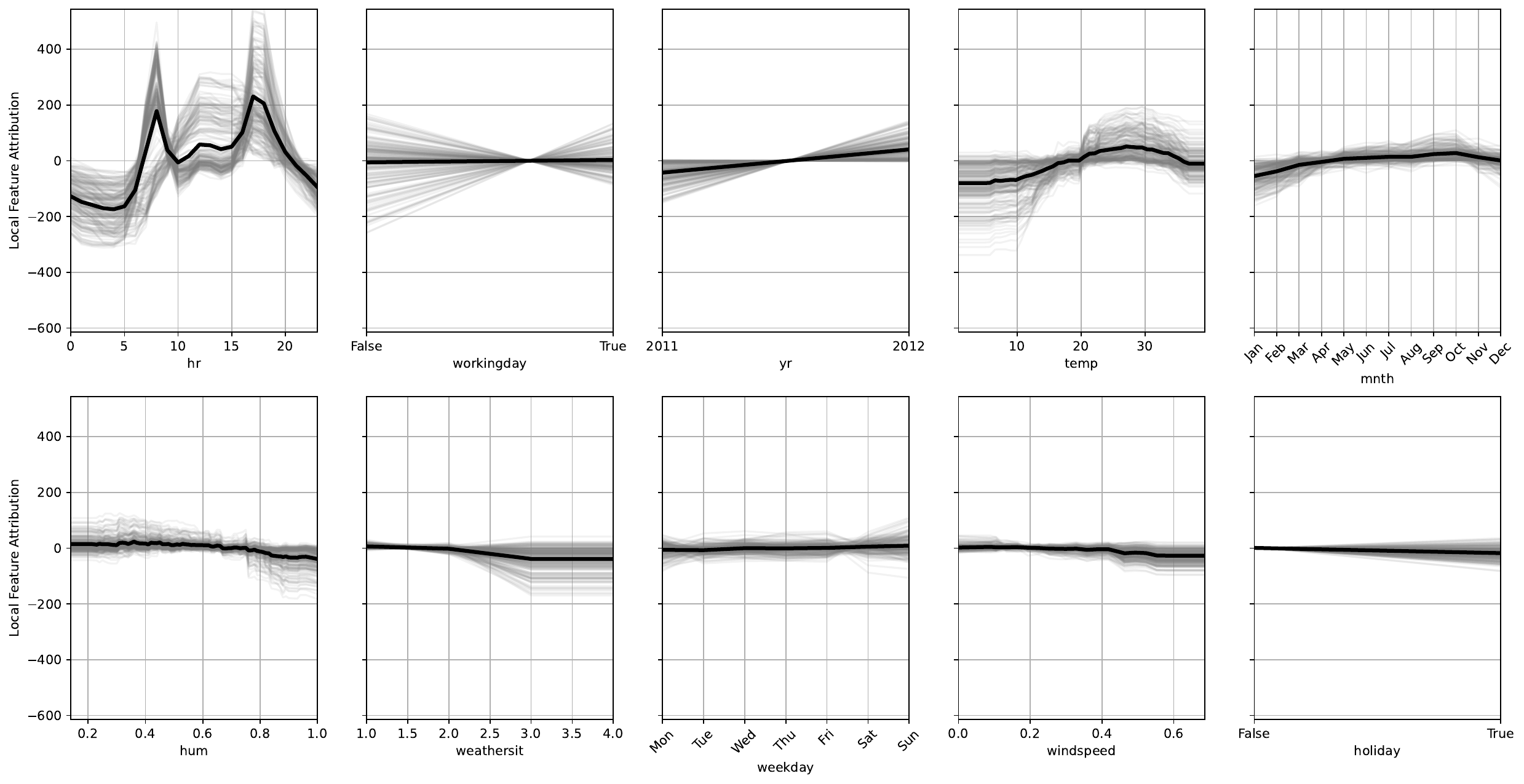}
    \caption{\interaction{Interaction disagreement} with $\mathcal{B}_{loc,x_0}$ and $ \mathcal{M}_m$ between full (ICE curves as thin lines) and pure (PDP as thick line) 
    explanations. \emph{Hour}, \emph{temperature}, \emph{yr}, and \emph{workingday} exhibit the highest disagreements between ICE and PDP.}
    \label{fig:bike_full_pure_global_all}
\end{figure}
\begin{figure}[!t]
    \centering
    \includegraphics[width=0.85\linewidth]{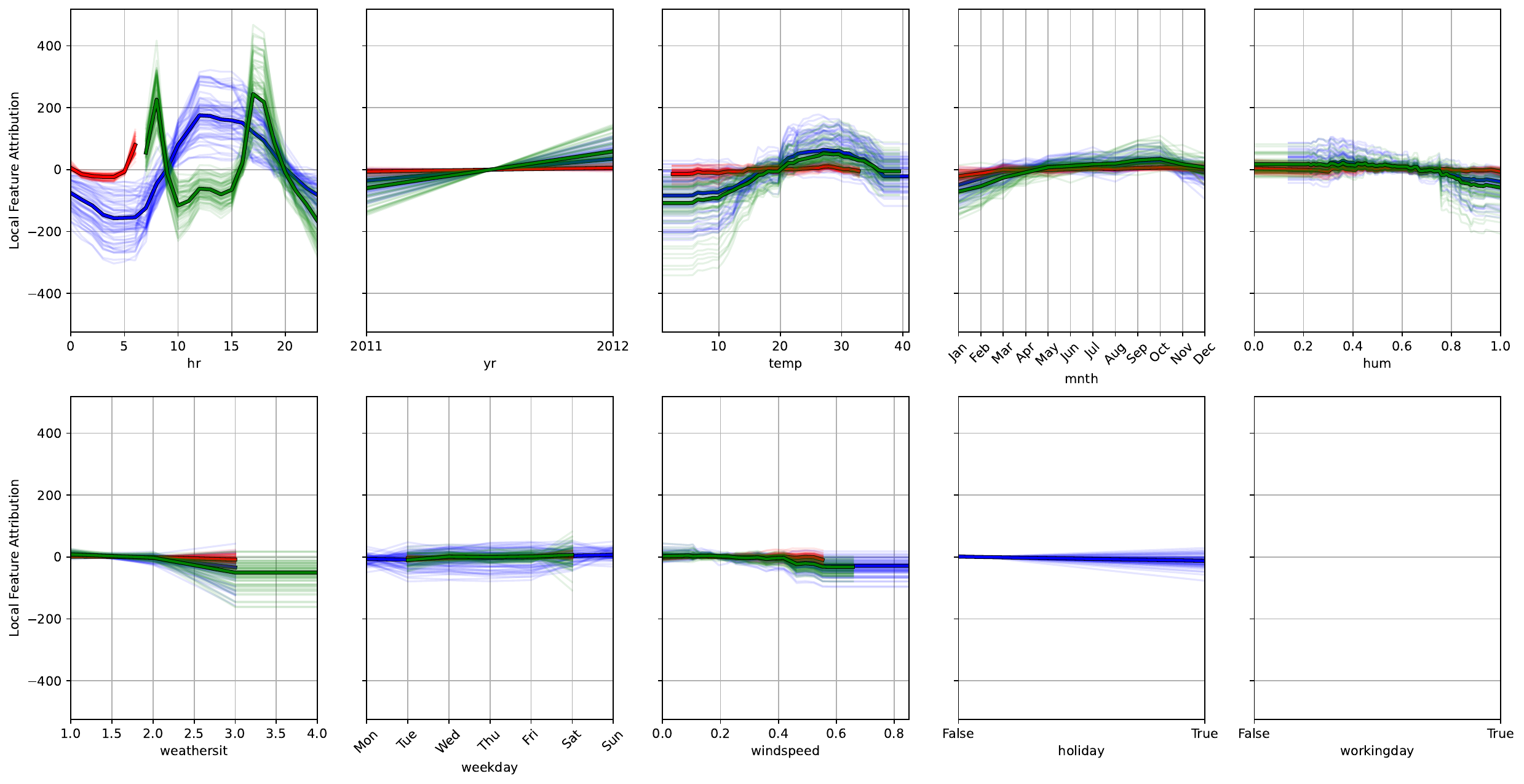}
    \includegraphics[width=0.5\linewidth]{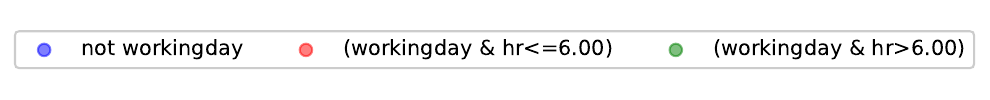}
    \caption{\interaction{Minimizing Interaction disagreement} with $\mathcal{B}_{loc,x_0}$ and $ \mathcal{M}_m$ between full (ICE curves, thin lines) and pure 
    (PDP, thick line) explanations. The disagreements are lessened when the explanations are restricted to the three regions.}
    \label{fig:bike_full_pure_regional_all}
\end{figure}

\correlation{Minimizing Feature Dependencies.} Besides feature interactions, dependencies can complicate the interpretation of black-box models, because marginal masking may evaluate the model on unrealistic inputs. Since model behavior is often unpredictable outside the observed data, marginal-based explanations should be interpreted with caution.
To confirm that extrapolation is occurring in marginal
explanations, we compared the MPlot (pure-conditional-local effects) alongside the PDP (pure-marginal-local effects) in Figure \ref{fig:bike_marg_cond_global_all}.
We indeed observe that, for weather-related features (\emph{mnth}, \emph{weathersit}, \emph{temps}, \emph{hum}), the MPlot and PDP tend to disagree on
extreme values of the feature. To reduce the disagreement between the conditional and marginal explanations, GRANITE suggests partitioning the input space into
three regions that separate early/daytime hours and hot/cold temperatures. Figure \ref{fig:bike_marg_cond_regional_all} shows the corresponding regional explanations.
It is apparent that disagreements between PDP and MPlot are greatly reduced, which implies that regional marginal explanations are less affected by extrapolation
than their global counterparts.

A similar conclusion can be drawn when comparing CFI and PFI (full-conditional-risk versus full-marginal-risk). When using the whole dataset as the reference distribution,
the two methods disagree on the feature importance rankings, as shown in Figure \ref{fig:bike_pfi_vs_cfi}~(a). Notably, CFI and PFI contradict each other on
the importance of \emph{workingday}: PFI identifies it as the second most important feature, whereas CFI attributes it no importance. This disagreement is induced by a dependency between
\emph{workingday} and the features \emph{weekday} and \emph{holiday}. Indeed, one feature can be determined by knowing the other two:
$\emph{workingday} = [\emph{weekday}\notin{Sat, Sun}]\wedge [\neg\emph{holiday}]$. This redundancy induces conflicting interpretations between the marginal
and conditional explanations. PFI considers \emph{workingday} important because the model relies on it to make predictions, while CFI claims it is not important because
the learning algorithm does not \textbf{need} it to yield an accurate model (the other two features could be used instead). GRANITE addresses
this ambiguity by restricting the distribution to well-chosen regions (see Figure \ref{fig:bike_pfi_vs_cfi}~(b)). In this case, both methods attribute no regional
importance to \emph{workingday}. This does not mean that this feature is not important globally, however, since \emph{workingday} is used to split the input space.
It remains important in the sense that it substantially changes the model’s behavior when conditioned upon.

\begin{figure}[tbh]
    \centering
    \includegraphics[width=0.9\linewidth]{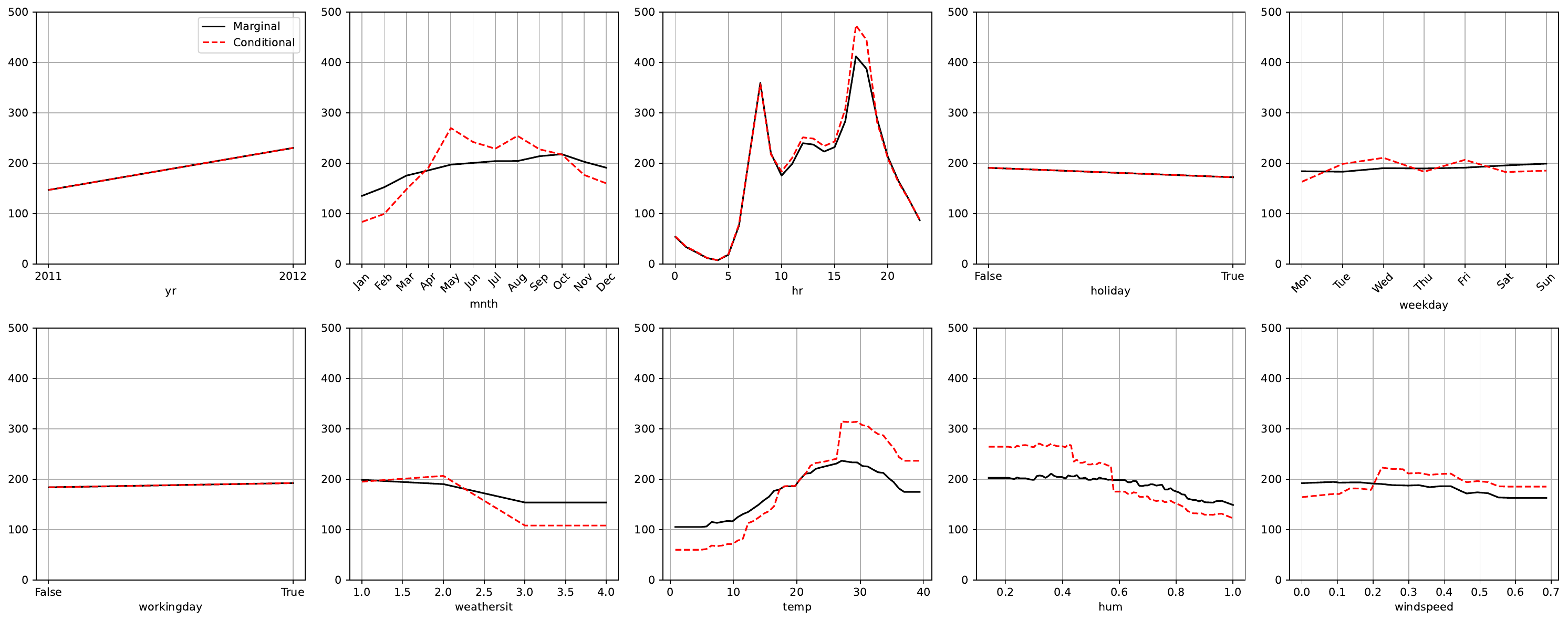}
    \caption{\correlation{Distribution disagreement} with $\mathcal{B}_{loc,x_0}$ and $\mathcal{I}_{pure}$ between MPlot ($\mathcal M_c$) and PDP ($\mathcal M_m$). 
    Highest disagreements occur for weather-related features.}
    \label{fig:bike_marg_cond_global_all}
\end{figure}

\begin{figure}[tbh]
    \centering
    \includegraphics[width=0.9\linewidth]{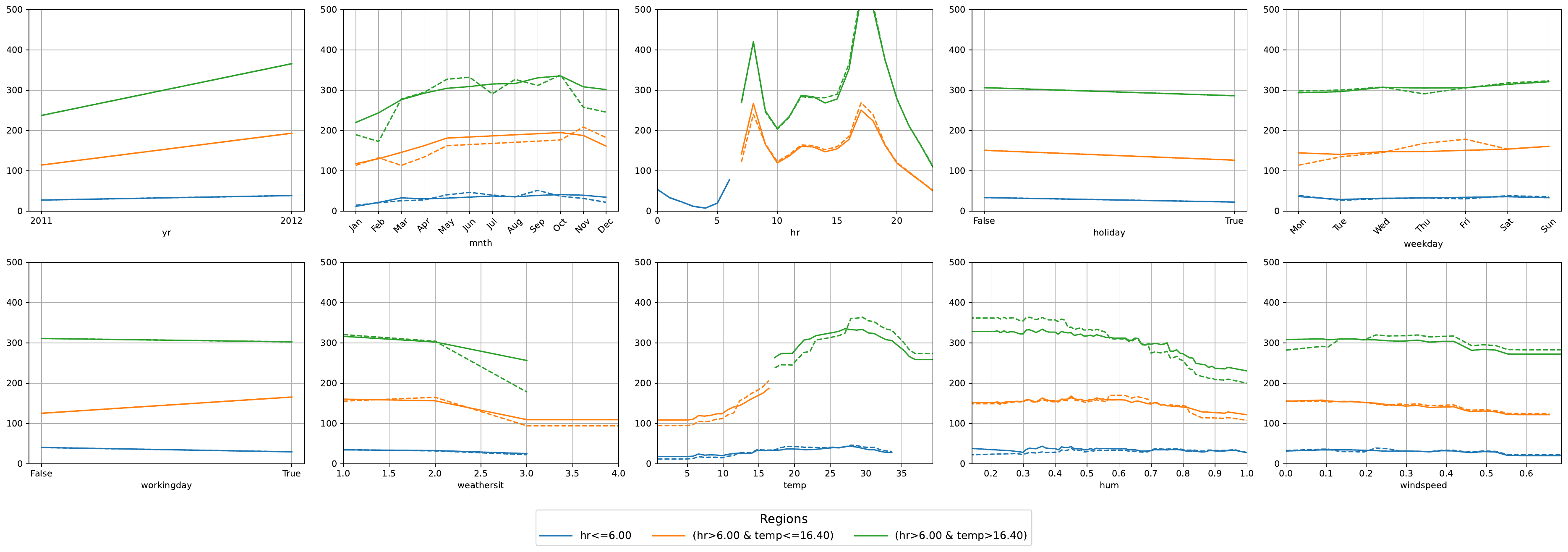}
    \caption{\correlation{Minimizing distribution disagreement} with $\mathcal{B}_{loc,x_0}$ and $\mathcal{I}_{pure}$ between MPlot ($\mathcal M_c$) and PDP 
    ($\mathcal M_m$). The disagreements are less important when restricting to early/daytime hours and hot/cold temperatures.}
    \label{fig:bike_marg_cond_regional_all}
\end{figure}
\begin{figure}[!t]
    \centering
    \begin{subfigure}{0.28\textwidth}
        \centering
        \includegraphics[width=\linewidth]{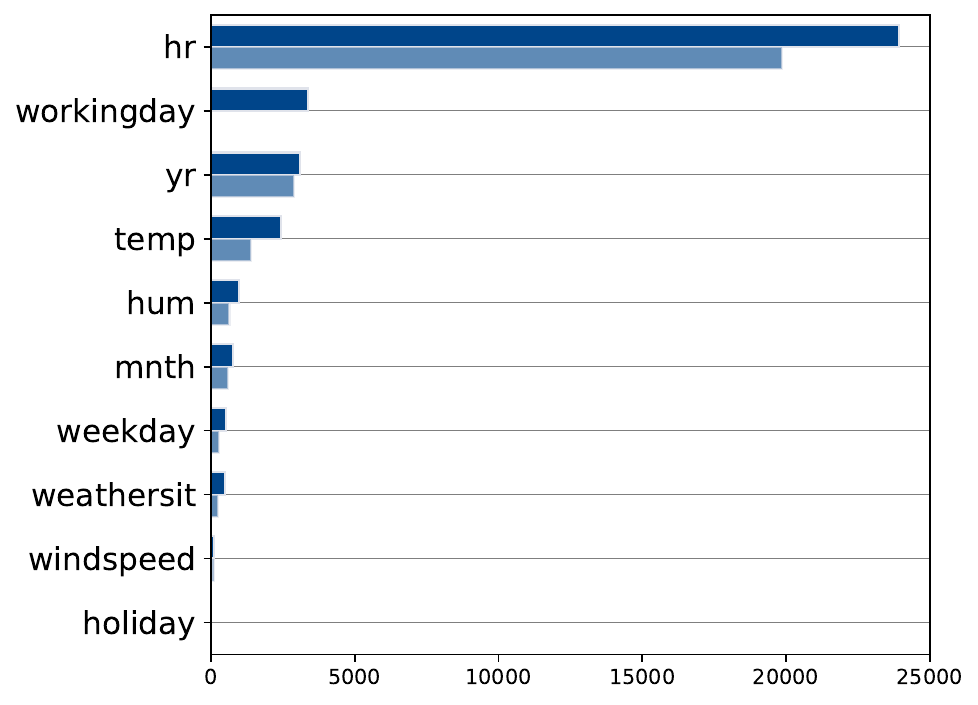}
        \caption{Global.}
    \end{subfigure}
    \begin{subfigure}{0.65\textwidth}
        \centering
        \includegraphics[width=\linewidth]{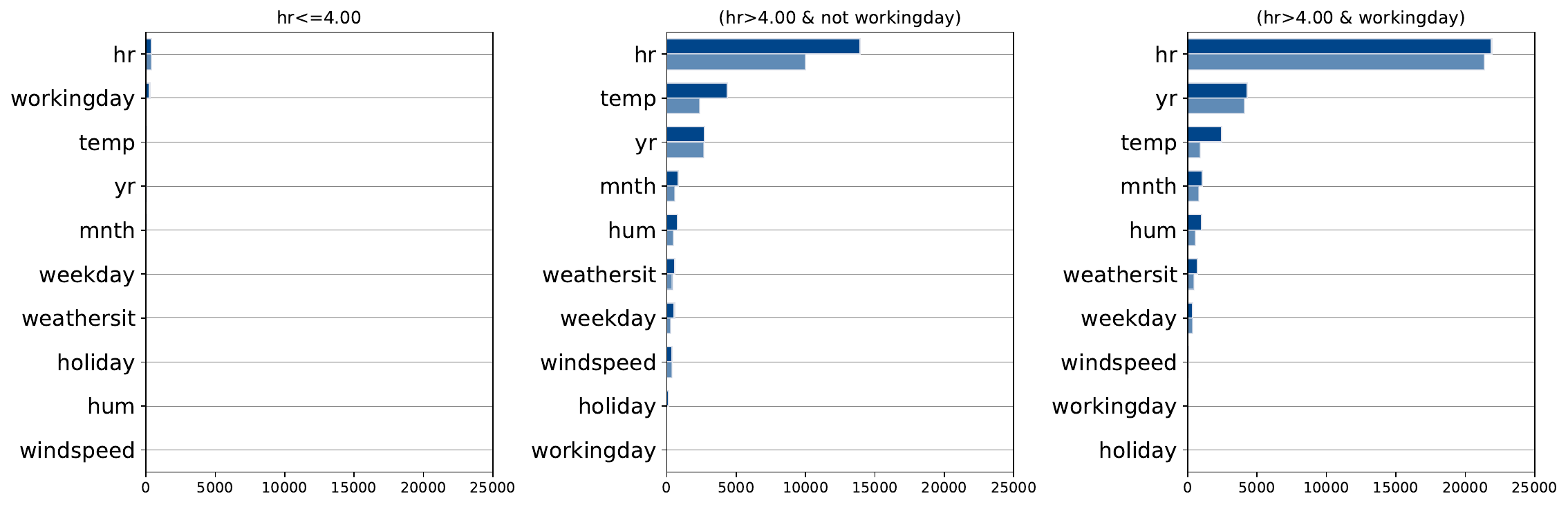}
        \caption{Regional.}
    \end{subfigure}
    \caption{\correlation{Minimizing distribution disagreement} with $\mathcal{B}_{risk}$ and $\mathcal{I}_{full}$ between CFI 
    ($\mathcal M_c$) and PFI ($\mathcal M_m$). The disagreements are reduced restricted to early hour or workingdays/not-workindays.}
    \label{fig:bike_pfi_vs_cfi}
\end{figure}

\subsection{Kin8nm}\label{app:experiments:kin8nm}

The kin8nm dataset is a simulation of the dynamics of a 8-link robot arm. The regression task aims at predicting the 
distance between the tip of the arm and a target based on 8 features, each representing an angle of the robot arm joints. 
Although our understanding of the dataset is extremely limited  (we do not know what each angle represents), studying it 
is still relevant because it involves high-order interactions that are extremely hard to visualize. This subsection 
presents how GRANITE can be used to handle these high-order interactions.

We fitted a Multi-Layered Perception (MLP) on this dataset reaching an $R^2$ coefficient of $0.9$. Although not a perfect
estimate, this model can still potentially highlight the strong interactions present in the robotic arm's dynamics.
To illustrate the presence of such interactions, we first computed the PDP (pure) and ICE (full) curves 
(see Figure~\ref{fig:kin8nm_ice_pdp}). The ICE curves show substantial variability, particularly for 
\texttt{theta6} and \texttt{theta7}, indicating that these features interact strongly with others.

\begin{figure}[t!]
    \centering
    \includegraphics[width=0.7\linewidth]{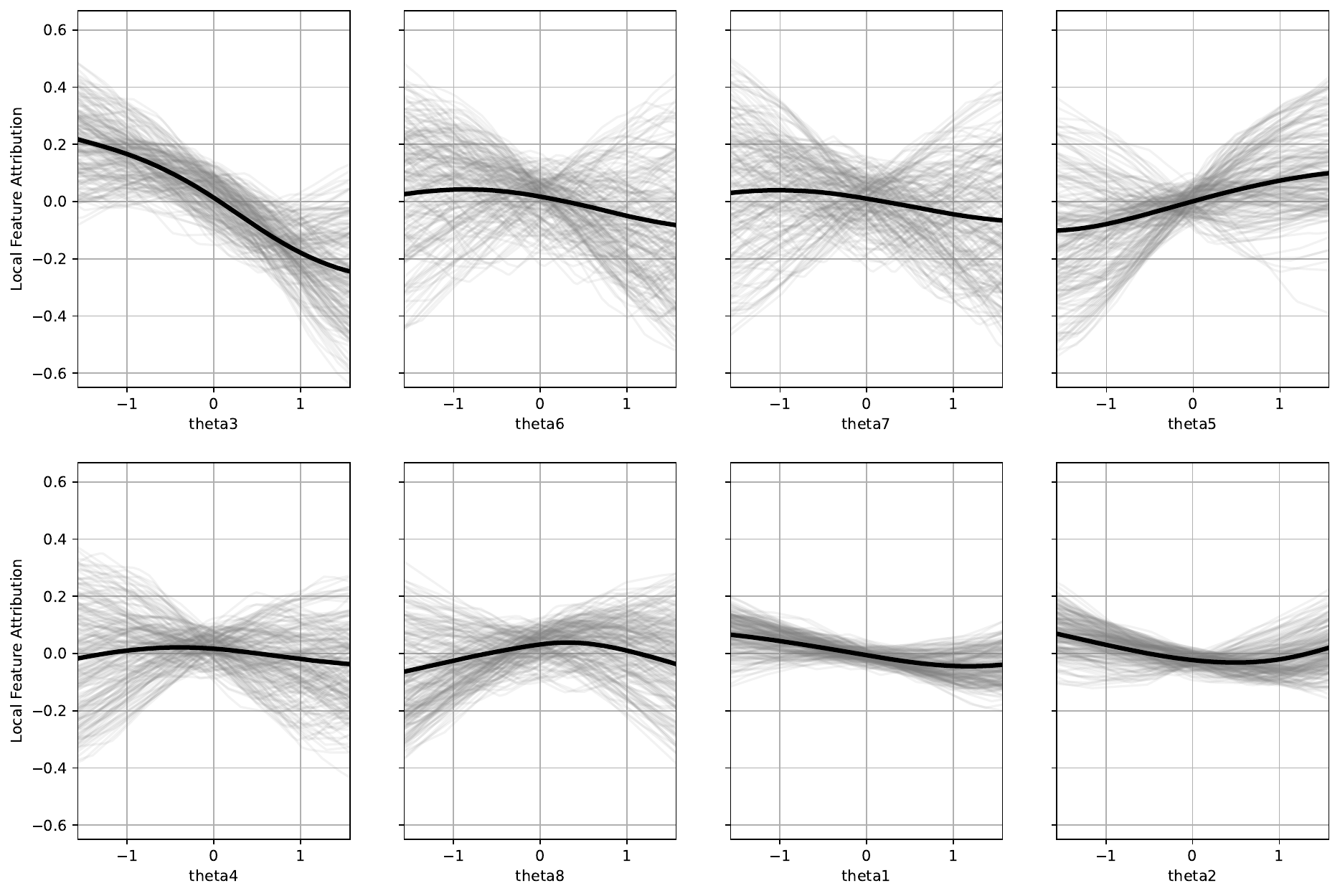}
    \caption{PDP (thick line) and ICE curves (thin lines) on Kin8nm.}
    \label{fig:kin8nm_ice_pdp}
\end{figure}
\begin{figure}[t!]
    \centering
    \includegraphics[width=0.7\linewidth]{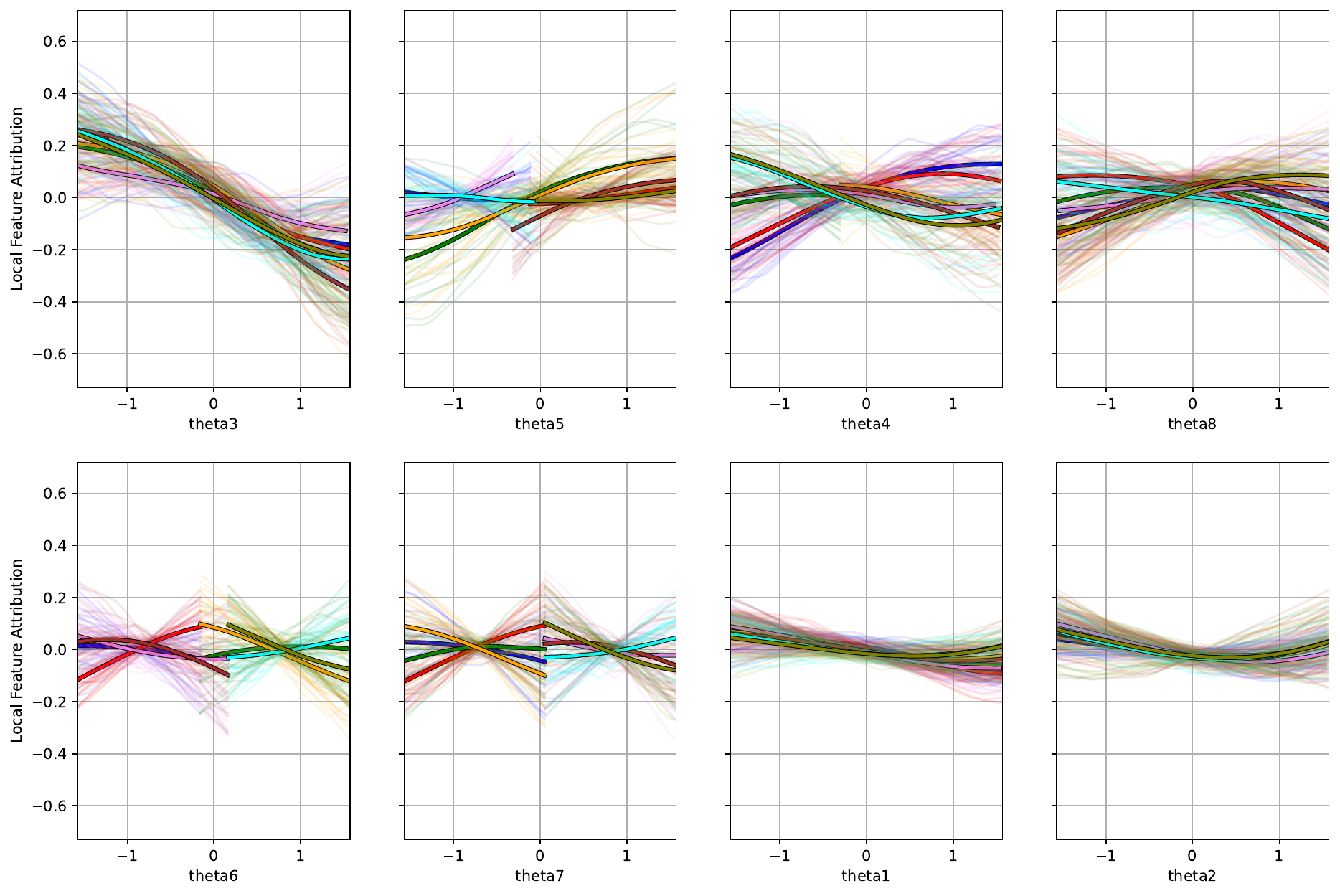}
    \includegraphics[width=0.9\linewidth]{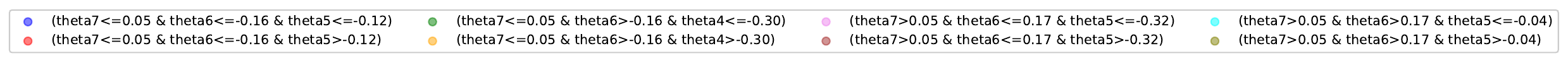}
    \caption{PDP (thick line) and ICE curves (thin lines) in 8 regions discovered by GRANITE on Kin8nm.}
    \label{fig:kin8nm_regional_ice_pdp}
\end{figure}

We then applied GRANITE to find regions where the pure and full marginal feature effects exhibit increased agreement. This corresponds to minimizing the interactions within the functional decomposition in each region. Figure \ref{fig:kin8nm_regional_ice_pdp}
shows the resulting regional PDP and ICE curves. Interactions are clearly reduced, as indicated by the diverse trends 
in regional PDP curves. For example, the slope of the \texttt{theta4}-PDP when $\texttt{theta4}<0$ can be positive, negative, or null depending on the region.

While univariate explanations provide useful insights, they can become difficult to interpret when numerous regions are considered. For example, moving from visualizing 8 univariate functions to $8 \times 8 = 64$ curves (corresponding to 8 regions) quickly becomes overwhelming. 
One reason GRANITE returns a relatively large number of regions is that it penalizes \textbf{all} interactions within the model, including order-2 interactions, which are often still interpretable. 
To address this, we opted to increase the expressivity of our explanations by incorporating not only individual feature effects but also pairwise interactions.
In this refined approach, GRANITE is then applied to minimize only higher-order interactions, preserving the interpretability of lower-order effects.

Figures~\ref{fig:kin8nm_interactions_regions}~(a) and (b) show the pure (a) and full (b) sensitivity effects of individual features and their pairwise interactions, 
computed using the whole dataset. Differences in edge widths between the two graphs reflect the presence of higher-order feature interactions. 
These higher-order interactions indicate that explaining the model using only pure pairwise interactions, while interpretable, can sometimes be misleading. 
To address this, we applied GRANITE to minimize the disagreements between the pure and full measures of pairwise interactions.
\begin{align*}
    \mathcal R_{\ell \mid \Omega}(\phi_{1 \mid \Omega},\phi_{2\mid\Omega}) := 
    \mathbb{E}_{P_{X}^\Omega}\left[ \sum_{S \in \mathcal{S}} \ell\left(\phi_{1\mid\Omega}(S)-\phi_{2\mid\Omega}(S)\right)\right],
\end{align*}
with $\mathcal{S}=\{S\subseteq D : |S| = 2\}$. Thus, we search for regions where interactions of order $3$ or more are minimized.
Figures~\ref{fig:kin8nm_interactions_regions}~(c)--(j) show the pure and full sensitivity effects across the four regions produced by the partitioning algorithm. 
It is evident that the pure and full pairwise interactions are more closely aligned within these regions. 
Therefore, the regional pure pairwise interaction effects provide a more faithful representation of the model.

As an illustration, we examined the pure interaction effects between \texttt{theta5} and \texttt{theta8} in regions 0 and 3. 
These regions were selected because, in the other two, strong disagreements between pure and full \texttt{theta5:theta8} interactions are still evident in Figure~\ref{fig:kin8nm_interactions_regions}. 
Figure~\ref{fig:kin8nm_interactions_regions_local} shows the pure individual effects and pure pairwise interactions across both regions. 
Notably, the interaction completely flips sign between the two regions, providing evidence of a three-way (or higher-order) interaction. 
Such effects would have been very difficult to visualize without using GRANITE to partition the input space.
\begin{figure}
    \centering
    \begin{subfigure}{0.45\textwidth}
        \centering
        \includegraphics[width=0.45\linewidth]{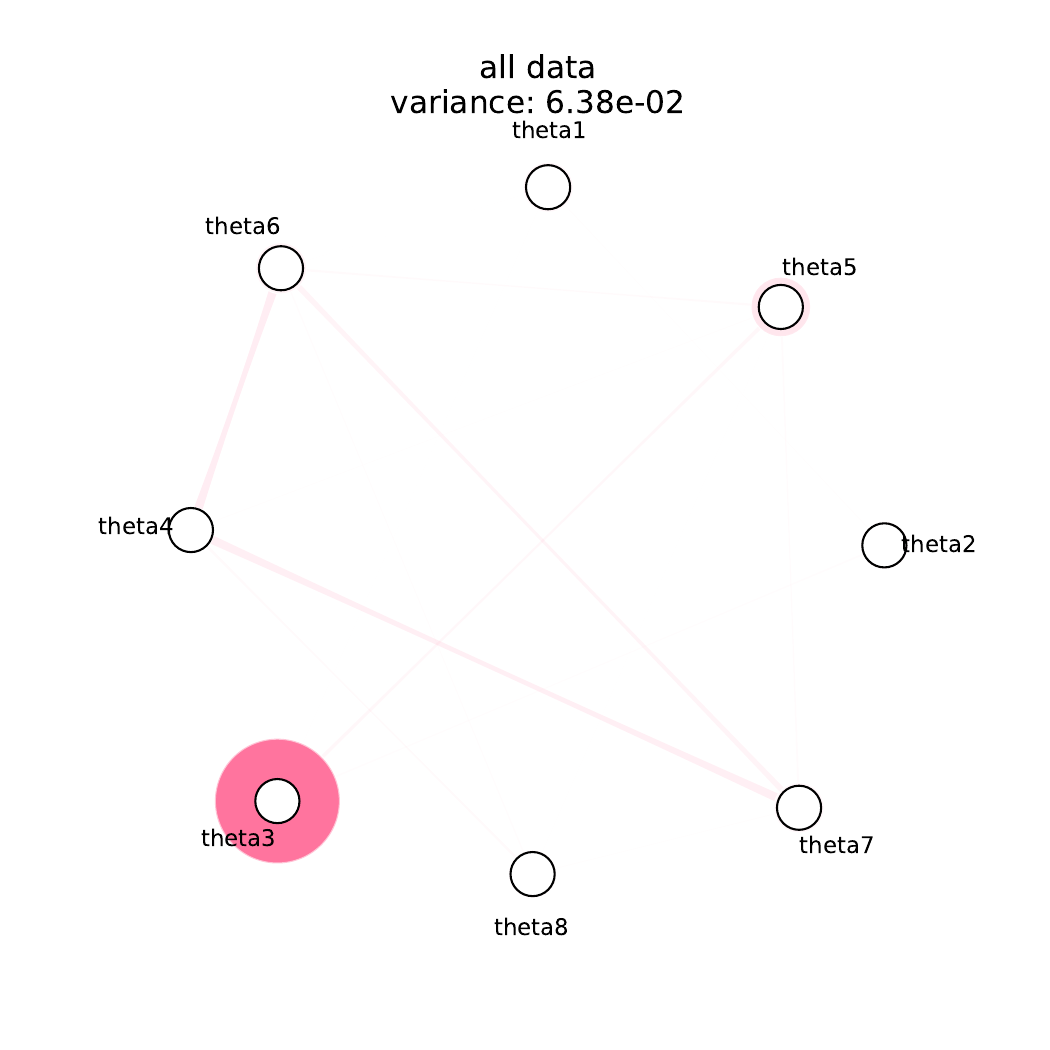}
        \caption{Marginal-\interaction{Pure}-Sensitivity effects for features and pair-wise interactions
        using whole dataset as reference.}
    \end{subfigure}
    \begin{subfigure}{0.45\textwidth}
        \centering
        \includegraphics[width=0.45\linewidth]{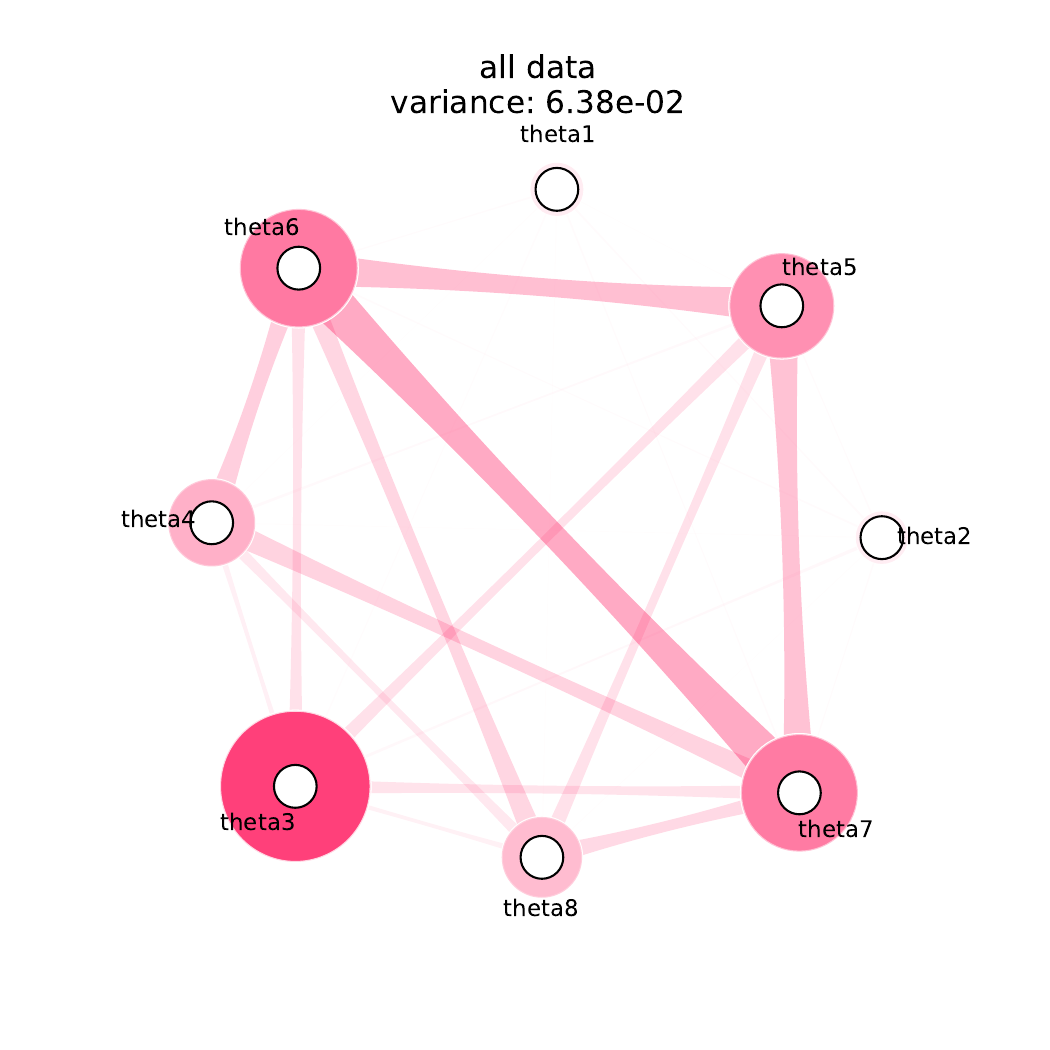}
        \caption{Marginal-\interaction{Full}-Sensitivity effects for features and pair-wise interactions
        using whole dataset as reference.}
    \end{subfigure}
    \\
    \begin{subfigure}{0.24\textwidth}
        \includegraphics[width=0.9\linewidth]{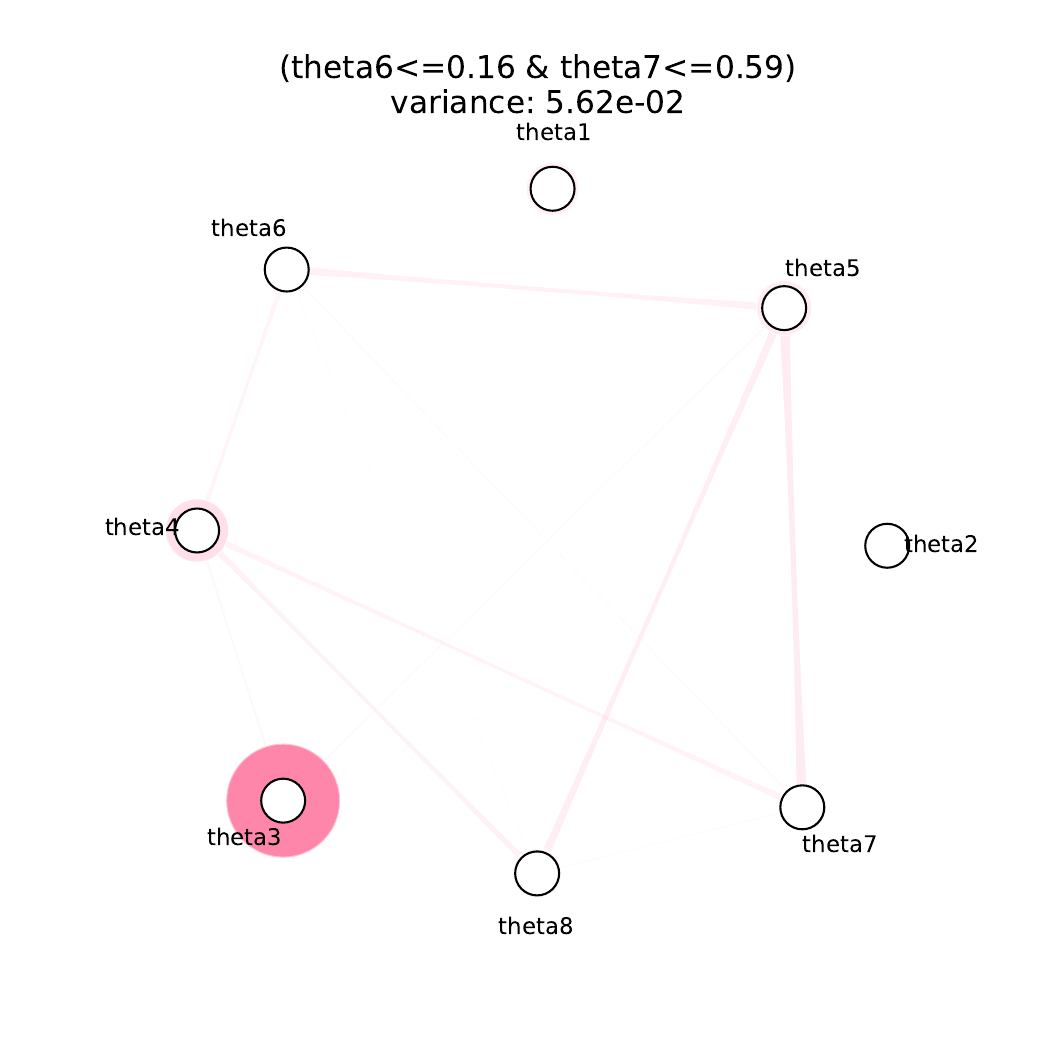}
        \caption{Pure Region 0.}
    \end{subfigure}
    \begin{subfigure}{0.24\textwidth}
        \includegraphics[width=0.9\linewidth]{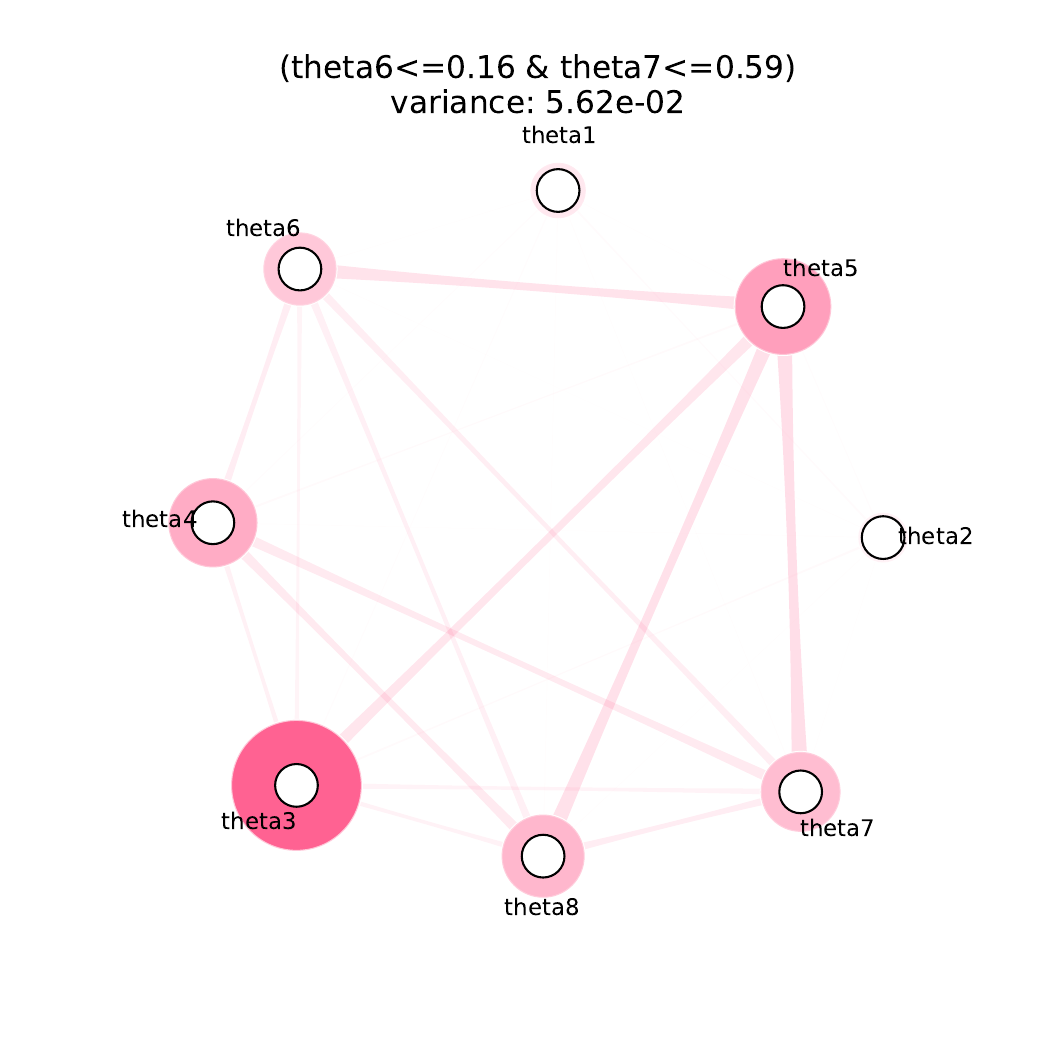}
        \caption{Full Region 0.}
    \end{subfigure}
    \begin{subfigure}{0.24\textwidth}
        \includegraphics[width=0.9\linewidth]{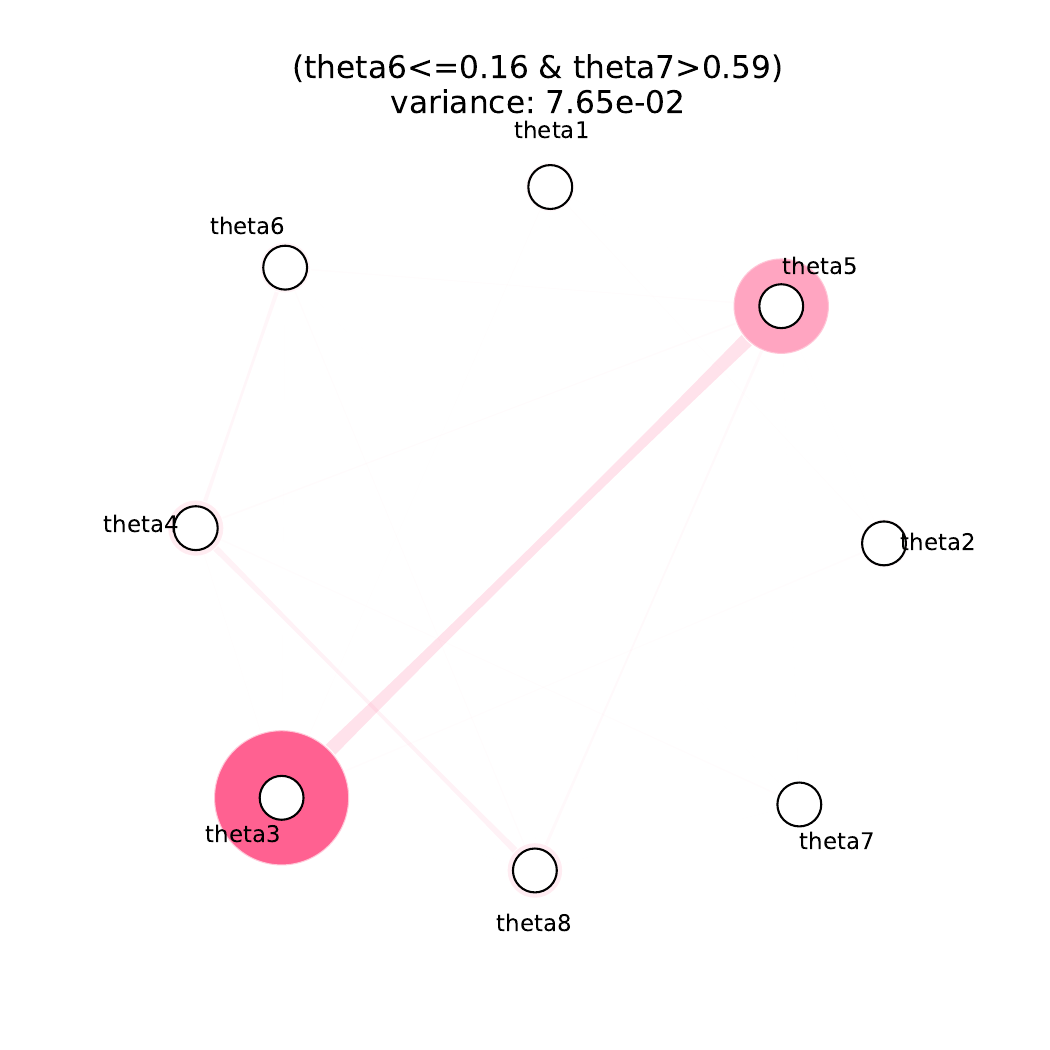}
        \caption{Pure Region 1.}
    \end{subfigure}
    \begin{subfigure}{0.24\textwidth}
        \includegraphics[width=0.9\linewidth]{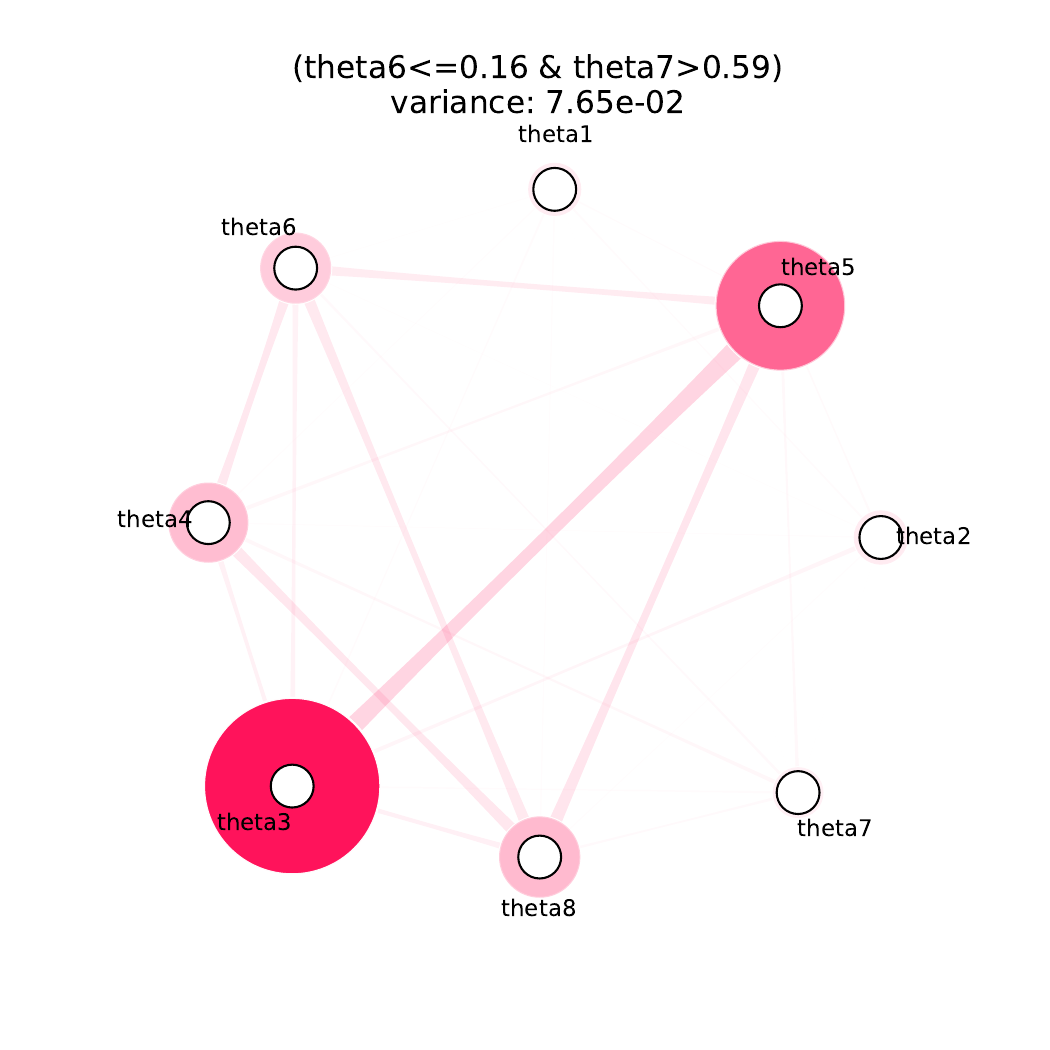}
        \caption{Full Region 1.}
    \end{subfigure}
    \begin{subfigure}{0.24\textwidth}
        \includegraphics[width=0.9\linewidth]{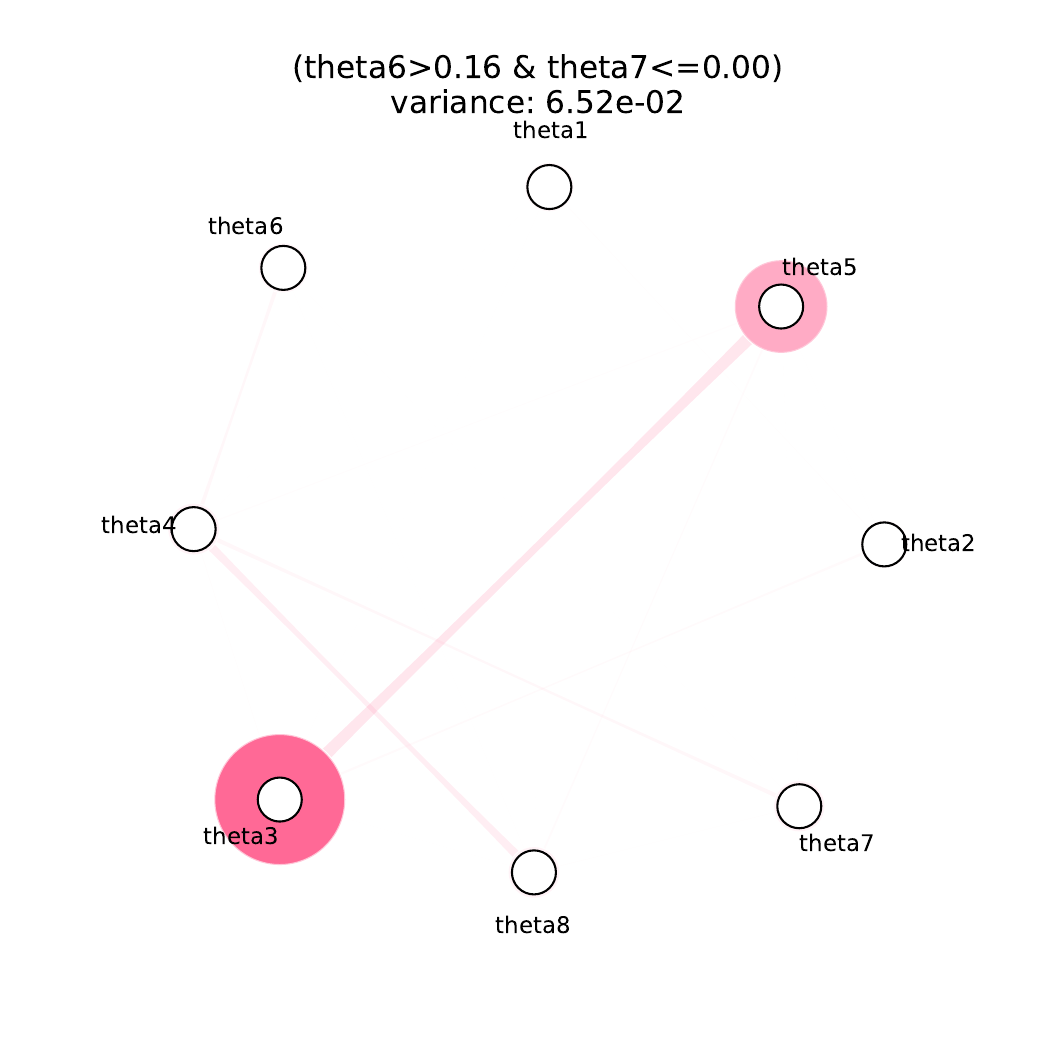}
        \caption{Pure Region 2.}
    \end{subfigure}
    \begin{subfigure}{0.24\textwidth}
        \includegraphics[width=0.9\linewidth]{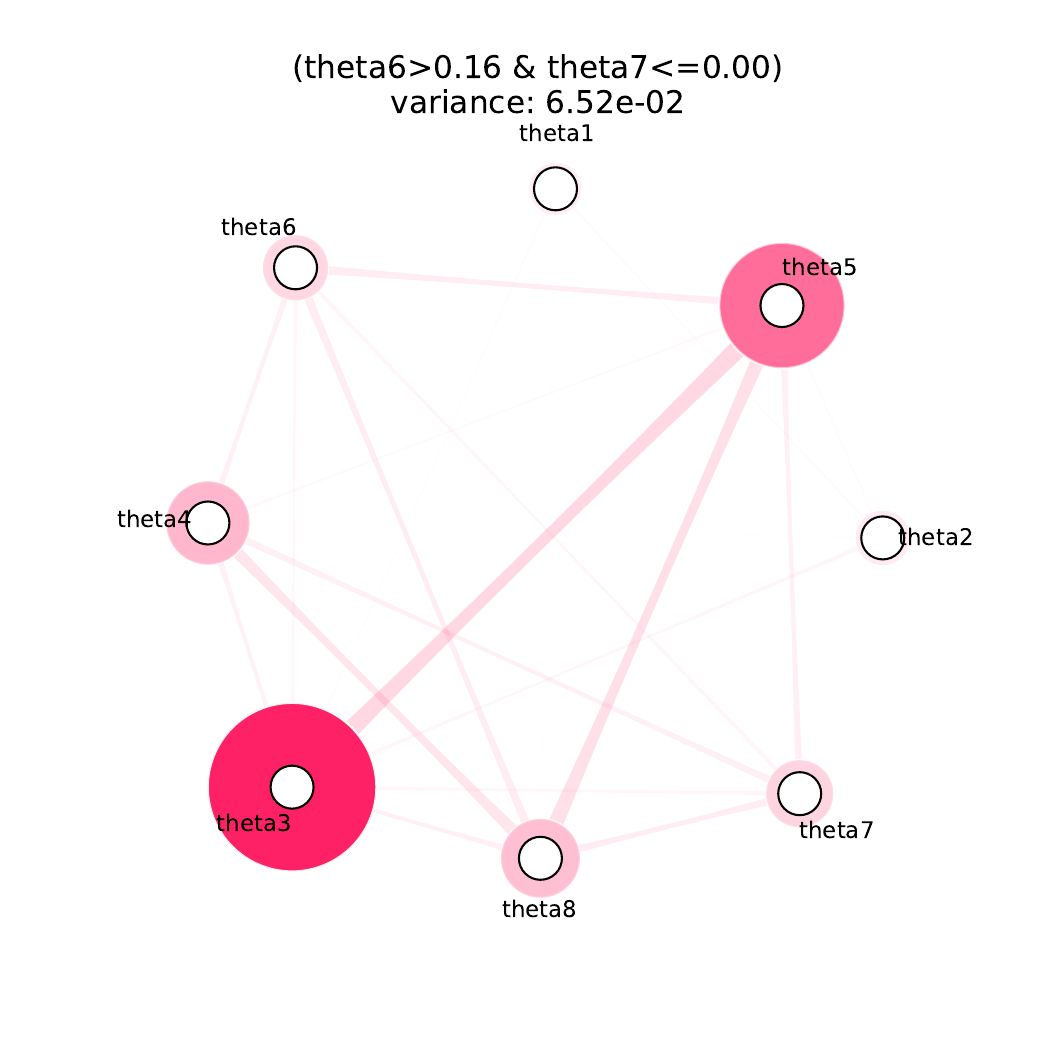}
        \caption{Full Region 2.}
    \end{subfigure}
    \begin{subfigure}{0.24\textwidth}
        \includegraphics[width=0.9\linewidth]{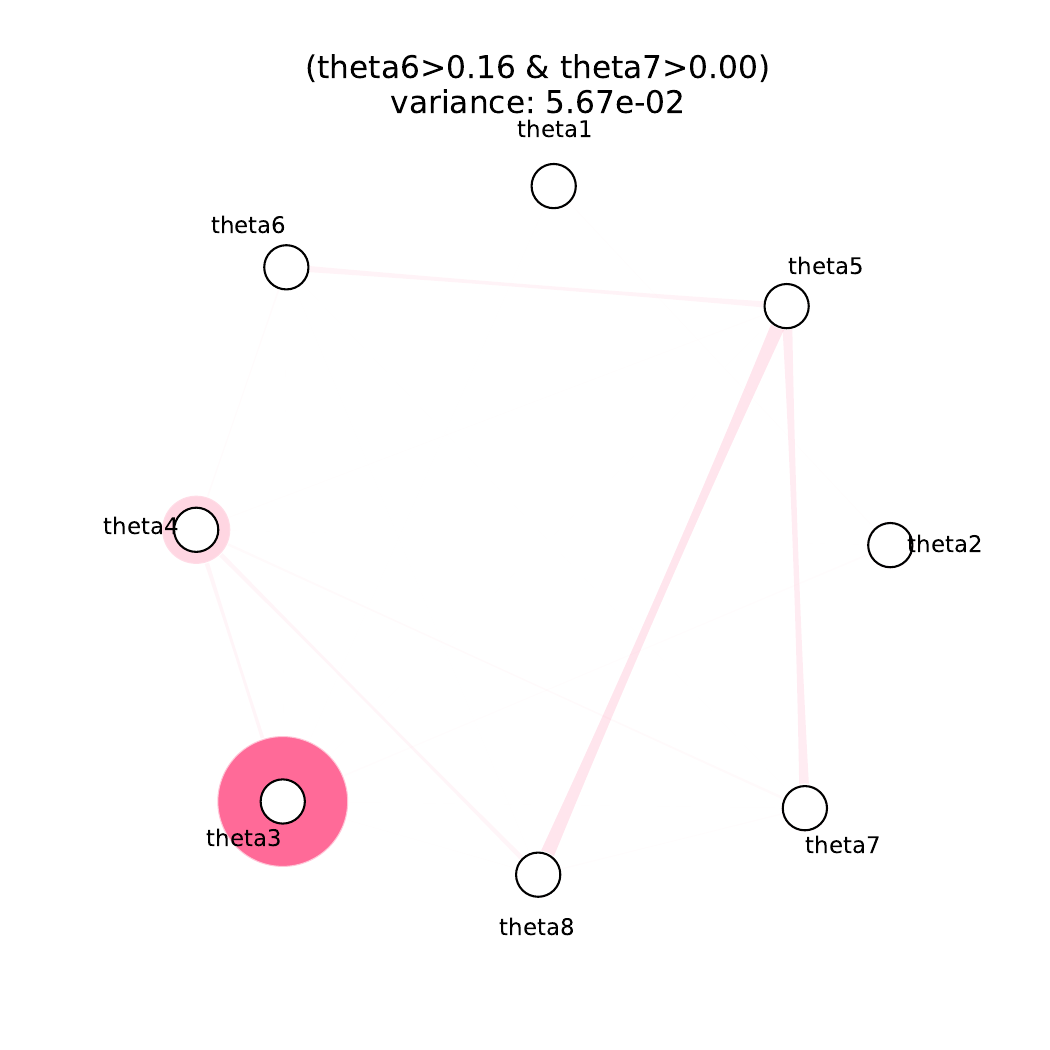}
        \caption{Pure Region 3.}
    \end{subfigure}
    \begin{subfigure}{0.24\textwidth}
        \includegraphics[width=0.9\linewidth]{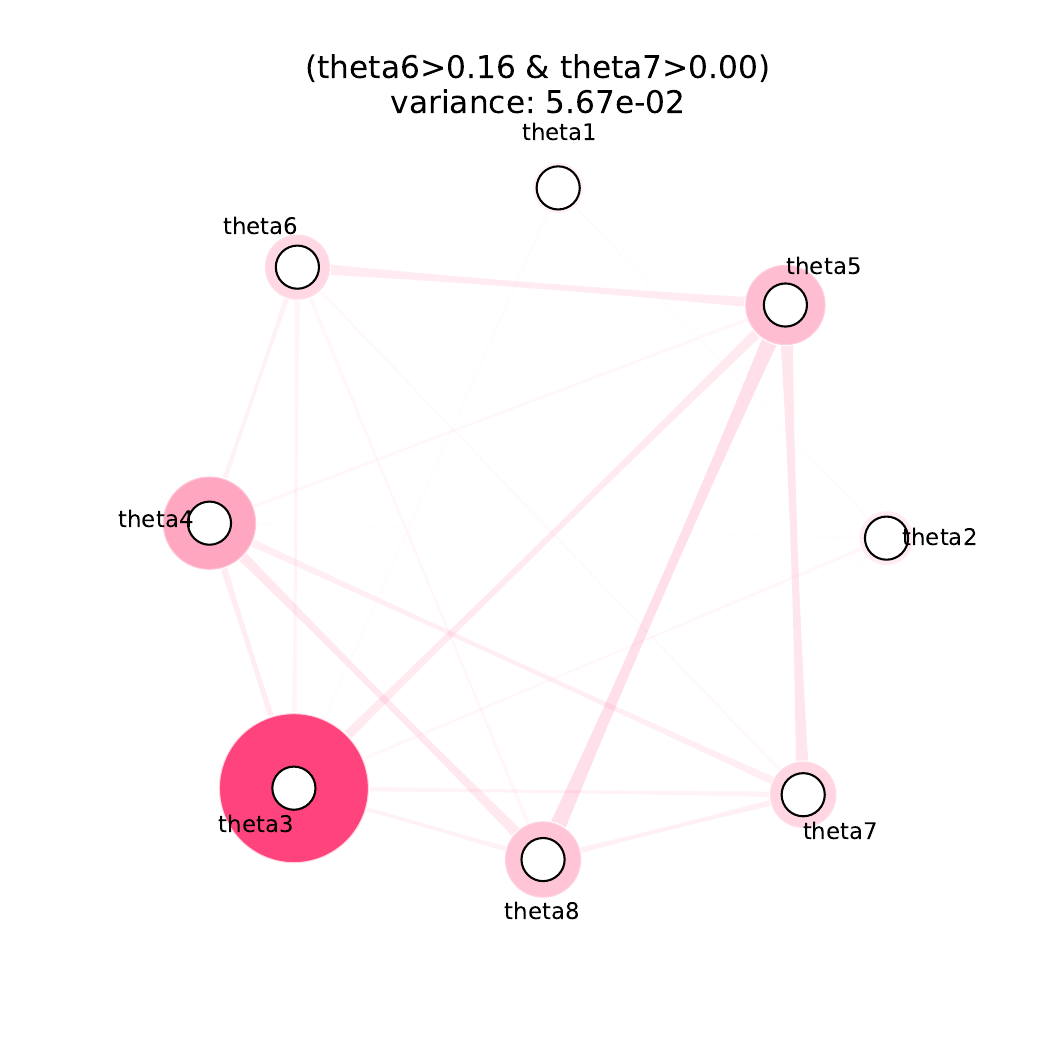}
        \caption{Full Region 3.}
    \end{subfigure}
    \vspace{0.2cm}
    \caption{GRANITE splits the input space of Kin8nm into four regions, each with reduced high-order 
    interactions, i.e., the pure and full pair-wise interactions agree more regionally.}
    \label{fig:kin8nm_interactions_regions}
\end{figure}
\begin{figure}
    \centering
    \includegraphics[width=0.65\linewidth]{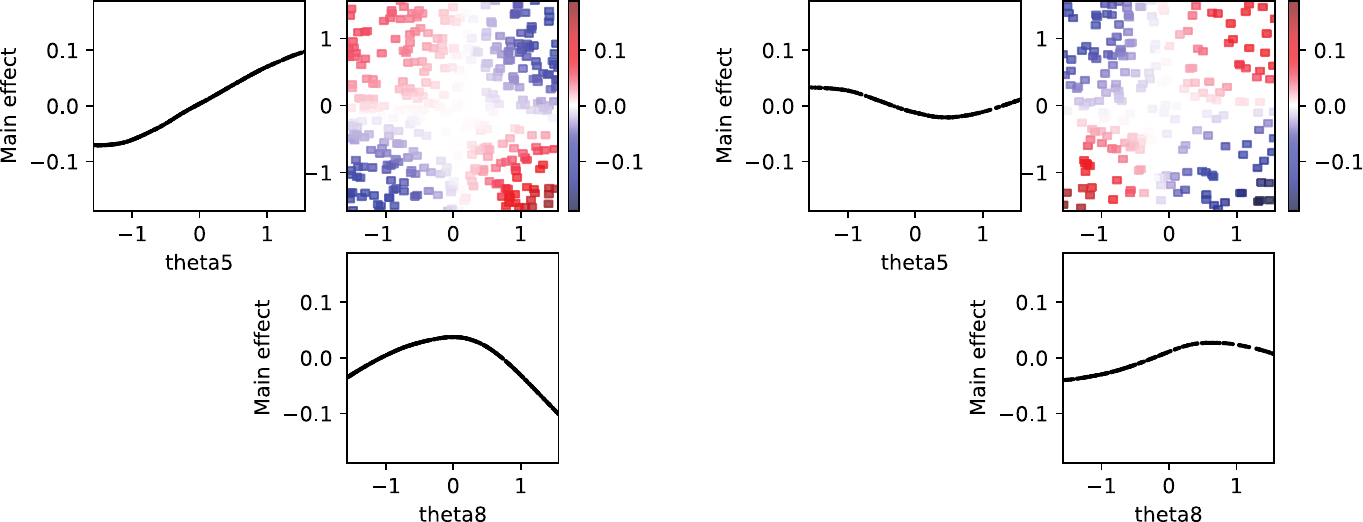}
    \\
    \begin{subfigure}{0.35\textwidth}
        \centering
        \caption{Region 0.}
    \end{subfigure}
    \begin{subfigure}{0.35\textwidth}
        \centering
        \caption{Region 3.}
    \end{subfigure}
    \vspace{0.2cm}
    \caption{Kin8nm: Pure main effects for \texttt{theta5}, \texttt{theta8} and their pairwise interactions across two regions.}
    \label{fig:kin8nm_interactions_regions_local}
\end{figure}

\subsection{Diabetes}
\label{app:experiments:diabetes}
The \texttt{Diabetes} dataset is a regression task aimed at predicting disease progression one year after baseline. 
It contains $442$ samples and $10$ numerical features: \emph{age}, \emph{sex}, \emph{bmi}, \emph{bp} (blood pressure), and six serum measurements (\emph{s1}--\emph{s6}) representing various blood and lipid concentrations. 
The data were split into training and test sets with an 80--20 ratio, resulting in $354$ training and $88$ test instances. 
A \texttt{HistGradientBoostingRegressor}
was optimized using 5-fold cross-validation on the training set and random search for hyperparameter tuning.

The given features allow for a meaningful grouping as follows: 
\textbf{Demographics:} \emph{age}, \emph{sex}; \quad 
\textbf{Body metrics:} \emph{bmi}, \emph{bp}; \quad 
\textbf{Lab markers:} \emph{s1}--\emph{s6}.

\interaction{Minimizing Feature Interactions.} 
We now analyze the overall importance of these feature groups and how their ranking changes when interactions between groups are minimized. 
We apply GRANITE for \textbf{joint influence measures} with sensitivity behavior $\mathcal{B}_{\text{sens}}$ and marginal masking $\mathcal{M}_m$ to minimize interactions between full and pure Sobol variances. 
Figure~\ref{fig:diabetes_joint_sensitivity} shows that, on a global level, body metrics and lab markers dominate the influence of demographic features, with body metrics being slightly more important. 
While the pure and full joint influences yield similar global rankings, regional analyses reveal notable differences: in particular, lab markers are by far the most important group for individuals with low BMI. 
In these regions, body metrics contribute the least, whereas demographics and body metrics exhibit higher influence for individuals with both low BMI and low blood pressure, compared to those with low BMI but high blood pressure, where lab markers appear to be the sole dominant group. 
For individuals with high BMI, both lab markers and body metrics play a substantial role, although their relative importance varies by gender.

\begin{figure}[htbp]
    \centering
    \begin{minipage}[c]{0.35\linewidth}
        \centering
        \includegraphics[width=\linewidth]{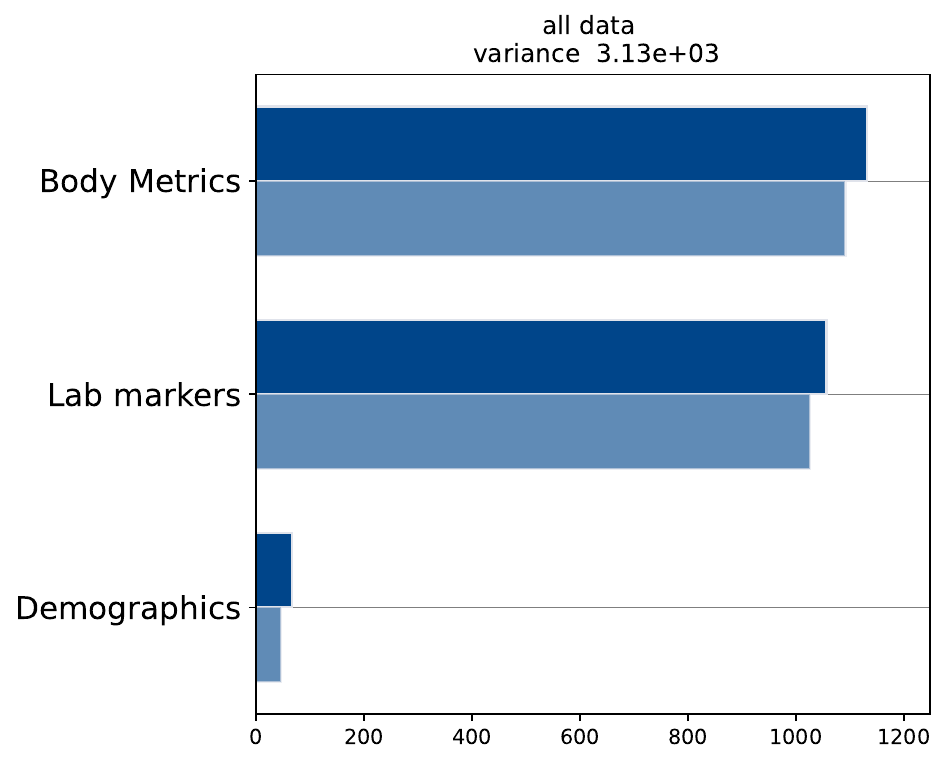}
    \end{minipage}
    \hspace{0.02\linewidth}
    \begin{minipage}[c]{0.58\linewidth}
        \centering
        \begin{minipage}{0.48\linewidth}
            \centering
            \includegraphics[width=\linewidth]{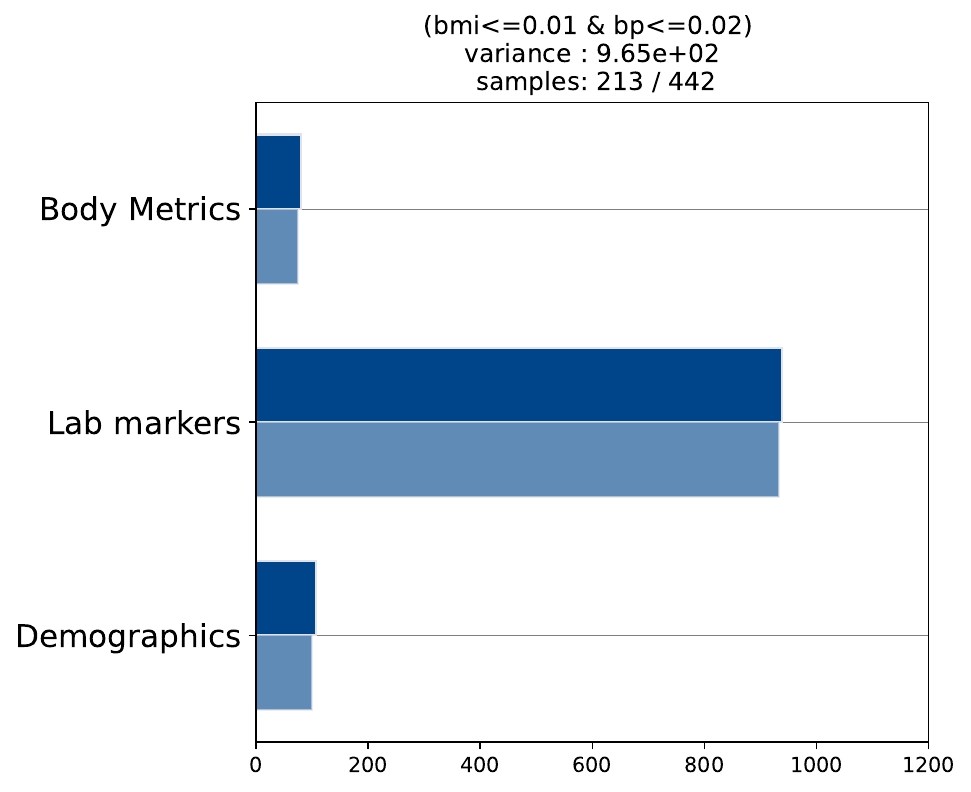}
        \end{minipage}
        \begin{minipage}{0.48\linewidth}
            \centering
            \includegraphics[width=\linewidth]{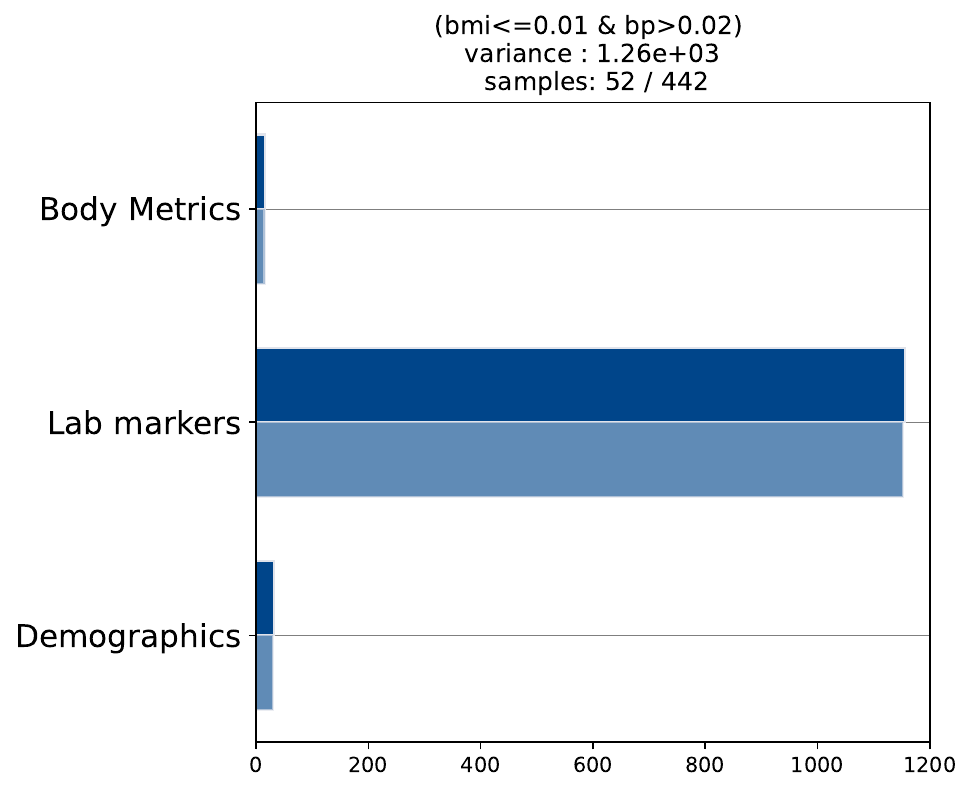}
        \end{minipage}

        \vspace{2mm}

        \begin{minipage}{0.48\linewidth}
            \centering
            \includegraphics[width=\linewidth]{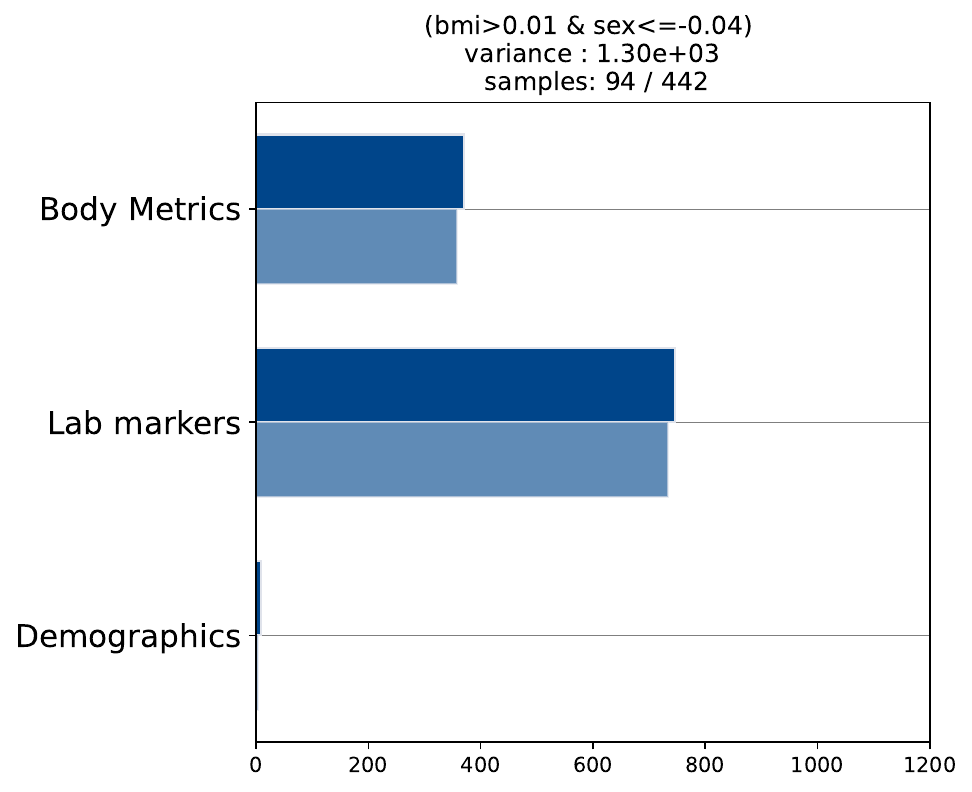}
        \end{minipage}
        \begin{minipage}{0.48\linewidth}
            \centering
            \includegraphics[width=\linewidth]{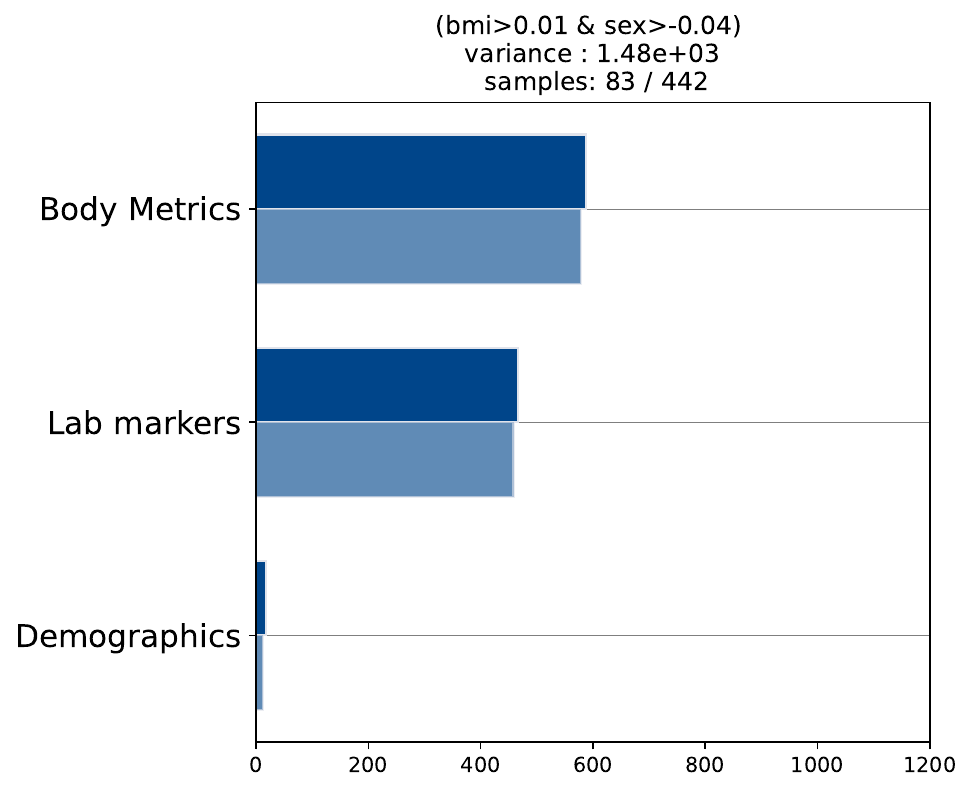}
        \end{minipage}
    \end{minipage}
    \caption{Diabetes: Full (top bar) and pure (bottom bar) sensitivity-based joint influences for global (left) and regional (right) effects measured by Sobol variances for the three feature groups.}
    \label{fig:diabetes_joint_sensitivity}
\end{figure}

\end{document}